\documentclass{report}

\usepackage[T2A]{fontenc}			
\usepackage[utf8]{inputenc}			
\usepackage{microtype}
\usepackage{graphicx}
\usepackage{subfigure}
\usepackage{booktabs} 
\usepackage{amsmath, amssymb, amsthm}

\usepackage{hyperref}

\usepackage{geometry} 

\renewcommand{\AA}{\mathcal{A}}
\newcommand{\HH}{\mathcal{H}}

\newcommand{\MM}{\mathcal{M}}
\newcommand{\BB}{\mathcal{B}}
\newcommand{\PP}{\mathcal{P}}
\newcommand{\LL}{\mathcal{L}}

\newcommand{\XX}{\mathcal{X}}
\newcommand{\EE}{\mathbb{E}\,}
\newcommand{\DD}{\mathcal{D}}
\newcommand{\RR}{\mathbb{R}}

\newcommand{\NN}{\mathcal{N}}
\newcommand{\CC}{\mathcal{C}}
\newcommand{\Cplx}{\mathbb{C}}
\newcommand{\FF}{\mathcal{F}}
\newcommand{\VV}{\mathcal{V}}

\newcommand{\xx}{\mathbf{x}}

\newcommand{\Var}{\mathbb{V}\mathrm{ar}\,}

\newcommand{\const}{\mathrm{const}}
\newcommand{\citet}{\cite}

\newcommand{\Rad}[2]{\mathrm{Rad}({#1} \, | \, {#2})}
\newcommand{\VC}[1]{\mathrm{VC}({#1})}
\newcommand{\kld}[2]{\mathrm{KL}({#1}\; || \;{#2})}

\DeclareMathOperator*{\argmin}{arg\,min}
\DeclareMathOperator*{\Argmax}{Arg\,max}
\DeclareMathOperator*{\Argmin}{Arg\,min}

\DeclareMathOperator*{\diag}{diag}
\DeclareMathOperator*{\tr}{tr}
\DeclareMathOperator*{\rk}{rk}
\DeclareMathOperator*{\supp}{supp}
\DeclareMathOperator{\Ind}{Ind}
\DeclareMathOperator*{\plim}{plim}

\newtheorem{theorem}{Theorem}
\newtheorem{prop}{Proposition}
\newtheorem{lemma}{Lemma}
\newtheorem{corollary}{Corollary}
\newtheorem{conjecture}{Conjecture}

\title{Notes on Deep Learning Theory}

\author{
    \textbf{Evgenii (Eugene) Golikov}
    \\ Neural Networks and Deep Learning lab.
    \\ Moscow Institute of Physics and Technology 
    \\ Moscow, Russia 
    \\ \href{mailto:golikov.ea@mipt.ru}{\texttt{golikov.ea@mipt.ru}}
}

\begin{document}
    \maketitle

    \begin{abstract}
        These are the notes for the lectures that I was giving during Fall 2020 at the Moscow Institute of Physics and Technology (MIPT) and at the Yandex School of Data Analysis (YSDA).
        The notes cover some aspects of initialization, loss landscape, generalization, and a neural tangent kernel theory.
        While many other topics (e.g. expressivity, a mean-field theory, a double descent phenomenon) are missing in the current version, we plan to add them in future revisions.
    \end{abstract}
    
    \tableofcontents

    \chapter{Introduction}

    Machine learning aims to solve the following problem:
    \begin{equation}
        R(f) \to \min_{f \in \FF}.
        \label{eq:true_risk_minimization}
    \end{equation}
    Here $R(f) = \EE_{x,y \sim \DD} r(y, f(x))$ is a true risk of a model $f$ from a class $\FF$, and $\DD$ is a data distribution.
    However, we do not have an access to the true data distribution; instead we have a finite set of i.i.d. samples from it: $S_n = \{(x_i,y_i)\}_{i=1}^n \sim \DD^n$.
    For this reason, instead of approaching~(\ref{eq:true_risk_minimization}), we substitite it with an empirical risk minimization:
    \begin{equation}
        \hat R_n(f) \to \min_{f \in \FF},
        \label{eq:emp_risk_minimization}
    \end{equation}
    where $\hat R_n(f) = \EE_{x,y \in S_n} r(y, f(x))$ is an empirical risk of a model $f$ from a class $\FF$.

    \section{Generalization ability}

    How does the solution of~(\ref{eq:emp_risk_minimization}) relate to~(\ref{eq:true_risk_minimization})?
    In other words, we aim to upper-bound the difference between the two risks:
    \begin{equation}
        R(\hat f_n) - \hat R_n(\hat f_n) \leq
        \mathrm{bound}(\hat f_n, \FF, n, \delta)
        \quad
        \text{w.p. $\geq 1 - \delta$ over $S_n$},
        \label{eq:general_bound}
    \end{equation}
    where $\hat f_n \in \FF$ is a result of training the model on the dataset $S_n$.

    We call the bound~(\ref{eq:general_bound}) \emph{a-posteriori} if it depends on the resulting model $\hat f_n$, and we call it \emph{a-priori} if it does not.
    An a-priori bound allows one to estimate the risk difference before training, while an a-posteriori bound estimates the risk difference based on the final model.

    \emph{Uniform} bounds are instances of an a-priori class:
    \begin{equation}
        R(\hat f_n) - \hat R_n(\hat f_n) \leq
        \sup_{f \in \FF} |R(f) - \hat R_n(f)| \leq
        \mathrm{ubound}(\FF, n, \delta)
        \quad
        \text{w.p. $\geq 1 - \delta$ over $S_n$},
        \label{eq:uniform_bound}
    \end{equation}
    A typical form of the uniform bound is the following:
    \begin{equation}
        \mathrm{ubound}(\FF, n, \delta) =
        O\left(\sqrt{\frac{\CC(\FF) + \log(1/\delta)}{n}}\right),
        \label{eq:uniform_bound_instance}
    \end{equation}
    where $\CC(\FF)$ is a \emph{complexity} of the class $\FF$.
    
    The bound above suggests that the generalization ability, measured by the risk difference, decays as the model class becomes larger.
    This suggestion conforms the classical \emph{bias-variance trade-off} curve.
    The curve can be reproduced if we fit the \emph{Runge function} with a polynomial using a train set of equidistant points; the same phenomena can be observed for decision trees.

    A typical notion of model class complexity is VC-dimension \cite{vapnik1971}.
    For neural networks, VC-dimension grows at least linearly with the number of parameters \cite{bartlett2019nearly}.
    Hence the bound~(\ref{eq:uniform_bound_instance}) becomes vacuous for large enough nets.
    However, as we observe, the empirical (train) risk $\hat R_n$ vanishes, while the true (test) risk \emph{saturates} for large enough width (see Figure 1 of \cite{neyshabur2014in}).

    One might hypothesize that the problem is in VC-dimension, which overestimates the complexity of neural nets.
    However, the problem turns out to be in uniform bounds in general.
    Indeed, if the class $\FF$ contains a \emph{bad network}, i.e. a network that perfectly fits the train data but fails desperately on the true data distribution, the uniform bound~(\ref{eq:uniform_bound}) becomes at least nearly vacuous.
    In realistic scenarios, such a bad network can be found explicitly: \cite{zhang2016understanding} demonstrated that practically large nets can fit data with random labels; similarly, these nets can fit the training data plus some additional data with random labels.
    Such nets fit the training data perfectly but generalize poorly.

    Up to this point, we know that among the networks with zero training risk, some nets generalize well, while some generalize poorly.
    Suppose we managed to come with some \emph{model complexity measure} that is symptomatic for poor generalization: bad nets have higher complexity than good ones.
    If we did, we can come up with a better bound by prioritizing less complex models.

    Such prioritization is naturally supported by a \emph{PAC-bayesian paradigm.}
    First, we come up with a \emph{prior} distribution $P$ over models.
    This distribution should not depend on the train dataset $S_n$.
    Then we build a \emph{posterior} distribution $Q \mid S_n$ over models based on observed data.
    For instance, if we fix random seeds, a usual network training procedure gives a posterior distribution concentrated in a single model $\hat f_n$.
    The \emph{PAC-bayesian bound} \cite{mcallester1999some} takes the following form:
    \begin{equation}
        R(Q \mid S_n) - \hat R_n(Q \mid S_n) \leq
        O\left(\sqrt{\frac{KL(Q \mid S_n \| P) + \log(1/\delta)}{n}}\right)
        \quad
        \text{w.p. $\geq 1 - \delta$ over $S_n$},
        \label{eq:pac_bayesian_bound}
    \end{equation}
    where $R(Q)$ is an expected risk for models sampled from $Q$; similarly for $\hat R_n(Q)$.
    If more complex models are less likely to be found, then we can embed this information into prior, thus making the KL-divergence typically smaller.
    
    The PAC-bayesian bound~(\ref{eq:pac_bayesian_bound}) is an example of an a-posteriori bound, since the bound depends on $Q$.
    However, it is possible to obtain an a-priori bound using the same paradigm \cite{neyshabur2018a}.

    The bound~(\ref{eq:pac_bayesian_bound}) becomes better when our training procedure tends to find models that are probable according to the prior.
    But what kind of models does the gradient descent typically find?
    Does it implicitly minimize some complexity measure of the resulting model?
    Despite the existence of bad networks, minimizing the train loss using a gradient descent typically reveals well-performing solutions.
    This phenomenon is referred as an \emph{implicit bias} of gradient descent.

    Another problem with a-priori bounds is that they all are effectively \emph{two-sided:} all of them are bounding an absolute value of the risk difference, rather then the risk difference itself. 
    Two-sided bounds fail if there exist networks that generalize well, while failing on a given train set.
    \cite{nagarajan2019uniform} have constructed a problem for which such networks are typically found by gradient descent.

    \section{Global convergence}

    We have introduced the empirical minimization problem~(\ref{eq:emp_risk_minimization}) because we were not able to minimize the true risk directly: see~(\ref{eq:true_risk_minimization}).
    But are we able to minimize the empirical risk?
    Let $f(x; \theta)$ be a neural net evaluated at input $x$ with parameters $\theta$.
    Consider a loss function $\ell$ that is a convex surrogate of a risk $r$.
    Then minimizing the train loss will imply empirical risk minimization:
    \begin{equation}
        \hat\LL_n(\theta) =
        \EE_{x,y \in S_n} \ell(y, f(x; \theta)) \to \min_\theta.
        \label{eq:loss_minimization}
    \end{equation}

    Neural nets are complex non-linear functions of both inputs and weights; we can hardly expect the loss landscape $\hat\LL_n(\theta)$ induced by such functions to be simple.
    At least, for non-trivial neural nets $\hat\LL_n$ is a non-convex function of $\theta$.
    Hence it can have local minima that are not global.

    The most widely-used method of solving the problem~(\ref{eq:loss_minimization}) for deep learning is gradient descent (GD), or some of its variants.
    Since GD is a local method, it cannot have any global convergence guarantees in general case.
    However, for practically-sized neural nets it \emph{always} succeeds in finding a global minimum.

    Given this observation, it is tempting to hypothesize that despite of the non-convexity, all local minima of $\hat\LL_n(\theta)$ are global.
    This turns to be true for linear nets \cite{kawaguchi2016deep,lu2017depth,laurent2018deep}, and for non-linear nets if they are sufficiently wide \cite{nguyen2019connected}.

    While globality of local minima implies almost sure convergence of gradient descent \cite{lee2016gradient, panageas2017gradient}, there are no guarantees on convergence speed.
    Generally, convergence speed depends on initialization.
    For instance, initializing linear nets orthogonally makes the optimization speed independent of depth \cite{saxe2013exact}.
    Ill-posed initialization may heavily slow down the optimization process.
    \cite{glorot2010understanding,he2015delving} proposed heuristics for preventing such situations.

    \section{From weight space to function space}

    Consider the optimization problem~(\ref{eq:loss_minimization}).
    The gradient descent dynamics for this problem looks as follows:
    \begin{equation}
        \dot\theta_t =
        -\eta \EE_{x,y \in S_n} \left.\frac{\partial \ell(y,z)}{\partial z}\right|_{z=f(x;\theta_t)} \nabla_\theta f(x;\theta_t).
        \label{eq:gd_dynamics_theta}
    \end{equation}
    This dynamics is very complex due to non-linearity of $f(x;\theta)$ as a function of $\theta$.
    Let us multiply both sides of~(\ref{eq:gd_dynamics_theta}) by $\nabla^T_\theta f(x;\theta_t)$:
    \begin{equation}
        \dot f_t(x') =
        -\eta \EE_{x,y \in S_n} \left.\frac{\partial \ell(y,z)}{\partial z}\right|_{z=f_t(x)} K_t(x',x),
        \label{eq:gd_dynamics_f}
    \end{equation}
    where $f_t(x) = f(x;\theta_t)$, and $K_t(x',x)$ is a \emph{tangent kernel}:
    \begin{equation}
        K_t(x',x) =
        \nabla^T_\theta f(x';\theta_t) \nabla_\theta f(x;\theta_t).
    \end{equation}

    Generally, the kernel is stochastic and evolves with time.
    For this reason, dynamics~(\ref{eq:gd_dynamics_f}) is not completely defined.
    However, if the network is parameterized in a certain way, the kernel $K_t$ converges to a stationary deterministic kernel $\bar K_0$ as the number of hidden units (width) goes to infinity \cite{jacot2018neural}.

    If the kernel is stationary and deterministic, the dynamics~(\ref{eq:gd_dynamics_f}) is much simpler than~(\ref{eq:gd_dynamics_theta}).
    Indeed, for square loss~(\ref{eq:gd_dynamics_f}) is a linear ODE, which can be solved analytically \cite{lee2019wide}, while~(\ref{eq:gd_dynamics_theta}) still remains non-linear.
    
    This allows us to prove several results on convergence and generalization for large enough nets \cite{du2018gradient,arora2019fine}.
    Indeed, for a large enough network, its kernel is almost deterministic, and one have to prove that is is almost stationary.
    Given this, we can transfer results from statinary deterministic kernels of infinitely wide nets to sufficiently wide ones.

    Kernels of realistic finite-sized networks turn out to be non-stationary.
    It is possible to take this effect into account by introducing \emph{finite-width corrections} \cite{huang2019dynamics,Dyer2020Asymptotics}.

    \chapter{Initialization}

    \section{Preserving the variance}

    Consider a network with $L$ hidden layers and no biases:
    \begin{equation}
        f(x) = W_L \phi(W_{L-1} \ldots \phi(W_0 x)),
    \end{equation}
    where $W_l \in \RR^{n_{l+1} \times n_l}$ and a non-linearity $\phi$ is applied element-wise.
    Note that $x \in \RR^{n_0}$; we denote with $k = n_{L+1}$ the dimensionality of the output: $f: \; \RR^{n_0} \to \RR^k$.
    
    We shall refer $n_l$ as the width of the $l$-th hidden layer.
    Denote $n = n_1$, and take $n_l = \alpha_l n$ $\forall l \in [L]$.
    We shall refer $n$ as the width of the network, and keep $\alpha$-factors fixed.

    Consider a loss function $\ell(y,z)$.
    We try to minimize the average loss of our model: $\LL = \EE_{x,y} \ell(y,f(x))$.
    
    Let us assume both $x$ and $y$ are fixed.
    Define:
    \begin{equation}
    h_1 = W_0 x \in \RR^{n_1},
    \quad
    x_l = \phi(h_l) \in \RR^{n_l},
    \quad
    h_{l+1} = W_l x_l \in \RR^{n_{l+1}} \; \forall l \in [L].
    \end{equation}
    Hence given $x$ $f(x) = h_{L+1}$.
    
    This is forward dynamics; we can express backward dynamics similarly.
    Define a loss gradient with respect to the hidden representation:
    \begin{equation}
        g_l = \frac{\partial \ell(y,h_{L+1})}{\partial h_l} \in \RR^{n_l}
        \quad
        \forall l \in [L+1].
    \end{equation}
    We have then:
    \begin{equation}
        g_l = D_l W_l^T g_{l+1}
        \quad
        \forall l \in [L],
        \quad
        g_{L+1} = \frac{\partial \ell(y,h)}{\partial h},
    \end{equation}
    where $D_l = \diag(\phi'(h_l))$.
    
    Using the backward dynamics, we are able to compute gradients wrt the weights:
    \begin{equation}
        \nabla_l =
        \frac{\partial \ell(y,h_{L+1})}{\partial W_l} \in \RR^{n_{l+1} \times n_l}
        \quad
        \forall l \in [L]_0.
    \end{equation}
    Then,
    \begin{equation}
        \nabla_l = 
        g_{l+1} x_l^T
        \quad
        \forall l \in [L]_0.
    \end{equation}
    
    Assume the weights are initialized with zero mean and layer-dependent variance $v_l$:
    \begin{equation}
    \EE W_{l,0}^{ij} = 0,
    \quad
    \Var W_{l,0}^{ij} = v_l.
    \end{equation}

    Note that $\forall l \in [L+1]$ all components of the vector $h_l$ are i.i.d. with zero mean.
    Let $q_l$ be its variance:
    \begin{equation}
        q_l =
        \Var h_l^i =
        \EE (h_l^i)^2 =
        \frac{1}{n_l} \EE h_l^T h_l.
    \end{equation}
    The same holds for $g_l$; we denote its variance by $\delta_l$:
    \begin{equation}
        \delta_l =
        \Var g_l^i =
        \EE (g_l^i)^2 =
        \frac{1}{n_l} \EE g_l^T g_l.        
    \end{equation}

    \subsection{Linear case}
    \label{sec:linear}

    Consider first $\phi(h) = h$.

    Consider the following two properties of the initialization:
    \begin{enumerate}
        \item Normalized forward dynamics: $q_l$ does not depend neither on $n_{0:l-1}$, nor on $l$ $\forall l \in [L+1]$.
        \item Normalized backward dynamics: $\delta_l$ does not depend neither on $n_{l+1:L+1}$, nor on $l$ $\forall l \in [L+1]$.
    \end{enumerate}

    The first property implies that the model stays finite at the initialization irrespective of width $n$ and depth $L$.
    The two properties combined imply finite weight gradients at the initialization:
    $\Var \nabla_l^{ij}$ does not depend neither on $n_{0:L+1}$, nor on $l$ $\forall l \in [L]_0$.
    
    In turn, these two imply that weight increments stay finite at the initialization irrespective of width $n$ and depth $L$ if we consider training with SGD:
    \begin{equation}
        \Delta W_l =
        -\eta \EE_{x,y} \nabla_l(x,y).            
    \end{equation}

    Since the initial weights have zero mean, all hidden representations $h_l$ have zero mean too, and due to initial weight independence:
    \begin{equation}
        q_{l+1} =
        \frac{1}{n_{l+1}} \EE h_l^T W_l^T W_l h_l =
        v_l \EE h_l^T h_l =
        n_l v_l q_l
        \quad
        \forall l \in [L],
        \qquad
        q_1 =
        \frac{1}{n_1} \EE x^T W_0^T W_0 x =
        \|x\|_2^2 v_0 \propto
        n_0 v_0.
    \end{equation}
    Hence forward dynamics is normalized if $v_l = n_l^{-1}$ $\forall l \in [L]_0$.

    We can compute variances for gradients wrt hidden representations in a similar manner:
    \begin{equation}
        \delta_l =
        \frac{1}{n_l} \EE g_{l+1}^T W_l W_l^T g_{l+1} =
        v_l \EE g_{l+1}^T g_{l+1} =
        n_{l+1} v_l \delta_{l+1}
        \quad
        \forall l \in [L-1],
    \end{equation}
    \begin{equation}
        \delta_L =
        \frac{1}{n_L} \EE g_{L+1}^T W_L W_L^T g_{L+1} =
        \left\|\frac{\partial \ell(y,h)}{\partial h}\right\|_2^2 v_L \propto 
        n_{L+1} v_L.
    \end{equation}
    As we see, backward dynamics is normalized if $v_l = n_{l+1}^{-1}$ $\forall l \in [L]_0$.
    This means that we cannot have both forward dynamics and backward dynamics normalized at the same time.
    \cite{glorot2010understanding} proposed taking a harmonic mean of the variances for the two normalization requirements:
    \begin{equation}
        v_l =
        \frac{2}{n_l + n_{l+1}}
        \quad
        \forall l \in [L]_0.
    \end{equation}

    \subsection{ReLU case}
    \label{sec:relu}

    We start with the forward dynamics:
    \begin{equation}
        q_{l+1} =
        \frac{1}{n_{l+1}} \EE x_l^T W_l^T W_l x_l =
        v_l \EE x_l^T x_l
        \quad
        \forall l \in [L],
        \qquad
        q_1 =
        \frac{1}{n_1} \EE x^T W_0^T W_0 x =
        \|x\|_2^2 v_0 \propto
        n_0 v_0.
    \end{equation}

    \begin{equation}
        \EE x_l^T x_l =
        \EE [h_l]_+^T [h_l]_+ =
        \frac{1}{2} \EE h_l^T h_l =
        \frac{1}{2} n_l q_l
        \quad
        \forall l \in [L].
    \end{equation}
    Here the second equality holds due to the symmetry of $h_l$ distribution.
    The latter in its turn holds by induction on $l$.

    Hence for ReLU the forward dynamics is normalized if $v_l = 2 n_l^{-1}$, a result first presented in \cite{he2015delving}.
    Let us consider the backward dynamics then:
    \begin{equation}
        \delta_l =
        \frac{1}{n_l} \EE g_{l+1}^T W_l D_l^2 W_l^T g_{l+1} =
        \frac{1}{2} v_l \EE g_{l+1}^T g_{l+1} =
        \frac{1}{2} n_{l+1} v_l \delta_{l+1}
        \quad
        \forall l \in [L-1],
    \end{equation}
    \begin{equation}
        \delta_L =
        \frac{1}{n_L} \EE g_{L+1}^T W_L D_L^2 W_L^T g_{L+1} =
        \frac{1}{2} v_L \EE g_{L+1}^T g_{L+1} =
        \frac{1}{2} \left\|\frac{\partial \ell(y,h)}{\partial h}\right\|_2^2 v_L \propto 
        \frac{1}{2} n_{L+1} v_L.
    \end{equation}
    Similarly, we have to take $v_l = 2 n_{l+1}^{-1}$ for this type of normalization.
    Note that here we have assumed that $g_{l+1}$ does not depend on $W_l$ and $h_l$, which is not true: $g_{l+1}$ depends on $h_{l+1}$ through $D_{l+1}$ which depends on both $W_l$ and $h_l$.

    Again, we have a contradiction between the two normalization requirements.
    However in some practical cases satisfying only one of these is enough.
    For instance, if we consider minimizing the cross-entropy loss, the model diverging or vanishing at the initialization does not break the optimization process.
    Moreover, the magnitude of hidden representations does not matter, thanks to homogeneity of ReLU.
    Hence in this case normalizing the forward dynamics is not a strict requirement.

    On the other hand, using an optimizer that normalizes the gradient, i.e. Adam, makes backward normalization unnecessary.

    \subsection{Tanh case}
    \label{sec:tanh}

    Assume $\phi \in C^3(\RR)$, $\phi'(z) > 0$, $\phi(0) = 0$, $\phi'(0) = 1$, $\phi''(0) = 0$, $\phi'''(0) < 0$, and $\phi$ is bounded.
    The guiding example is the hyperbolic tangent: 
    \begin{equation}
        \phi(z) = \frac{e^z - e^{-z}}{e^z + e^{-z}}.
    \end{equation}

    In this case taking $v_l = n_l^{-1}$ ensures that activations $x_l$ are neither in a linear regime ($\Var h_l$ are not too small), nor in a saturated regime ($\Var h_l$ are not too large).
    However, if we take $v_l = n_{l+1}^{-1}$, $\Var g_l$ still vanishes for large $l$ since $|\phi'(h)| \leq 1$.
    Nevertheless, \cite{glorot2010understanding} suggests initializing with a harmonic mean of variances for the class of non-linearities we consider.
    Rationale: in this case a network is almost linear at the initialization.

    Let us assume that $v_l = \sigma_w^2 / n_l$.
    Consider the forward dynamics in detail:
    \begin{equation}
        q_{l+1} =
        \frac{1}{n_{l+1}} \EE_{h_l} \EE_{W_l} \phi(h_l)^T W_l^T W_l \phi(h_l) =
        \frac{\sigma_w^2}{n_l} \EE_{h_l} \phi(h_l)^T \phi(h_l).
    \end{equation}

    By the Cenral Limit Theorem, $\forall i$ $h_l^i$ converges to $\NN(0, q_l)$ in distribution as $n_{l-1} \to \infty$.
    Note that for a fixed $x$ $h_1$ is always normally distributed.
    Hence by taking subsequent limits $n_1 \to \infty$, $n_2 \to \infty$, and so on, we come up with the following recurrent relation (see \cite{poole2016exponential}):
    \begin{equation}
        q_{l+1} =
        \sigma_w^2 \EE_{z \sim \NN(0,1)} \phi(\sqrt{q_l} z)^2 =
        \VV(q_l | \sigma_w^2),
        \quad
        q_1 =
        \sigma_w^2 \frac{\|x\|_2^2}{n_0}.
    \end{equation}

    Let us study properties of the length map $\VV$:
    \begin{equation}
        \VV'(q | \sigma_w^2) =
        \sigma_w^2 \EE_{z \sim \NN(0,1)} \phi(\sqrt{q} z) \phi'(\sqrt{q} z) z / \sqrt{q} > 0.
    \end{equation}
    The last inequality holds since $\phi(\sqrt{q} z) z > 0$ for $z \neq 0$ due to monotonicity of $\phi$ and since $\phi(0) = 0$.
    Hence $\VV$ monotonically increases.
    \begin{equation}
        \VV'(q | \sigma_w^2) =
        \sigma_w^2 \EE_{z \sim \NN(0,1)} \phi(\sqrt{q} z) \phi'(\sqrt{q} z) z / \sqrt{q} =
        \sigma_w^2 \EE_{z \sim \NN(0,1)} \left(\phi'(\sqrt{q} z)^2 + \phi(\sqrt{q} z) \phi''(\sqrt{q} z)\right).
    \end{equation}
    In particular,
    \begin{equation}
        \VV'(0 | \sigma_w^2) =
        \sigma_w^2 \EE_{z \sim \NN(0,1)} (\phi'(0))^2 =
        \sigma_w^2 > 0.
    \end{equation}

    \cite{poole2016exponential} claimed that the second derivative is always negative for $\phi$ being a hyperbolic tangent, which we were not able to show.
    We can check it for $q = 0$ though:
    \begin{equation}
        \VV''(0 | \sigma_w^2) =
        4 \sigma_w^2 \EE_{z \sim \NN(0,1)} \phi'(0) \phi'''(0) =
        4 \sigma_w^2 \phi'''(0) < 0.
    \end{equation}
    Hence at least, $\VV$ is concave at zero.

    Whenever $\sigma_w \leq 1$, $q = 0$ is a stable fixed point for the length map.
    However for $\sigma_w > 1$ $q = 0$ becomes unstable; since $\VV(\infty | \sigma_w^2) < \infty$ due to boundedness of $\phi$, there should be at least one stable fixed point for the length map.
    If we believe that $\VV$ is indeed concave everywhere, this stable fixed point is unique.
    We denote it as $q_\infty$.

    This means that assuming $L = \infty$, $q_l = \Var h_l$ has finite non-zero limit as $n \to \infty$ and $l \to \infty$ whenever $\sigma_w^2 > 1$.
    We shall refer this property as \emph{asymptotically normalized forward dynamics.}
    Note that asymptotic and non-asymptotic forward dynamics normalizations are equivalent for linear and ReLU nets, and they hold exactly for $\sigma_w^2 = 1$ and $\sigma_w^2 = 2$, respectively.

    Let us proceed with backward dynamics.
    Similar to the forward case, we have:
    \begin{equation}
        \delta_l =
        \frac{1}{n_l} \EE g_{l+1}^T W_l \diag(\phi'(h_l))^2 W_l^T g_{l+1}.
    \end{equation}
    We cannot factorize the expectation since $g_{l+1}$ depends on $W_{0:l}$ unless $\phi'$ is constant.
    Nevertheless, assume that $g_{l+1}$ does not depend on $W_{0:l}$.
    Hence it does not depend on $h_l$, and we have the following:
    \begin{multline}
        \delta_l \approx
        \frac{1}{n_l} \EE_{g_{l+1}} (g_{l+1}^T \EE_{W_l} (W_l \EE_{h_l} \diag(\phi'(h_l))^2 W_l^T) g_{l+1}) =
        \frac{1}{n_l} \EE_{h \sim \NN(0,q_l)} (\phi'(h))^2 \EE_{g_{l+1}} (g_{l+1}^T \EE_{W_l} (W_l W_l^T) g_{l+1}) =\\=
        \frac{\sigma_w^2}{n_l} \EE_{z \sim \NN(0,1)} (\phi'(\sqrt{q_l} z))^2 \EE_{g_{l+1}} g_{l+1}^T g_{l+1} =
        \sigma_w^2 \frac{\alpha_{l+1}}{\alpha_l} \delta_{l+1} \EE_{z \sim \NN(0,1)} (\phi'(\sqrt{q_l} z))^2.
    \end{multline}

    We have already noted that given concavity of $\VV$ the latter has a single unique stable point $q_\infty$; this also implies $q_l \to q_\infty$ as $l \to \infty$.
    \cite{poole2016exponential} has noted that convergence to $q_\infty$ is fast; assume $q_l = q_\infty$ then.
    This allows us to express the dynamics solely in terms of $\delta_l$:
    \begin{equation}
        \delta_l =
        \sigma_w^2 \frac{\alpha_{l+1}}{\alpha_l} \delta_{l+1} \EE_{z \sim \NN(0,1)} (\phi'(\sqrt{q_\infty} z))^2.
    \end{equation}
    
    For simplicity assume $\alpha_l = 1$ $\forall l \geq 1$ (all matrices $W_{1:L+1}$ are square).
    Define:
    \begin{equation}
        \chi_1 =
        \sigma_w^2 \EE_{z \sim \NN(0,1)} (\phi'(\sqrt{q_\infty} z))^2.
    \end{equation}
    We get (see \cite{schoenholz2016deep}):
    \begin{equation}
        \delta_l = \chi_1 \delta_{l+1}.
    \end{equation}
    Obviously, $\chi_1 > 1$ implies exploding gradients, while $\chi_1 < 1$ causes gradients to vanish.
    We shall refer the case $\chi_1 = 1$ as \emph{asymptotically normalized backward dynamics.}
    Note that for linear and ReLU nets $\chi_1 = 1$ is equivalent to $\sigma_w^2 = 1$ and $\sigma_w^2 = 2$, respectively.

    \subsubsection{Correlation stability}
    \label{sec:corr}

    The term $\chi_1$ has a remarkable interpretation in terms of correlation stability (see \cite{poole2016exponential}).
    Consider two inputs, $x^1$ and $x^2$, together with their hidden representations $h_l^1$ and $h_l^2$.
    Define the terms of the covariance matrix for the latter two:
    \begin{equation}
        \Sigma_l =
        \begin{pmatrix}
            q_l^{11} & q_l^{12} \\
            q_l^{12} & q_l^{22}
        \end{pmatrix};
        \qquad
        q_l^{ab} =
        \frac{1}{n_l} \EE h_l^{a,T} h_l^b,
        \quad
        a, b \in \{1,2\}.
    \end{equation}

    Consider a correlation factor $c_l^{12} = q_l^{12} / \sqrt{q_l^{11} q_l^{22}}$.
    We have already derived the dynamics for the diagonal terms in the subsequent limits of infinite width:
    \begin{equation}
        q_{l+1}^{aa} =
        \sigma_w^2 \EE_{z \sim \NN(0,1)} \phi(\sqrt{q_l^{aa}} z)^2,
        \quad
        q_1^{aa} =
        \sigma_w^2 \frac{\|x^a\|_2^2}{n_0},
        \quad
        a \in \{1,2\}.
    \end{equation}
    Consider the diagonal term:
    \begin{equation}
        q_{l+1}^{12} =
        \frac{1}{n_{l+1}} \EE_{h_l^1,h_l^2} \EE_{W_l} \phi(h_l^1)^T W_l^T W_l \phi(h_l^2) =
        \frac{\sigma_w^2}{n_l} \EE_{h_l^1,h_l^2} \phi(h_l^1)^T \phi(h_l^2).
    \end{equation}
    Taking the same subsequent limits as before, we get:
    \begin{equation}
        q_{l+1}^{12} =
        \sigma_w^2 \EE_{(u^1,u^2)^T \sim \NN(0,\Sigma_l)} \phi(u^1) \phi(u^2) =
        \sigma_w^2 \EE_{(z_1,z_2)^T \sim \NN(0,I)} \phi(u_l^1(z^1)) \phi(u_l^2(z^1,z^2)) =
        \CC(c_l^{12}, q_l^{11}, q_l^{22} | \sigma_w^2),
    \end{equation}
    where $u_l^1 = \sqrt{q_l^{11}} z^1$, while $u_l^2 = \sqrt{q_l^{22}} (c_l^{12} z^1 + \sqrt{1 - (c_l^{12})^2} z^2)$.
    We shall refer $\CC$ as a correlation map.

    As before, assume that $q_l^{aa} = q_\infty$, $a \in \{1,2\}$, $\forall l$.
    This assumption results in a self-consistent dynamics of the correlation factor:
    \begin{equation}
        c_{l+1}^{12} =
        q_\infty^{-1} \CC(c_l^{12}, q_\infty, q_\infty | \sigma_w^2).
    \end{equation}

    Note that $c^{12} = 1$ is a fixed point of the $c$-dynamics.
    Indeed:
    \begin{equation}
        c^{12} =
        q_\infty^{-1} \CC(1, q_\infty, q_\infty | \sigma_w^2) =
        q_\infty^{-1} \sigma_w^2 \EE_{z \sim \NN(0,1)} \phi(\sqrt{q_\infty} z)^2 =
        q_\infty^{-1} \VV(q_\infty | \sigma_w^2) =
        1.
    \end{equation}
    In order to study its stability, we have to consider a derivative of the $\CC$-map at $c^{12} = 1$.
    Let us compute the derivative for a $c^{12} < 1$ first:
    \begin{equation}
        \left.\frac{\partial c_{l+1}^{12}}{\partial c_l^{12}}\right|_{c_l^{12} = c} =
        q_\infty^{-1} \sigma_w^2 \EE_{(z^1,z^2)^T \sim \NN(0,I)} \phi(\sqrt{q_\infty} z^1) \phi'(\sqrt{q_\infty} (c z^1 + \sqrt{1 - c^2} z^2)) (\sqrt{q_\infty} (z^1 - z^2 c / \sqrt{1 - c^2})).
    \end{equation}
    We shall use the following equivalence:
    \begin{equation}
        \EE_{z \sim \NN(0,1)} F(z) z =
        \int_{-\infty}^{+\infty} F(z) z e^{-z^2/2} \, dz =
        \int_{-\infty}^{+\infty} (-F(z)) \, d e^{-z^2/2} =
        \int_{-\infty}^{+\infty} F'(z) e^{-z^2/2} \, dz =
        \EE_{z \sim \NN(0,1)} F'(z).
    \end{equation}
    We begin the integration with analyzing one of the parts of this equation:
    \begin{equation}
        \EE_{z \sim \NN(0,1)} \phi'(\sqrt{q_\infty} (c z^1 + \sqrt{1 - c^2} z)) \sqrt{q_\infty} z c / \sqrt{1 - c^2} =
        q_\infty \EE_{z \sim \NN(0,1)} \phi''(\sqrt{q_\infty} (c z^1 + \sqrt{1 - c^2} z)) c.
    \end{equation}
    Henceforth,
    \begin{equation}
        \left.\frac{\partial c_{l+1}^{12}}{\partial c_l^{12}}\right|_{c_l^{12} = c} =
        q_\infty^{-1} \sigma_w^2 \EE_{z^1 \sim \NN(0,1)} \phi(u^1(z^1)) \EE_{z^2 \sim \NN(0,1)} (\sqrt{q_\infty} z^1 \phi'(u^2(z^1,z^2)) - q_\infty c \phi''(u^2(z^1,z^2))),
    \end{equation}
    where $u^1 = \sqrt{q_\infty} z^1$, while $u^2 = \sqrt{q_\infty} (c z^1 + \sqrt{1 - c^2} z^2)$.
    Consider the limit of $c \to 1$:
    \begin{equation}
        \lim_{c \to 1} \left.\frac{\partial c_{l+1}^{12}}{\partial c_l^{12}}\right|_{c_l^{12} = c} =
        q_\infty^{-1} \sigma_w^2 \EE_{z \sim \NN(0,1)} \phi(\sqrt{q_\infty} z) (\sqrt{q_\infty} z \phi'(\sqrt{q_\infty} z) - q_\infty \phi''(\sqrt{q_\infty} z)).        
    \end{equation}
    Let us compute the first term first:
    \begin{equation}
        \EE_{z \sim \NN(0,1)} \phi(\sqrt{q_\infty} z) \sqrt{q_\infty} z \phi'(\sqrt{q_\infty} z) =
        q_\infty \EE_{z \sim \NN(0,1)} \left((\phi'(\sqrt{q_\infty} z))^2 + \phi(\sqrt{q_\infty} z) \phi''(\sqrt{q_\infty} z) \right).
    \end{equation}
    This gives the final result:
    \begin{equation}
        \lim_{c \to 1} \left.\frac{\partial c_{l+1}^{12}}{\partial c_l^{12}}\right|_{c_l^{12} = c} =
        \sigma_w^2 \EE_{z \sim \NN(0,1)} (\phi'(\sqrt{q_\infty} z))^2 =
        \chi_1.
    \end{equation}

    We see that $\chi_1$ drives the stability of the correlation of strongly correlated hidden representations, or, equivalently, of nearby input points.
    For $\chi_1 < 1$ nearby points with $c^{12} \approx 1$ become more correlated as they propagate through the layers.
    Hence initially different points become more and more similar.
    We refer this regime as \emph{ordered.}
    In contrast, for $\chi_1 > 1$ nearby points separate as they propagate deeper in the network.
    We refer this regime as \emph{chaotic.}
    Hence the case of $\chi_1 = 1$ is \emph{the edge of chaos.}

    \section{Dynamical stability}
    \label{sec:dynamical_stability}

    Following \cite{pennington2017resurrecting}, let us turn our attention to the input-output jacobian:
    \begin{equation}
        J = 
        \frac{\partial h_{L+1}}{\partial h_1} =
        \prod_{l=1}^L W_l D_l \in \RR^{n_{L+1} \times n_1}.
    \end{equation}

    We now compute the mean square Frobenius norm of $J^T J \in \RR^{n_1 \times n_1}$:
    \begin{multline}
        \EE \|J^T J\|_F^2 =
        \EE \tr(J^T J) =
        \EE_{W_{0:L}} \tr\left( \left(\prod_{l=1}^L W_l D_l \right)^T \prod_{l=1}^L W_l D_l \right) =
        \tr\left( \EE_{W_{0:L}} \left(\left(\prod_{l=1}^L W_l D_l \right)^T \prod_{l=1}^L W_l D_l \right)\right) =\\=
        \tr\left(\EE_{W_{0:L-1}} \left(\left(\prod_{l=1}^{L-1} W_l D_l \right)^T D_L \EE_{W_L} (W_L^T W_L) D_L \prod_{l=1}^{L-1} W_l D_l \right)\right) =\\=
        n_{L+1} v_L \tr\left(\EE_{W_{0:L-1}} \left(\left(\prod_{l=1}^{L-1} W_l D_l \right)^T D_L^2 \prod_{l=1}^{L-1} W_l D_l \right)\right).
    \end{multline}
    Assuming that $\tr(D_l^2)$ does not depend on $W_{0:l}$ $\forall l \in [L]$ allows us to proceed with calculations:
    \begin{multline}
        \EE \|J^T J\|_F^2 =
        n_{L+1} v_L \EE_{h_L} \tr(D_L^2) v_{L-1} \tr\left(\EE_{W_{0:L-2}} \left(\left(\prod_{l=1}^{L-2} W_l D_l \right)^T D_{L-1}^2 \prod_{l=1}^{L-2} W_l D_l \right)\right) =\\=
        n_{L+1} v_L \prod_{l=2}^L \left(\EE_{h_l} \tr(D_l^2) v_{l-1}\right) \tr\left(\EE_{W_0} D_1^2 \right) =
        n_{L+1} \prod_{l=1}^{L} v_l \EE_{h_l} \tr(D_l^2).
    \end{multline}

    Suppose we aim to normalize the backward dynamics: $v_l = \sigma_w^2 / n_{l+1}$ $\forall l \in [L]$.
    Assume then (see Section~\ref{sec:tanh}) $h_l \sim \NN(0, q_\infty)$ $\forall l \in [L]$.
    Then the calculation above gives us the mean average eigenvalue of $J^T J$:
    \begin{equation}
        \frac{1}{n_1} \EE \sum_{i=1}^{n_1} \lambda_i =
        \frac{1}{n_1} \EE \|J^T J\|_F^2 =
        \sigma_w^{2L} \left(\EE_{z\sim\NN(0,1)} \phi'(\sqrt{q_\infty} z)\right)^L =
        \chi_1^L.
    \end{equation}
    Hence $\chi_1^L$ is the mean average eigenvalue of the input-ouput jacobian of the network of depth $L$.

    Let us assume that our non-linearity is homogeneous: $\phi(\beta z) = \beta \phi(z)$.
    This property holds for leaky ReLU with arbitrary slope; in particular, it holds in the linear case.
    Then we have the following:
    \begin{equation}
        h_{L+1} = 
        J h_1;
        \quad
        q_{L+1} =
        \frac{1}{n_{L+1}} \EE \|J h_1\|_2^2 =
        \frac{1}{n_{L+1}} h_1^T (\EE J^T J) h_1 =
        \frac{1}{n_{L+1}} \EE \sum_{i=1}^{n_1} \lambda_i (v_i^T h_1)^2.
    \end{equation}
    \begin{equation}
        g_1 =
        J^T g_{L+1};
        \quad
        \delta_1 =
        \frac{1}{n_1} \EE \|J^T g_{L+1}\|_2^2 =
        \frac{1}{n_1} g_{L+1}^T (\EE J J^T) g_{L+1} =
        \frac{1}{n_1} \EE \sum_{i=1}^{n_1} \lambda_i (u_i^T g_{L+1})^2.
    \end{equation}
    One can perceive $q_{L+1}$ as a mean normalized squared length of the network output.
    We may want to study a \emph{distribution} of normalized squared lengths instead.

    In this case it suffices to study a distribution of the empirical spectral density:
    \begin{equation}
        \hat\rho(x) =
        \frac{1}{n_1} \sum_{i=1}^{n_1} \delta(x - \lambda_i).
    \end{equation}
    Besides being random, it converges to a deterministic limiting spectral density $\rho$ as $n \to \infty$ if we assume $n_l = \alpha_l n$ $\forall l \in [L+1]$ with $\alpha_l$ being constant.

    Assume all matrices $W_l$ are square: $n_1 = \ldots = n_{L+1} = n$.
    In this case the choice of $v_l = 1/n$ normalizes both forward and backward dynamics.
    On the other hand, in the linear case the limiting spectrum can be parameterized as (see \cite{pennington2017resurrecting}):
    \begin{equation}
        \lambda(\phi) =
        \frac{\sin^{L+1}((L+1)\phi)}{\sin\phi \sin^L(L\phi)}.
    \end{equation}
    We shall prove this result in the upcoming section.
    Notice that $\lim_{\phi\to 0} \lambda(\phi) = (L+1)^{L+1} / L^L \sim e (L+1)$ for large $L$.
    Hence in this case despite we preserve lengths of input vectors on average, some of the input vectors get expanded with positive probability during forward propagation, while some get contracted.
    The same holds for the backward dynamics.

    \subsection{Linear case}

    Our goal is to compute a limiting spectrum of the matrix $J J^T \in \RR^{n \times n}$ with $J = \prod_{l=1}^L W_l$ with all $W_l$ composed of i.i.d. gaussians with variance $1/n$; it is referred as \emph{product Wishart ensemble}.
    The case of $L=1$, $W W^T$, is known as \emph{Wishart ensemble}.
    The limiting spectrum of the Wishart ensemble is known as Marchenko-Pastur law \cite{marchenko1967}:
    \begin{equation}
        \rho_{W W^T}(x) =
        \frac{1}{2 \pi} \sqrt{\frac{4}{x} - 1} \Ind_{[0,4]}(x).
        \label{eq:wishart_density}
    \end{equation}

    It is possible to derive a limiting spectrum for $J J^T$ by using the so-called \emph{S-transform}, which we shall define later.
    A high-level algorithm is the following.
    First, we compute an S-transform for the Wishart ensemble:
    \begin{equation}
        S_{W W^T}(z) = 
        \frac{1}{1 + z}.
        \label{eq:wishart_s}
    \end{equation}
    The S-transform has a following fundamental property.
    Given two \emph{asymptotically free} random matrices $A$ and $B$, we have \cite{voiculescu1987multiplication}\footnote{see also \url{https://mast.queensu.ca/~speicher/survey.html}}:
    \begin{equation}
        S_{AB} =
        S_A S_B
        \label{eq:s_mult}
    \end{equation}
    in the limit of $n \to \infty$.

    As we shall see later, the S-transform of $J_L = \prod_{l=1}^L W_L$ depends only on traces of the form $n^{-1} \tr(J_L^k)$ which are invariant under cyclic permutations of matrices $W_l$.
    This allows us to compute $S_{J J^T}$:
    \begin{equation}
        S_{J J^T} =
        S_{J_L J_L^T} =
        S_{W_L^T W_L J_{L-1} J_{L-1}^T} =
        S_{W_L^T W_L} S_{J_{L-1} J_{L-1}^T} =
        \prod_{l=1}^L S_{W_l^T W_l} =
        S^L_{W^T W}.
        \label{eq:product_wishart_s}
    \end{equation}
    The last equation holds since all $W_l$ are distributed identically.
    The final step is to recover the spectrum of $J J^T$ from its S-transform.

    \paragraph{Free independence.}
    
    We say that $A$ and $B$ are \emph{freely independent}, or just \emph{free}, if:
    \begin{equation}
        \tau((P_1(A) - \tau(P_1(A))) (Q_1(B) - \tau(Q_1(B))) \ldots (P_k(A) - \tau(P_k(A))) (Q_k(B) - \tau(Q_k(B)))) = 0,
    \end{equation}
    where $\forall i \in [k]$ $P_i$ and $Q_i$ are polynomials, while $\tau(A) = n^{-1} \EE \tr(A)$ --- an analogue of the expectation for scalar random variables.
    Note that $\tau$ is a linear operator and $\tau(I) = 1$.
    Compare above with the definition of classical independence:
    \begin{equation}
        \tau((P(A) - \tau(P(A))) (Q(B) - \tau(Q(B)))) = 0,
    \end{equation}
    for all polynomials $P$ and $Q$.

    Note that two scalar-valued random variables are free iff one of them is constant; indeed:
    \begin{equation}
        \EE ((\xi - \EE \xi) (\eta - \EE \eta) (\xi - \EE \xi) (\eta - \EE \eta)) =
        \EE ((\xi - \EE \xi)^2 (\eta - \EE \eta)^2) =
        (\EE (\xi - \EE \xi)^2) (\EE(\eta - \EE \eta)^2) =
        \Var \xi \Var \eta.
    \end{equation}
    Hence having $\Var \xi = 0$ or $\Var \eta = 0$ is necessary; this implies $\xi = \const$ or $\eta = \const$, which gives free independence.
    
    This means that the notion of free independence is too strong for scalar random variables.
    The reason for this is their commutativity; only non-commutative objects can have a non-trivial notion of free independence.
    As for random matrices with classically independent entries, they have a remarkable property that they become free in the limit of $n \to \infty$:
    \begin{equation}
        \lim_{n \to \infty} \tau((P_1(A_n) - \tau(P_1(A_n))) (Q_1(B_n) - \tau(Q_1(B_n))) \ldots (P_k(A_n) - \tau(P_k(A_n))) (Q_k(B_n) - \tau(Q_k(B_n)))) = 0,
    \end{equation}
    for $A_n$ and $B_n \in \RR^{n \times n}$ such that the moments $\tau(A_n^k)$ and $\tau(B_n^k)$ are finite for large $n$ for $k \in \mathbb{N}$.
    We shall sat that the two sequences $\{A_n\}$ and $\{B_n\}$ are asymptotically free as $n \to \infty$.

    \paragraph{Asymptotic free independence for Wigner matrices.}

    In order to illustrate the above property, consider $X$ and $Y$ being classically independent $n \times n$ Wigner matrices, i.e. $X_{ij} = X_{ji} \sim \NN(0, n^{-1})$, and similarly for $Y$.
    Of course, $\tau(X) = \tau(Y) = 0$, while $\tau (X^2 Y^2) = n^{-1} \tr(\EE X^2 \EE Y^2) = n^{-1} \tr(I) = 1$.
    Let us compute $\tau (XYXY)$:
    \begin{multline}
        \tau (XYXY) =
        \frac{1}{n} \EE X_{ij} Y^{jk} X_{kl} Y^{li} =
        \frac{1}{n^3} (
            (\delta_{ik} \delta_{jl} + \delta_{il} \delta_{jk}) (\delta^{jl} \delta^{ki} + \delta^{ji} \delta^{kl}) - C n
        ) =\\=
        \frac{1}{n^3} (n^2 + (3 - C) n) =
        O_{n\to\infty} (n^{-1}).
    \end{multline}
    This means that $X$ and $Y$ are not freely independent, however, it suggests that they become free in the limit of large $n$.

    \paragraph{A sum of freely independent random matrices.}

    Before moving to the definition of the S-transform used for finding the product, we discuss a simpler topic of finding the distribution of the sum of freely independent random matrices.

    Let $\xi$ and $\eta$ be scalar-valued independent random variables.
    The density of their sum can be computed using a charasteric function:
    \begin{equation}
        F_{\xi+\eta}(t) =
        \EE e^{i (\xi+\eta) t} =
        \EE \left(e^{i \xi t} e^{i \eta t}\right) =
        \left(\EE e^{i \xi t}\right) \cdot \left(\EE e^{i \eta t}\right) =
        F_\xi(t) + F_\eta(t).
        \label{eq:F_of_sum}
    \end{equation}
    The first equality is a definition of the charasteric function.
    The third equlaity holds due to independence of $\xi$ and $\eta$.
    A (generalized) density of their sum can be computed by taking the inverse Fourier transform:
    \begin{equation}
        p_{\xi+\eta}(x) =
        \frac{1}{2\pi} \int_{\RR} e^{-i x t} F_{\xi+\eta}(t) \, dt.
    \end{equation}

    Let $X$ and $Y$ be random matrix ensembles of sizes $n \times n$.
    We cannot apply the same technique to random matrices since they do not generally commute; for this reason, $e^{i (X+Y) t} \neq e^{i X t} e^{i Y t}$ generally, and the second equality of~(\ref{eq:F_of_sum}) does not hold.
    However, there exists a related technique for freely independent random matrices.

    Following \cite{tao2012topics}, define the Stieltjes transform as:
    \begin{equation}
        G_X(z) = 
        \tau((z-X)^{-1}),
    \end{equation}
    where $\tau(X) = n^{-1} \EE \tr(X)$.
    This allows for formal Laurent series which give the following:
    \begin{equation}
        G_X(z) = 
        \sum_{k=0}^\infty \frac{\tau(X^k)}{z^{k+1}} =
        \sum_{k=0}^\infty \frac{n^{-1} \EE_X \sum_{i=1}^n \lambda_i(X)^k}{z^{k+1}} =
        \sum_{k=0}^\infty \frac{\EE_X \EE_{\lambda \sim \hat\rho_X} \lambda^k}{z^{k+1}} =
        \sum_{k=0}^\infty \frac{\EE_{\lambda \sim \rho_X} \lambda^k}{z^{k+1}} =
        \EE_{\lambda \sim \rho_X} (z-\lambda)^{-1},
    \end{equation}
    where $\rho_X(\lambda)$ denotes the expected spectral desnity:
    \begin{equation}
        \rho_X(\lambda) = 
        \EE_X \hat\rho_X(\lambda) = 
        \frac{1}{n} \EE_X \sum_{i=1}^n \delta(\lambda - \lambda_i(X)).
    \end{equation}

    Let $\zeta = G_X(z) = \tau((z-X)^{-1})$.
    Here $\zeta$ is a function of $z$; let assume that $z$ is a function of $\zeta$: $z = z_X(\zeta)$.
    We have the following then:
    \begin{equation}
        (z_X(\zeta)-X)^{-1} =
        \zeta (1 - E_X),
    \end{equation}
    where $\tau(E_X) = 0$.
    Rearranging gives:
    \begin{equation}
        X =
        z_X(\zeta) - \zeta^{-1} (1 - E_X)^{-1},
    \end{equation}
    while for $Y$ we have the same:
    \begin{equation}
        Y =
        z_Y(\zeta) - \zeta^{-1} (1 - E_Y)^{-1},
    \end{equation}
    and so:
    \begin{equation}
        X + Y =
        z_X(\zeta) + z_Y(\zeta) - \zeta^{-1} ((1 - E_X)^{-1} + (1 - E_Y)^{-1}).
    \end{equation}
    We have:
    \begin{multline}
        (1 - E_X)^{-1} + (1 - E_Y)^{-1} =
        (1 - E_X)^{-1} (1 - E_X + 1 - E_Y) (1 - E_Y)^{-1} =\\=
        (1 - E_X)^{-1} ((1 - E_X) (1 - E_Y) + 1 - E_X E_Y) (1 - E_Y)^{-1} =
        1 + (1 - E_X)^{-1} (1 - E_X E_Y) (1 - E_Y)^{-1}.
    \end{multline}
    Hence:
    \begin{equation}
        (z_X(\zeta) + z_Y(\zeta) - X - Y - \zeta^{-1})^{-1} =
        \zeta (1 - E_Y) (1 - E_X E_Y)^{-1} (1 - E_X).
    \end{equation}
    We have:
    \begin{equation}
        (1 - E_Y) (1 - E_X E_Y)^{-1} (1 - E_X) =
        (1 - E_Y) \sum_{k=0}^\infty (E_X E_Y)^k (1 - E_X).
    \end{equation}
    The last expression is a sum of alternating products of $E_X$ and $E_Y$.
    Since $X$ and $Y$ are freely independent, $E_X$ and $E_Y$ are freely independent too.
    Applying $\tau$ gives:
    \begin{equation}
        \tau((z_X(\zeta) + z_Y(\zeta) - X - Y - \zeta^{-1})^{-1}) =
        \zeta.
    \end{equation}
    At the same time:
    \begin{equation}
        \tau((z_{X+Y}(\zeta) - X - Y)^{-1}) =
        \zeta.
    \end{equation}
    Hence:
    \begin{equation}
        z_{X+Y}(\zeta) =
        z_X(\zeta) + z_Y(\zeta) - \zeta^{-1}.
    \end{equation}
    Define $R_X(\zeta) = z_X(\zeta) - \zeta^{-1}$.
    Hence:
    \begin{equation}
        R_{X+Y}(\zeta) =
        R_X(\zeta) + R_Y(\zeta).
        \label{eq:R_additivity}
    \end{equation}
    Alternatively, we can say that the R-transform is a solution of the following equation:
    \begin{equation}
        R_X(G_X(z)) + (G_X(z))^{-1} =
        z.
    \end{equation}

    As a sanity check, consider the R-transform of a scalar constant $x$.
    In this case, $G_x(z) = (z - x)^{-1}$.
    This gives $R_x((z-x)^{-1}) + z - x = z$, hence $R_x((z-x)^{-1}) = x$.
    This means simply $R_x \equiv x$.

    \paragraph{S-transform.}

    Let us now define the S-transform.
    We start with defining the moment generating function $M$:
    \begin{equation}
        M(z) =
        z G(z) - 1 =
        \sum_{k=1}^\infty \frac{\tau(X^k)}{z^{k}} =
        \sum_{k=1}^\infty \frac{\EE_{\lambda \sim \rho_X} \lambda^k}{z^{k}} =
        \sum_{k=1}^\infty \frac{m_k(X)}{z^k},
    \end{equation}
    where the $k$-th moment of $\rho$ is defined as follows:
    \begin{equation}
        m_k(X) = 
        \EE_{\lambda \sim \rho_X} \lambda^k =
        \tau(X^k).
    \end{equation}
    The moment generating function $M$ is a mapping from $\Cplx \setminus \{0\}$ to $\Cplx$.
    Let $M^{-1}$ be its functional inverse.
    We are now ready to define the S-transform:
    \begin{equation}
        S(z) =
        \frac{1+z}{z M^{-1}(z)}.
    \end{equation}

    In order to get some intuition concerning the property~(\ref{eq:s_mult}), we consider the case $\rho(\lambda) = \delta(\lambda - x)$.
    In this case $M(z) = x / (z - x)$; hence $z = x (1 + 1 / M(z))$.
    This gives $M^{-1}(z) = x (1 + 1 / z)$, and $S(z) = 1/x$, which obviously satisfies the property.

    \paragraph{Recovering the limiting spectrum.}

    We are not going to compute the S-transform~(\ref{eq:wishart_s}) of the Wishart ensemble~(\ref{eq:wishart_density}), but we aim to recover the spectrum of the product Wishart enesmble from its S-transform~(\ref{eq:product_wishart_s}).
    We have:
    \begin{equation}
        S_{J J^T}(z) =
        S_{W^T W}^L(z) =
        \frac{1}{(1+z)^L},
        \qquad
        M_{J J^T}^{-1}(z) =
        \frac{(1+z)^{L+1}}{z}.
    \end{equation}
    First we need to recover the Stieltjes transform $G$.
    Recall $M_{J J^T}(z) = z G_{J J^T}(z) - 1$.
    This gives:
    \begin{equation}
        z =
        M_{J J^T}^{-1}(M_{J J^T}(z)) =
        \frac{(z G_{J J^T}(z))^{L+1}}{z G_{J J^T}(z) - 1},
    \end{equation}
    or:
    \begin{equation}
        z G_{J J^T}(z) - 1 =
        z^L G_{J J^T}(z)^{L+1}.
        \label{eq:G_eq}
    \end{equation}
    This equation gives a principle way to recover $G_{J J^T}$.
    However, our goal is the spectral density $\rho_{J J^T}$.
    The density can be recovered from its Stieltjes transform using the \emph{inversion formula:}
    \begin{equation}
        \rho(\lambda) =
        -\frac{1}{\pi} \lim_{\epsilon \to 0+} \Im G(\lambda + i \epsilon).
    \end{equation}
    Indeed:
    \begin{multline}
        \lim_{\epsilon \to 0+} \Im G(\lambda + i \epsilon) =
        \lim_{\epsilon \to 0+} \Im \int \frac{\rho(t)}{\lambda - t + i \epsilon} \, dt =
        \lim_{\epsilon \to 0+} \Im \int \frac{\rho(t)}{(\lambda - t)^2 + \epsilon^2} (\lambda - t - i \epsilon) \, dt =\\=
        \lim_{\epsilon \to 0+} \int \frac{(-\epsilon) \rho(t)}{(\lambda - t)^2 + \epsilon^2} \, dt =
        \lim_{\epsilon \to 0+} \int \frac{(-\epsilon) \rho(u + \lambda)}{u^2 + \epsilon^2} \, du =
        \lim_{\epsilon \to 0+} \int \frac{(-1) \rho(v \epsilon + \lambda)}{v^2 + 1} \, dv =
        -\pi \rho(\lambda).
    \end{multline}
    Hence we should consider $z = \lambda + i \epsilon$ and take the limit of $\epsilon \to 0+$.
    Assume also $G_{J J^T}(\lambda + i \epsilon) = r e^{i \phi}$.
    Substituting it to~(\ref{eq:G_eq}) gives:
    \begin{equation}
        (\lambda + i \epsilon) r e^{i \phi} - 1 =
        (\lambda + i \epsilon)^L r^{L+1} e^{i (L+1) \phi}.
    \end{equation}
    
    Let us consider the real and imaginary parts of this equation separately:
    \begin{equation}
        r(\lambda \cos\phi + O(\epsilon)) - 1 =
        \lambda^L r^{L+1} ((1 + O(\epsilon^2)) \cos((L+1)\phi) + O(\epsilon));
    \end{equation}
    \begin{equation}
        r(\lambda \sin\phi + O(\epsilon)) =
        \lambda^L r^{L+1} ((1 + O(\epsilon^2)) \sin((L+1)\phi) + O(\epsilon)).
    \end{equation}
    Taking the limit of $\epsilon \to 0+$ gives:
    \begin{equation}
        r \lambda \cos\phi - 1 =
        \lambda^L r^{L+1} \cos((L+1)\phi),
        \qquad
        r \lambda \sin\phi =
        \lambda^L r^{L+1} \sin((L+1)\phi).
    \end{equation}
    Consequently:
    \begin{equation}
        \frac{r \lambda \sin\phi}{r \lambda \cos\phi - 1} =
        \tan((L+1)\phi),
        \qquad
        r^L =
        \lambda^{1-L} \frac{\sin\phi}{\sin((L+1)\phi)}.
    \end{equation}
    From the first equality we get:
    \begin{equation}
        r =
        \lambda^{-1} \frac{1}{\cos\phi - \sin\phi / \tan((L+1)\phi)} =
        \lambda^{-1} \frac{\sin((L+1)\phi)}{\sin(L\phi)}.
    \end{equation}
    This equality together with the second on the previous line give:
    \begin{equation}
        1 =
        \lambda \frac{\sin\phi}{\sin((L+1)\phi)} \frac{\sin^L(L\phi)}{\sin^L((L+1)\phi)}.
    \end{equation}
    Hence:
    \begin{equation}
        \lambda =
        \frac{\sin^{L+1}((L+1)\phi)}{\sin\phi \sin^L(L\phi)}.
    \end{equation}
    We also get the density:
    \begin{equation}
        \rho(\lambda) =
        -\frac{1}{\pi} r \sin\phi =
        -\frac{1}{\pi} \frac{\sin^2\phi \sin^{L-1}(L\phi)}{\sin^{L}((L+1)\phi)}.
    \end{equation}
    For the sake of convenience, we substitute $\phi$ with $-\phi$; this gives:
    \begin{equation}
        \lambda =
        \frac{\sin^{L+1}((L+1)\phi)}{\sin\phi \sin^L(L\phi)},
        \qquad
        \rho(\lambda(\phi)) =
        \frac{1}{\pi} \frac{\sin^2\phi \sin^{L-1}(L\phi)}{\sin^{L}((L+1)\phi)}.
    \end{equation}
    All eigenvalues of $J J^T$ are real and non-negative.
    This gives us a constraint: $\phi \in [0, \pi / (L+1)]$.
    The left edge of this segment gives a maximal $\lambda = (L+1)^{L+1} / L^L$, while the right edge gives a minimum: $\lambda = 0$.
    Note that the same constraint results in non-negative spectral density.

    As a sanity check, take $L=1$ and compare with~(\ref{eq:wishart_density}):
    \begin{equation}
        \lambda =
        \frac{\sin^{2}(2\phi)}{\sin^2\phi} =
        4 \cos^2 \phi,
        \qquad
        \rho(\lambda(\phi)) =
        \frac{1}{\pi} \frac{\sin^2\phi}{\sin(2\phi)} =
        \frac{1}{2\pi} \tan\phi =
        \frac{1}{2\pi} \sqrt{\frac{1}{\cos^2\phi} - 1} =
        \frac{1}{2\pi} \sqrt{\frac{4}{\lambda(\phi)} - 1}.
    \end{equation}

    \subsection{ReLU case}

    For gaussian initialization, we expect similar problems as we had for linear case.
    However, curing the expanding jacobian spectrum for a linear net with square layers is easy: one have to assume orthogonal initialization instead of i.i.d. gaussian:
    \begin{equation}
        W_l \sim U(O_{n \times n})
        \quad
        \forall l \in [L].
    \end{equation}
    In this case $\|J h_1\| = \|h_1\|$ a.s; the same holds for $\|J^T g_{L+1}\|$.
    The goal of the current section is to check whether orthogonal initialization helps in the ReLU case.

    Similarly to the linear case, we have:
    \begin{multline}
        S_{J J^T} =
        S_{J_L J_L^T} =
        S_{D_L W_L^T W_L D_L J_{L-1} J_{L-1}^T} =
        S_{D_L W_L^T W_L D_L} S_{J_{L-1} J_{L-1}^T} =\\=
        S_{D^2_L W_L W_L^T} S_{J_{L-1} J_{L-1}^T} =
        S_{D^2_L} S_{W_L W_L^T} S_{J_{L-1} J_{L-1}^T} =
        \prod_{l=1}^L S_{D^2_l} S_{W_l W_l^T} =
        S^L_{W W^T} \prod_{l=1}^L S_{D^2_l}.
        \label{eq:product_wishart_s_nonlinear}
    \end{multline}
    Consider orthogonal initialization.
    In order to normalize forward and backward dynamics, we have to introduce a factor $\sigma_w = \sqrt{2}$:
    \begin{equation}
        W_l \sim \sigma_w U(O_{n \times n})
        \quad
        \forall l \in [L].
    \end{equation}
    For a scaled orthogonal matrix $W$ $S_{W W^T} = S_{\sigma_w^2 I} \equiv \sigma_w^{-2} = 1/2$.
    We have to compute $S_{D^2_{l+1}}$ then.

    Since we have assumed that $\forall l \in [L]$ $h_l \sim \NN(0, q_\infty)$, the spectrum of $D^2$ is given simply as:
    \begin{equation}
        \rho_{D^2_{l+1}}(x) =
        \frac{1}{2} \delta(x) + \frac{1}{2} \delta(x - 1).
    \end{equation}
    Taking the Stieltjes transform we get:
    \begin{equation}
        G_{D^2_{l+1}}(z) =
        \frac{1}{2} \left(\frac{1}{z} + \frac{1}{z-1}\right).
    \end{equation}
    This gives the moment generating function and its inverse:
    \begin{equation}
        M_{D^2_{l+1}}(z) =
        \frac{1}{2(z-1)},
        \qquad
        M^{-1}_{D^2_{l+1}}(z) =
        \frac{1}{2z} + 1.
    \end{equation}
    Finally, we get the S-transform:
    \begin{equation}
        S_{D^2_l}(z) =
        \frac{z+1}{z M^{-1}_{D^2_{l+1}}(z)} =
        \frac{z+1}{z+1/2}.
    \end{equation}
    The S-transform of $J J^T$ is then given as:
    \begin{equation}
        S_{J J^T} =
        \sigma_w^{-2L} \left(\frac{z+1}{z+1/2}\right)^L.
    \end{equation}
    \begin{equation}
        M^{-1}_{J J^T} =
        \sigma_w^{2L} \frac{(z+1/2)^L}{z (z+1)^{L-1}}.
    \end{equation}
    Recall $M_{J J^T}(z) = z G_{J J^T}(z) - 1$.
    This gives:
    \begin{equation}
        z =
        M_{J J^T}^{-1}(M_{J J^T}(z)) =
        \sigma_w^{2L} \frac{(z G_{J J^T}(z) - 1/2)^L}{(z G_{J J^T}(z) - 1) (z G_{J J^T}(z))^{L-1}},
    \end{equation}
    or:
    \begin{equation}
        z (z G_{J J^T}(z) - 1) (z G_{J J^T}(z))^{L-1} =
        (2 z G_{J J^T}(z) - 1)^L.
        \label{eq:G_eq_nonlinear}
    \end{equation}
    Taking its imaginary part gives a sequence of transformations:
    \begin{equation}
        \lambda^2 r \sin\phi (\lambda^{L-1} r^{L-1} \cos((L-1)\phi)) + \lambda (\lambda r \cos\phi - 1) (\lambda^{L-1} r^{L-1} \sin((L-1)\phi)) =
        L (2\lambda r \cos\phi - 1)^{L-1} 2\lambda r \sin\phi + O(\sin^2\phi).
    \end{equation}
    \begin{equation}
        \lambda^{L+1} r^L \sin(L\phi) - \lambda^L r^{L-1} \sin((L-1)\phi) =
        L (2\lambda r \cos\phi - 1)^{L-1} 2\lambda r \sin\phi + O(\sin^2\phi).
    \end{equation}
    \begin{equation}
        \lambda^{L+1} r^L L \sin\phi - \lambda^L r^{L-1} (L-1) \sin\phi =
        L (2\lambda r - 1)^{L-1} 2\lambda r \sin\phi + O(\sin^2\phi).
    \end{equation}
    Hence for $\phi=0$ we have:
    \begin{equation}
        \lambda^{L+1} r^L L - \lambda^L r^{L-1} (L-1) =
        L (2\lambda r - 1)^{L-1} 2\lambda r =
        L (2\lambda r - 1)^L + L (2\lambda r - 1)^{L-1}.
        \label{eq:G_eq_nonlinear_im}
    \end{equation}
    A real part of~(\ref{eq:G_eq_nonlinear}) in its turn gives:
    \begin{equation}
        \lambda (\lambda r \cos\phi - 1) (\lambda^{L-1} r^{L-1} \cos((L-1)\phi)) - \lambda^2 r \sin\phi (\lambda^{L-1} r^{L-1} \sin((L-1)\phi)) =
        (2 \lambda r \cos\phi - 1)^L + O(\sin^2\phi).
    \end{equation}
    \begin{equation}
        \lambda^{L+1} r^L \cos(L\phi) - \lambda^L r^{L-1} \cos((L-1)\phi) =
        (2 \lambda r \cos\phi - 1)^L + O(\sin^2\phi).
    \end{equation}
    \begin{equation}
        \lambda^{L+1} r^L - \lambda^L r^{L-1} =
        (2 \lambda r - 1)^L + O(\sin^2\phi).
    \end{equation}
    Hence for $\phi=0$ we have:
    \begin{equation}
        \lambda^{L+1} r^L - \lambda^L r^{L-1} =
        (2 \lambda r - 1)^L.
        \label{eq:G_eq_nonlinear_real}
    \end{equation}
    Eq.~(\ref{eq:G_eq_nonlinear_im}) $-L \times$ eq.~(\ref{eq:G_eq_nonlinear_real}) results in:
    \begin{equation}
        \lambda^L r^{L-1} =
        L (2\lambda r - 1)^{L-1}.
        \label{eq:G_eq_nonlinear_temp1}
    \end{equation}
    Putting this to~(\ref{eq:G_eq_nonlinear_real}) gives:
    \begin{equation}
        L (2\lambda r - 1)^{L-1} (\lambda r - 1) =
        (2 \lambda r - 1)^L.
    \end{equation}
    \begin{equation}
        L (\lambda r - 1) =
        2 \lambda r - 1.
    \end{equation}
    \begin{equation}
        \lambda r =
        \frac{L-1}{L-2}.
    \end{equation}
    Putting this back to~(\ref{eq:G_eq_nonlinear_temp1}) gives:
    \begin{equation}
        \lambda \left(\frac{L-1}{L-2}\right)^{L-1} =
        L \left(\frac{L}{L-2}\right)^{L-1}.
    \end{equation}
    \begin{equation}
        \lambda =
        L \left(\frac{L}{L-1}\right)^{L-1} =
        L \left(1 + \frac{1}{L-1}\right)^{L-1}.
    \end{equation}
    The last equation is equivalent to $e L$ for large $L$.
    Hence the spectral density $\rho_{J J^T}$ gets expanded with depth at least linearly in ReLU case even for orthogonal initialization.
    
    \section{GD dynamics for orthogonal initialization}

    It seems natural that a well-conditioned jacobian is necessary for trainability.
    But does a well-conditioned jacobian \emph{ensure} trainability?
    In fact, yes, in linear case.
    Following  \cite{saxe2013exact} we will show that for a linear net with $L$ hidden layers initialized orthogonally and trained with square loss, the number of optimization steps required to reach the minimum does not depend on $L$ for large $L$.

    \paragraph{Shallow nets.}

    In order to show this, we start with the case of $L=1$:
    \begin{equation}
        f(x) =
        W_1 W_0 x.
    \end{equation}
    Consider square loss:
    \begin{equation}
        \ell(y,z) =
        \frac{1}{2} \|y - z\|_2^2,
        \quad
        \LL = 
        \EE_{x,y} \ell(y,f(x)).
    \end{equation}
    Gradient descent step:
    \begin{equation}
        \dot W_0 =
        \eta \EE_{x,y} W_1^T (y x^T - W_1 W_0 x x^T),
        \quad
        \dot W_1 =
        \eta \EE_{x,y} (y x^T - W_1 W_0 x x^T) W_0^T.
    \end{equation}
    Define $\Sigma_{xx} = \EE x x^T$ --- input correlation matrix, and $\Sigma_{xy} = \EE y x^T$ --- input-output correlation matrix.
    Assume then that the data is whitened: $\Sigma_{xx} = I$.
    Consider an SVD decomposition for the input-output correlation:
    \begin{equation}
        \Sigma_{xy} =
        U_2 S_{2,0} V_0^T =
        \sum_{r=1}^{n} s_r u_r v_r^T.
    \end{equation}
    Perform a change of basis:
    \begin{equation}
        \bar W_1 =
        U_2^T W_1,
        \quad
        \bar W_0 =
        W_0 V_0.
    \end{equation}
    Gradient descent step becomes:
    \begin{equation}
        \dot{\bar W}_0 =
        \eta \bar W_1^T (S_{2,0} - \bar W_1 \bar W_0),
        \quad
        \dot{\bar W}_1 =
        \eta (S_{2,0} - \bar W_1 \bar W_0) \bar W_0^T.
    \end{equation}

    Note that while the matrix element $W_{0,ij}$ connects a hidden neuron $i$ to an input neuron $j$, the matrix element $\bar W_{0,i\alpha}$ connects a hidden neuron $i$ to an input mode $\alpha$.
    Let $\bar W_0 = [a_1, \ldots, a_n]$, while $\bar W_1 = [b_1, \ldots, b_n]^T$.
    Then we get:
    \begin{equation}
        \frac{1}{\eta} \dot a_{\alpha} =
        s_\alpha b_{\alpha} - \sum_{\gamma=1}^n b_\gamma (b_\gamma^T a_\alpha) =
        (s_\alpha - (b_\alpha^T a_\alpha)) b_{\alpha} - \sum_{\gamma\neq\alpha} (b_\gamma^T a_\alpha) b_\gamma;
    \end{equation}
    \begin{equation}
        \frac{1}{\eta} \dot b_{\alpha} =
        s_\alpha a_{\alpha} - \sum_{\gamma=1}^n (a_\gamma^T b_\alpha) a_\gamma =
        (s_\alpha - (a_\alpha^T b_\alpha)) a_{\alpha} - \sum_{\gamma\neq\alpha} (a_\gamma^T b_\alpha) a_\gamma.
    \end{equation}
    This dynamics is a GD dynamics on the following energy function:
    \begin{equation}
        E =
        \frac{1}{2} \sum_{\alpha=1}^n (s_\alpha - a_\alpha b_\alpha)^2 + \frac{1}{2} \sum_{\alpha\neq\gamma} (a_\alpha b_\gamma)^2.
    \end{equation}

    Let assume that there exists an orthogonal matrix $R = [r_1, \ldots, r_n]$ such that $a_\alpha \propto r_\alpha$ and $b_\alpha \propto r_\alpha$.
    In other words, $\bar W_0 = R D_0$ and $\bar W_1 = D_1 R^T$, where $D_0$ and $D_1$ are diagonal matrices.
    Note that in this case $W_1 = U_2 D_1 R^T$, while $W_0 = R D_0 V_0^T$.

    Given this, the dynamics above decomposes into a system of independent equations of the same form:
    \begin{equation}
        \dot a =
        \eta (s - a b) b,
        \qquad
        \dot b =
        \eta (s - a b) a.
    \end{equation}
    Note that $a^2 - b^2$ is a motion integral, while the energy function for each individual equation depends only on $ab$: $E = (s - ab)^2 / 2$.

    There exists a solution for these equations that admits $a = b$.
    In this case $D_0 = D_1$.
    Let $u = ab$.
    We have:
    \begin{equation}
        \dot u = 
        2\eta (s - u) u.
    \end{equation}
    This ODE is integrable:
    \begin{equation}
        t = 
        \frac{1}{\eta} \int_{u_0}^{u_f} \frac{du}{2u (s - u)} =
        \frac{1}{2 s \eta} \int_{u_0}^{u_f} \left(\frac{du}{u} + \frac{du}{s-u}\right) =
        \frac{1}{2 s \eta} \left(\ln\left(\frac{u_f}{u_0}\right) - \ln\left(\frac{u_f-s}{u_0-s}\right)\right) =
        \frac{1}{2 s \eta} \ln\left(\frac{u_f (u_0-s)}{u_0 (u_f-s)}\right).
    \end{equation}
    Note that $u = s$ is a global minimizer.
    Hence the time required to achieve $u_f = s - \epsilon$ from $u_0 = \epsilon$ is:
    \begin{equation}
        t =
        \frac{1}{2 s \eta} \ln\left(\frac{(s-\epsilon)^2}{\epsilon^2}\right) =
        \frac{1}{s \eta} \ln(s/\epsilon - 1) \sim
        \frac{1}{s \eta} \ln(s/\epsilon) \; \text{for $\epsilon \to 0$}.
    \end{equation}
    This means that the larger the correlation $s$ between input and output modes $a$ and $b$, the faster convergence is.

    \paragraph{Deep nets.}
    
    Let us proceed with a linear network with $L$ hidden in the same setup:
    \begin{equation}
        f(x) =
        \left(\prod_{l=0}^L W_l\right) x.
    \end{equation}
    Gradient descent step:
    \begin{equation}
        \dot W_l =
        \eta \EE_{x,y} \left(\prod_{l'=l+1}^L W_{l'}\right)^T \left(y x^T - \left(\prod_{l'=0}^L W_{l'}\right) x x^T\right) \left(\prod_{l'=0}^{l-1} W_{l'}\right)^T
        \quad
        \forall l \in [L]_0.
    \end{equation}
    Again, assume that $\Sigma_{xx} = 1$ and $\Sigma_{xy} = U_{L+1} S_{L+1,0} V_0^T$.
    Moreover, in analogy to the shallow case suppose $W_l = R_{l+1} D_l R_l^T$ for $l \in [L]_0$, where $D_l$ is a diagonal matrix, while $R_l$ are orthogonal; $R_0 = V_0$, $R_{L+1} = U_{L+1}$.
    Note that if all $W_l$ are themselves orthogonal, and $\prod_{l=0}^L W_l = U_{L+1} V_0^T$, then the assumption above holds for $D_l = I$ $\forall l \in [L]_0$, $R_0 = V_0$, $R_{l+1} = W_l R_l$.
    This gives:
    \begin{equation}
        \dot D_l =
        \eta \left(\prod_{l'=l+1}^L D_{l'}\right)^T \left(S_{L+1,0} - \left(\prod_{l'=0}^L D_{l'}\right)\right) \left(\prod_{l'=0}^{l-1} D_{l'}\right)^T
        \quad
        \forall l \in [L]_0.
    \end{equation}
    The latter decouples into independent modes:
    \begin{equation}
        \dot a_l =
        \eta \left(s - \prod_{l'=0}^L a_{l'}\right) \prod_{l' \neq l} a_{l'},
    \end{equation}
    which is a gradient descent for the following energy function:
    \begin{equation}
        E(a_{0:L}) =
        \frac{1}{2} \left(s - \prod_{l=0}^L a_l\right)^2.
    \end{equation}

    Again, we are looking for solutions of the form $a_0 = \ldots = a_L$.
    Define $u = \prod_{l=0}^L a_l$.
    This gives an ODE:
    \begin{equation}
        \dot u =
        \eta (L+1) u^{2L/(L+1)} (s - u).
    \end{equation}
    For large $L$ we can approximate this equation with $\dot u = \eta (L+1) u^2 (s - u)$ \emph{(why?)} which is easily integrable:
    \begin{equation}
        t = 
        \frac{1}{(L+1) \eta} \int_{u_0}^{u_f} \frac{du}{u^2 (s - u)} =
        \frac{1}{(L+1) s \eta} \int_{u_0}^{u_f} \left(\frac{du}{u^2} + \frac{du}{u (s-u)}\right) =
        \frac{1}{(L+1) s \eta} \left(\frac{1}{u_0} - \frac{1}{u_f} + \frac{1}{s} \ln\left(\frac{u_f (u_0-s)}{u_0 (u_f-s)}\right)\right).
    \end{equation}

    We see that $t \sim L^{-1}$: training time decreases as the number of layers grows.
    Note that we cannot perform a gradient flow; we perform a gradient descent with discrete steps instead.
    Hence we have to count the number of steps as a function of $L$.

    The optimal learning rate is inversely proportional to the maximum eigenvalue of the Hessian of the energy function observed during training.
    Let us first compute the Hessian:
    \begin{equation}
        \nabla_i :=
        \frac{\partial E}{\partial a_i} =
        -\left(s - \prod_{l=0}^L a_l\right) \prod_{l \neq i} a_l.
    \end{equation}
    \begin{equation}
        \nabla^2_{ij} :=
        \frac{\partial^2 E}{\partial a_i \partial a_j} =
        \left(\prod_{l \neq i} a_l\right) \left(\prod_{l \neq j} a_l\right) - \left(s - \prod_{l=0}^L a_l\right) \prod_{l \neq i,j} a_l
        \quad \text{for $i \neq j$}.
    \end{equation}
    \begin{equation}
        \nabla^2_{ii} :=
        \frac{\partial^2 E}{\partial a_i^2} =
        \left(\prod_{l \neq i} a_l\right)^2.
    \end{equation}
    Taking into account our assumption $a_0 = \ldots = a_L = a$, we get:
    \begin{equation}
        \nabla_i =
        -(s - a^{L+1}) a^L,
        \quad
        \nabla^2_{ij} =
        2 a^{2L} - s a^{L-1},
        \quad
        \nabla^2_{ii} =
        a^{2L}.
    \end{equation}
    
    There is an eigenvector $v_1 = [1,\ldots,1]^T$ of value $\lambda_1 = \nabla^2_{ii} + L \nabla^2_{ij} = (1 + 2L) a^{2L} - s L a^{L-1}$.
    Also, there are $L$ eigenvectors of the form $v_i = [1, 0, \ldots, 0, -1, 0, \ldots, 0]$ of value $\lambda_i = \nabla^2_{ii} - \nabla^2_{ij} = s a^{L-1} - a^{2L}$.
    Notice that for large $L$ $\lambda_1$ becomes the largest eigenvalue irrespective of $a$.

    During the scope of optimization $u$ travels inside the segment $[0, s]$, hence $a$ lies inside $\left[0, s^{1/(L+1)}\right]$.
    Let us find the maximum of $\lambda_1$ on this segment:
    \begin{equation}
        \frac{d\lambda_1}{da} =
        2L (1+2L) a^{2L-1} - s L (L-1) a^{L-2}.
    \end{equation}
    Equating this derivative to zero yields:
    \begin{equation}
        a^* = 
        \left(\frac{s (L-1)}{2 (1+2L)}\right)^{1/(L+1)} =
        s^{1/(L+1)} \left(\frac{L-1}{2 (1+2L)}\right)^{1/(L+1)} <
        s^{1/(L+1)}.
    \end{equation}
    The second solution is, of course, $a=0$ if $L > 2$.
    Therefore we have three candidates for being a maximum: $a=0$, $a=s^{1/(L+1)}$, and $a=a^*$.
    Let us check them:
    \begin{equation}
        \lambda_1(0) =
        0,
        \qquad
        \lambda_1(s^{1/(L+1)}) =
        (1+L) s^{2L/(L+1)} \geq 0.
    \end{equation}
    \begin{multline}
        \lambda_1(a^*) = 
        s^{2L/(L+1)} (1+2L) \left(\frac{L-1}{2 (1+2L)}\right)^{2L/(L+1)} - s L s^{(L-1)/(L+1)} \left(\frac{L-1}{2 (1+2L)}\right)^{(L-1)/(L+1)}
        =\\=
        s^{2L/(L+1)} \left(\frac{1}{2 (1+2L)}\right)^{(L-1)/(L+1)} \left((L-1)^{2L/(L+1)} - L (L-1)^{(L-1)/(L+1)}\right)
        =\\=
        -s^{2L/(L+1)} \left(\frac{1}{2 (1+2L)}\right)^{(L-1)/(L+1)} (L-1)^{(L-1)/(L+1)}
        =\\=
        -s^{2L/(L+1)} \left(\frac{L-1}{2 (1+2L)}\right)^{(L-1)/(L+1)} \leq 0.
    \end{multline}
    Hence the maximal $\lambda_1$ during the scope of optimization is $\lambda_1(s^{1/(L+1)}) = (1+L) s^{2L/(L+1)}$.
    Recall the optimal learning rate is proportional to maximal eigenvalue of the Hessian:
    \begin{equation}
        \eta_{opt} \propto
        \frac{1}{\max_t \lambda_1} =
        (L+1)^{-1} s^{-2L/(L+1)}.
    \end{equation}
    Substituting it to $t$ yiels:
    \begin{equation}
        t_{opt} =
        \frac{1}{(L+1) s \eta_{opt}} \left(\frac{1}{u_0} - \frac{1}{u_f} + \frac{1}{s} \ln\left(\frac{u_f (u_0-s)}{u_0 (u_f-s)}\right)\right) =
        s^{(L-1)/(L+1)} \left(\frac{1}{u_0} - \frac{1}{u_f} + \frac{1}{s} \ln\left(\frac{u_f (u_0-s)}{u_0 (u_f-s)}\right)\right).
    \end{equation}
    This equation asymptotically does not depend on $L$.
    In other words, training time (in terms of the number of gradient steps) for very deep nets does not depend on depth.

    \chapter{Loss landscape}

    Neural network training process can be viewed as an optimization problem:
    \begin{equation}
        \LL(\theta) =
        \EE_{x,y \in \hat S_m} \ell(y,f(x;\theta)) \to \min_\theta,
    \end{equation}
    where $\ell$ is a loss function assumed to be convex, $f(\cdot;\theta)$ is a neural net with parameters $\theta$, and $\hat S_m = \{(x_i,y_i)\}_{i=1}^m$ is a dataset of size $m$ sampled from the data distribution $\DD$.

    This problem is typically non-convex, hence we do not have any guarantees for gradient descent convergence in general.
    Nevertheless, in realistic setups we typically observe that gradient descent \emph{always} succeeds in finding the global minimum of $\LL_m$; moreover, this is done in reasonable time.
    This observation leads us to the following hypotheses:
    \begin{enumerate}
        \item If the neural network and the data satisfy certain properties, all local minima of $\LL_m(\theta)$ are global.
        \item If the neural network and the data satisfy certain properties, gradient descent converges to a global minimum in a good rate with high probability.
    \end{enumerate}

    Neither of the two hypotheses are stronger than the other.
    Indeed, having all local minima being global does not tell us anything about convergence rate, while having convergence guarantee with high probability does not draw away the possibility to have (few) local minima.

    In the present chapter, we shall discuss the first hypothesis only, while the second one will be discussed later in the context of a Neural Tangent Kernel.

    \section{Wide non-linear nets}

    It turns out that if the training data is consistent, one can prove globality of local minima if the network is wide enough.

    Following \cite{yu1995local}, we shall start with a simplest case of a two-layered net trained to minimize the square loss:
    \begin{equation}
        f(x; W_{0,1}) =
        W_1 \phi(W_0 x);
    \end{equation}
    \begin{equation}
        \LL(W_{0,1}) =
        \frac{1}{2} \sum_{i=1}^m \|y_i - f(x_i; W_{0,1})\|_2^2 =
        \frac{1}{2} \|Y - W_1 \phi(W_0 X)\|_F^2,
    \end{equation}
    where $W_l \in \RR^{n_{l+1} \times n_l}$, $x_i \in \RR^{n_0}$, $y_i \in \RR^{n_2}$, $X \in \RR^{n_0 \times m}$, and $Y \in \RR^{n_2 \times m}$.

    Let $W_{0,1}^*$ be a local minimum of $\LL$.
    Consider $\LL$ with $W_0$ fixed to $W_0^*$:
    \begin{equation}
        \LL_{W_0^*}(W_1) =
        \frac{1}{2} \|Y - W_1 \phi(W_0^* X)\|_F^2.
    \end{equation}
    Since $W_{0,1}^*$ is a minimum of $\LL$, $W_1^*$ is a minimum of $\LL_{W_0^*}$.
    Minimizing $\LL_{W_0^*}(W_1)$ is a convex problem.
    Hence $W_1^*$ is a global minimum of $\LL_{W_0^*}$.
    
    Denote $H_1 = W_0 X$ and $X_1 = \phi(H_1)$; then $\LL_{W_0^*}(W_1) = \frac{1}{2m} \|Y - W_1 X_1^*\|_F^2$.
    Hence $\rk X_1^* = m$ implies $\min \LL_{W_0^*}(W_1) = 0$.
    Since $W_1^*$ is a global minimum of $\LL_{W_0^*}$, $\LL(W_{0,1}^*) = \LL_{W_0^*}(W_1^*) = 0$; hence $W_{0,1}^*$ is a global minimum of $\LL$.

    Suppose $\rk X_1^* < m$.
    If we still have $\min \LL_{W_0^*}(W_1) = 0$, we arrive at the same conclusion as previously.
    Suppose $\LL(W_{0,1}^*) = \LL_{W_0^*}(W_1^*) = \min \LL_{W_0^*}(W_1) > 0$.
    We shall prove that $W_{0,1}^*$ cannot be a minimum of $\LL$ in this case, as long as conditions of the following lemma hold:
    \begin{lemma}
        \label{lemma:full_rank_almost_surely}
        Suppose $\phi$ is non-zero real analytic.
        If $n_1 \geq m$ and $\forall i \neq j$ $x_i \neq x_j$ then $\mu(\{W_0: \; \rk X_1 < m\}) = 0$, where $\mu$ is a Lebesgue measure on $\RR^{n_1 \times n_0}$.
    \end{lemma}

    Since $\LL(W_{0,1}^*) > 0$ and $\LL$ is a continuous function of $W_{0,1}$, $\exists \epsilon > 0:$ $\forall W_{0,1} \in B_\epsilon(W_{0,1}^*)$ $\LL(W_{0,1}) > 0$.
    By the virtue of the lemma, $\forall \delta > 0$ $\exists W_0' \in B_\delta(W_0^*):$ $\rk X_1' \geq m$.

    Take $\delta \in (0,\epsilon)$.
    In this case $\LL(W_0',W_1^*) > 0$, while $\rk X_1' \geq m$.
    Note that minimizing $\LL_{W_0'}(W_1)$ is a convex problem and $\min \LL_{W_0'}(W_1) = 0$ (since $\rk X_1' \geq m$).
    Hence a (continuous-time) gradient descent on $\LL_{W_0'}$ that starts from $W_1^*$ converges to a point $W_1^{*,\prime}$ for which $\LL_{W_0'}(W_1^{*,\prime}) = 0$.
    Because of the latter, $(W_0',W_1^{*,\prime}) \notin B_\epsilon(W_{0,1}^*)$.

    Overall, we have the following: $\exists \epsilon > 0:$ $\forall \delta \in (0,\epsilon)$ $\exists (W_0',W_1^*) \in B_\delta(W_{0,1}^*):$ a continuous-time gradient descent on $\LL$ that starts from $(W_0',W_1^*)$ and that acts only on $W_1$ converges to a point $(W_0',W_1^{*,\prime}) \notin B_\epsilon(W_{0,1}^*)$.

    Obviously, we can replace "$\forall \delta \in (0,\epsilon)$" with "$\forall \delta > 0$".
    Given this, the statement above means that the gradient flow dynamics that acts only on $W_1$ is unstable in Lyapunov sense at $W_{0,1}^*$.
    Hence $W_{0,1}^*$ cannot be a minimum, and hence having $\min \LL(W_{0,1}) > 0$ is impossible as long as the conditions of Lemma~\ref{lemma:full_rank_almost_surely} hold.
    This means that all local minima of $\LL$ are global.

    Let us prove Lemma~\ref{lemma:full_rank_almost_surely}.
    Let $I_m \subset [n_1]$ and $|I_m| = m$.
    Consider $X_{1,I_m} \in \RR^{m \times m}$ --- a subset of rows of $X_1$ indexed by $I_m$.
    Note that $\rk X_1 < m$ is equivalent to $\det X_{1,I_m} = 0$ $\forall I_m$.
    
    Since $\phi$ is analytic, $\det X_{1,I_m}$ is an analytic function of $W_0$ $\forall I_m$.
    We shall use the following lemma:
    \begin{lemma}
        \label{lemma:full_rank_exists}
        Given the conditions of Lemma~\ref{lemma:full_rank_almost_surely}, $\exists W_0:$ $\rk X_1 = m$.
    \end{lemma}
    Given this, $\exists W_0:$ $\exists I_m:$ $\det X_{1,I_m} \neq 0$.
    Since the determinant is an analytic function of $W_0$, $\mu(\{W_0: \; \det X_{1,I_m} = 0\}) = 0$.
    This implies the statement of Lemma~\ref{lemma:full_rank_almost_surely}.

    \subsection{Possible generalizations}

    Let us list the properties we have used to prove the theorem of~\cite{yu1995local}:
    \begin{enumerate}
        \item The loss is square.
        \item The number of hidden layers $L$ is one.
        \item $n_L \geq m$.
        \item $\phi$ is real analytic.
    \end{enumerate}
    Can we relax any of them?
    First, note that it is enough to have "$\rk X_1^* = m$ implies $\min \LL_{W_0^*}(W_1) = 0$".
    For this it is enough to have convex $\ell(y,z)$ with respect to $z$ with $\min_z \ell(y,z) = 0$ $\forall y$ (in particular, minimum should exist).
    For this reason, cross-entropy loss should require a more sophisticated analysis.

    In order to relax the second property, it suffices to generalize Lemma~\ref{lemma:full_rank_almost_surely}:
    \begin{lemma}
        \label{lemma:full_rank_almost_surely_deep}
        Suppose $\phi$ is non-zero real analytic and $l \in [L]$.
        If $n_l \geq m$ and $\forall i \neq j$ $x_i \neq x_j$ then $\mu(\{W_{0:l-1}: \; \rk X_l < m\}) = 0$.
    \end{lemma}
    The generalized version is proven in~\cite{nguyen2017loss}.

    As for the third property, we may want to relax it in two directions.
    First, we may require not the last hidden layer, but some hidden layer to be wide enough.
    Second, we may try to make the lower bound on the number of hidden units smaller.
    It seems like the second direction is not possible for a general dataset $S_m$: one have to assume some specific properties of the data in order to improve the lower bound.

    \subsubsection{Deep nets with analytic activations}
    
    Following \cite{nguyen2017loss}, let us elaborate the first direction.
    We start with defining forward dynamics:
    \begin{equation}
        H_{l+1} = W_l X_l,
        \quad
        X_l = \phi(H_l) 
        \quad 
        \forall l \in [L],
        \qquad
        H_1 = W_0 X,
    \end{equation}
    where $X \in \RR^{n_0 \times m}$, $W_l \in \RR^{n_{l+1} \times n_l}$, $H_l \in \RR^{n_l \times m}$.
    We also define backward dynamics:
    \begin{equation}
        G_l = 
        \frac{\partial \LL}{\partial H_l} =
        \phi'(H_l) \odot (W_l^T \frac{\partial \LL}{\partial H_{l+1}}) =
        \phi'(H_l) \odot (W_l^T G_{l+1}),
    \end{equation}
    \begin{equation}
        G_{L+1} =
        \frac{\partial \LL}{\partial H_{L+1}} =
        \frac{1}{2} \frac{\partial}{\partial H_{L+1}} \|Y-H_{L+1}\|_F^2 =
        H_{L+1} - Y,
    \end{equation}
    where $G_l \in \RR^{n_l \times m}$.
    Then we have:
    \begin{equation}
        \nabla_l =
        \frac{\partial \LL}{\partial W_l} =
        G_{l+1} X_l^T \in \RR^{n_{l+1} \times n_l}.
    \end{equation}

    Let $W_{0:L}^*$ be a local minimum and suppose $d_l \geq m$ for some $l \in [L]$.
    As previously, we divide our reasoning in two parts: in the first part we assume that $\rk X_l^* = m$, while in the second one we show that if $\rk X_l^* < m$ then $W_{0:L}^*$ cannot be a minimum.

    Assume $\rk X_l^* = m$.
    We have:
    \begin{equation}
        0 =
        \nabla_l^* =
        G_{l+1}^* X_l^{*,T},
    \end{equation}
    or,
    \begin{equation}
        X_l^* G_{l+1}^{*,T} = 0 \in \RR^{n_l \times n_{l+1}}.
    \end{equation}
    Each column of the right-hand side is a linear combination of columns of $X_l^*$.
    Since columns of $X_l^*$ are linearly independent, $G_{l+1}^* = 0$.
    By the recurrent relation,
    \begin{equation}
        0 = G_{l+1}^* = 
        \phi'(H_{l+1}^*) \odot (W_{l+1}^{*,T} G_{l+2}^*).
    \end{equation}
    Assume that $\phi'$ never gets zero.
    This gives:
    \begin{equation}
        W_{l+1}^{*,T} G_{l+2}^* = 0 \in \RR^{n_{l+1} \times m}.
    \end{equation}
    If we assume that columns of $W_{l+1}^{*,T}$ (or, equivalently, rows of $W_{l+1}^*$) are linearly independent), we shall get $G_{l+2}^* = 0$.
    Linearly independent rows of $W_{l+1}^*$ is equivalent to $\rk W_{l+1}^* = n_{l+2}$ which implies $n_{l+1} \geq n_{l+2}$.

    Suppose this assumption holds.
    If we moreover assume that $\rk W_{l'}^* = n_{l'+1}$ $\forall l' \in \{l+1, \ldots, L\}$, we get $G_{L+1}^* = 0$.
    This implies $\LL(W_{0:L}^*) = 0$.
    The assumption on ranks of $W_{l'}^*$ requires $n_{l'} \geq n_{l'+1}$ $\forall l' \in \{l+1, \ldots, L\}$: the network does not expand after the $l$-th layer.

    Now we assume $\rk X_l^* < m$ and $\LL(W_{0:L}^*) > 0$, while still $n_l \geq m$.
    In the shallow case we have shown that first, infinitesimal perturbation of $W_{0:l-1}^*$ results in $\rk X_l \geq m$, and second, starting from this perturbed point, the gradient descent dynamics leaves a sufficiently large vicinity of $W_{0:L}^*$.
    These two together imply that $W_{0:L}^*$ cannot be a minimum which is a contradiction.

    While both statements still hold if $l = L$, the second one does not hold for $0 < l < L$ since the problem of minimizing $\LL_{W_{0:l-1}^*}(W_{l:L})$ is still non-convex, hence we have no guarantees on gradient descent convergence.
    Hence for the case $0 < l < L$ we have to come up with another way of reasoning.

    Define:
    \begin{equation}
        u = \mathrm{vec}(W_{0:l-1}),
        \qquad
        v = \mathrm{vec}(W_{l:L}),
        \qquad
        \psi = \frac{\partial \LL(u,v)}{\partial v}.
    \end{equation}
    Since $(u^*, v^*)$ is a minimum, we have $\psi(u^*,v^*) = 0$.
    Assume that jacobian of $\psi$ with respect to $v$ is non-singular at $(u^*,v^*)$:
    \begin{equation}
        \det(J_v \psi(u^*,v^*)) \neq 0.
        \label{eq:non_singular_det}
    \end{equation}
    Note that in the case $l = L$ this property is equivalent to $\rk X_L^* = n_L$:
    \begin{equation}
        \psi(u,v)_{i n_L + j} =
        (W_{L,ik} X_L^{kl} - Y_i^l) X_{L,jl};
        \qquad
        (J_v \psi(u,v))_{i n_L + j, i' n_L + j'} =
        \delta_{i i'} X_{L,j'}^l X_{L,jl} =
        \delta_{i i'} (X_L X_L^T)_{j j'}.
    \end{equation}
    We see that $J_v \psi(u,v)$ is a block-diagonal matrix constructed with $n_{L+1}$ identical blocks $X_L X_L^T$.
    Its determinant at $W_{0:L}^*$ is therefore $(\det(X_L^* X_L^{*,T}))^{n_{L+1}}$, which is positive as long as $\rk X_L^* = n_L$.
    Note that $\rk X_L^* \leq m$, hence we need $n_L \leq m$.
    Note that we need $n_L \geq m$ in order to apply Lemma~\ref{lemma:full_rank_almost_surely}, hence Сondition~\ref{eq:non_singular_det} actually requires a stronger property $n_L = m$ instead of $n_L \geq m$ used before.

    Condition~\ref{eq:non_singular_det} allows us to apply the implicit function theorem:
    \begin{equation}
        \exists \delta_1 > 0: \;
        \exists \tilde v \in C^1(B_{\delta_1}(u^*)): \;
        \tilde v(u^*) = v^* \; \text{and} \;
        \forall u \in B_{\delta_1}(u^*) \;
        \psi(u, \tilde v(u)) = 0.
    \end{equation}
    Since all matrices $W_{l+1:L}^*$ are full rank, and the set of non-full rank matrices has measure zero,
    \begin{equation}
        \exists \tilde \epsilon > 0: \;
        \forall v \in B_{\tilde \epsilon}(v^*) \;
        \forall l' \in \{l+1,\ldots,L\} \;
        \rk W_{l'} = n_{l'+1}.
    \end{equation}
    Since $\tilde v \in C^1(B_{\delta_1}(u^*))$,
    \begin{equation}
        \exists \delta_2 \in (0, \delta_1): \;
        \forall u \in B_{\delta_2}(u^*) \;
        \tilde v(u) \in B_{\tilde \epsilon}(v^*).
    \end{equation}
    Consequently, 
    \begin{equation}
        \forall u \in B_{\delta_2}(u^*) \;
        \forall l' \in \{l+1,\ldots,L\} \;
        \rk \tilde W_{l'} = n_{l'+1}.
    \end{equation}
    Due to Lemma~\ref{lemma:full_rank_almost_surely_deep},
    \begin{equation}
        \forall \epsilon > 0 \;
        \exists \tilde u \in B_\epsilon(u^*): \;
        \rk \tilde X_l = m.
    \end{equation}
    Hence
    \begin{equation}
        \forall \epsilon \in (0, \delta_2) \;
        \exists \tilde u \in B_\epsilon(u^*): \;
        \rk \tilde X_l = m \; \text{and} \;
        \forall l' \in \{l+1,\ldots,L\} \;
        \rk \tilde W_{l'} = n_{l'+1} \; \text{and} \;
        \psi(\tilde u, \tilde v(\tilde u)) = 0.
    \end{equation}
    Note that in the first part of the proof we have only used that $\rk X_l^* = m$, $\rk W_{l'}^* = n_{l'+1}$ $\forall l' \in \{l+1,\ldots,L\}$, and $\nabla_l^* = 0$.
    Hence we can conclude that $\LL(\tilde u, \tilde v(\tilde u)) = 0$.
    Since this is true for all $\epsilon \in (0, \delta_2)$ and the loss is continuous with respect to weights, this is also true for $\epsilon = 0$: $\LL(u^*, v^*) = \LL(W_{0:L}^*) = 0$.

    \subsubsection{Relaxing analyticity and other conditions}

    Overall, we have proven the following result first:
    \begin{prop}
        \label{prop:global_minima_full_rank}
        Consider a point in the weight space $W_{0:L}^*$.
        Suppose the following hold:
        \begin{enumerate}
            \item $\phi'$ is not zero anywhere;
            \item $G_{H+1} = 0$ implies $\LL \to \min$;
            \item $\rk X_l^* = m$;
            \item $\rk W_{l'}^* = n_{l'+1}$ $\forall l' \in \{l+1,\ldots,L\}$;
            \item $\nabla_l^* = 0$.
        \end{enumerate}
        Then $\LL(W_{0:L}^*) = \min \LL$.
    \end{prop}
    After that, we have relaxed the 3rd condition in the expense of few others:
    \begin{prop}
        Consider a point in the weight space $W_{0:L}^*$.
        Suppose the following hold:
        \begin{enumerate}
            \item $\phi'$ is not zero anywhere;
            \item $G_{H+1} = 0$ implies $\LL \to \min$;
            \item $\phi$ is non-zero real analytic;
            \item $\rk W_{l'}^* = n_{l'+1}$ $\forall l' \in \{l+1,\ldots,L\}$;
            \item $\det(\nabla^2_{W_{l+1:L}} \LL(W_{0:L}^*)) \neq 0$;
            \item $\nabla_{l'}^* = 0$ $\forall l' \in \{l,\ldots,L\}$.
        \end{enumerate}
        Then $\LL(W_{0:L}^*) = \min \LL$.
    \end{prop}

    However, besides of the 3rd condition, Proposition~\ref{prop:global_minima_full_rank} requires 4th condition that is hard to ensure.
    We can prove the following lemma which is due to~\cite{nguyen2019connected}:
    \begin{lemma}
        \label{lemma:path_to_full_rank}
        Let $\theta = W_{l+1:L}$.
        Suppose the following hold:
        \begin{enumerate}
            \item $\rk X_l = m$;
            \item $n_{l'} > n_{l'+1}$ $\forall l' \in \{l+1,\ldots,L\}$;
            \item $\phi(\RR) = \RR$ and $\phi$ is strictly monotonic.
        \end{enumerate}
        Then 
        \begin{enumerate}
            \item $\exists \theta':$ $\forall l' \in \{l+1,\ldots,L\}$ $\rk W_{l'} = n_{l'+1}$ and $\LL(\theta') = \LL(\theta)$;
            \item $\exists$ a continuous curve connecting $\theta$ and $\theta'$, and loss is constant on the curve.
        \end{enumerate}
    \end{lemma}
    Applying this lemma, we can drive $W^*_{l+1:L}$ to full-rank $W^{*,\prime}_{l+1:L}$ without altering the loss, however Lemma~\ref{lemma:path_to_full_rank} does not guarantee that $\nabla^{*,\prime}_l = 0$.
    Hence by applying Lemma~\ref{lemma:path_to_full_rank} we potentially violate the 5th condition of Proposition~\ref{prop:global_minima_full_rank}.
    Moreover, as we have discussed before, loss convexity is not enough to ensure that minima exist.
    For example, for cross-entropy loss there could be no critical points of $\LL$, hence we cannot statisfy the 5th condition at all.
    Hence we have to formulate a different variant of Proposition~\ref{prop:global_minima_full_rank}.

    Following~\cite{nguyen2019connected}, we define an $\alpha$-level set as $\LL^{-1}(\alpha)$ and $\alpha$-sublevel set as $\LL^{-1}((-\infty,\alpha))$.
    We also refer a connected connected of a sublevel set a "local valley", and we call a local valley global if its infium coincide with $\inf \LL$.
    There is a theorem which is due to~\cite{nguyen2019connected}:
    \begin{theorem}
        \label{thm:global_valley_full_rank}
        Suppose the following hold:
        \begin{enumerate}
            \item $\phi(\RR) = \RR$ and $\phi$ is strictly monotonic;
            \item $\ell(y,z)$ is convex wrt $z$ with $\inf_z \ell(y,z) = 0$ $\forall y$;
            \item $\rk X_l = m$;
            \item $n_{l'} > n_{l'+1}$ $\forall l' \in \{l+1,\ldots,L\}$.
        \end{enumerate}
        Then
        \begin{enumerate}
            \item Every sublevel set is connected;
            \item $\forall \epsilon > 0$ $\LL$ can attain a value $< \epsilon$.
        \end{enumerate}
    \end{theorem}
    Theorem~\ref{thm:global_valley_full_rank} not only formulates a global minimality condition in a way suitable for cross-entropy (i.e. that all local valleys are global), but also implies that all local valleys are connected.
    In a case when local minima exist, the latter implies that all of them are connected: a phenomena empirically observed in~\cite{garipov2018loss,draxler2018essentially}.

    Notice that it is enough to prove Theorem~\ref{thm:global_valley_full_rank} for $l=0$: otherwise we can just apply this result to a subnetwork starting from the $l$-th layer.
    Let $\Omega_l = \RR^{n_{l+1} \times n_l}$ be a set of all $n_{l+1} \times n_l$ matrices, while $\Omega_l^* \subset \Omega_l$ be a subset of full-rank matrices.
    We shall state the following result first:
    \begin{lemma}
        \label{lemma:recovering_W_0}
        Suppose the following hold:
        \begin{enumerate}
            \item $\phi(\RR) = \RR$ and $\phi$ is strictly monotonic;
            \item $\rk X = m$;
            \item $n_l > n_{l+1}$ $\forall l \in [L]$.
        \end{enumerate}
        Then there exists a map $h: \Omega_1^* \times \ldots \times \Omega_L^* \times \RR^{n_{L+1} \times m} \to \Omega_0$:
        \begin{enumerate}
            \item $\forall \tilde H_{L+1} \in \RR^{n_{L+1} \times m}$ for full-rank $W_{1:L}$ $H_{L+1}(h(W_{1:L}, \tilde H_{L+1}), W_{1:L}) = \tilde H_{L+1}$;
            \item $\forall W_{0:L}$ where all $W_{1:L}$ are full-rank, there is a continuous curve between $W_{0:L}$ and $(h(W_{1:L}, H_{L+1}(W_{0:L})), W_{1:L})$ such that the loss is constant on the curve.
        \end{enumerate}
    \end{lemma}

    The first statement can be proven easily.
    Indeed, let $X^\dagger$ be the left inverse of $X$, while $W_l^\dagger$ be the right inverse of $W_l$; this means that $X^\dagger X = I_m$, while $W_l W_l^\dagger = I_{n_{l+1}}$ $\forall l \in [L]$.
    These pseudo-inverses exist, because $X$ has full column rank, while all $W_l$ have full row rank (since $n_l > n_{l+1}$).
    Define the following recursively:
    \begin{equation}
        \tilde W_0 = 
        \tilde H_1 X^\dagger,
        \qquad
        \tilde H_l = 
        \phi^{-1}(\tilde X_l),
        \qquad
        \tilde X_l =
        W_l^\dagger \tilde H_{l+1}
        \quad
        \forall l \in [L].
    \end{equation}
    This gives the following:
    \begin{equation}
        \tilde W_0 X =
        \tilde H_1 X^\dagger X =
        \tilde H_1,
        \qquad
        W_l \phi(\tilde H_l) =
        W_l W_l^\dagger \tilde H_{l+1} =
        \tilde H_{l+1}
        \quad
        \forall l \in [L].
    \end{equation}
    This simply means that $H_{L+1}(\tilde W_0, W_{1:L}) = \tilde H_{L+1}$.
    Hence defining $h(W_{1:L}, \tilde H_{l+1}) := \tilde W_0$ gives the result.

    We shall omit the proof of the second statement.
    Then the proof of Theorem~\ref{thm:global_valley_full_rank} proceeds by constructing paths from two points $\theta = W_{0:L}$ and $\theta'$ to a common point such that the loss does not increase along both of these paths.
    We then show that the common meeting point can attain loss $< \epsilon$ $\forall \epsilon > 0$.

    Denote losses at points $\theta$ and $\theta'$ as $\LL$ and $\LL'$ respectively.
    Let us start from the point $\theta$.
    By the virtue of Lemma~\ref{lemma:path_to_full_rank} we can travel from $\theta$ to another point for which all matrices are full rank without altering the loss.
    Hence without loss of generality assume that all matrices of $\theta$ are full rank, and for $\theta'$ we can assume the same.
    This allows us to use Lemma~\ref{lemma:recovering_W_0} and travel from $\theta$ and $\theta'$ to the following points by curves of constant loss:
    \begin{equation}
        \theta =
        (h(W_{1:L}, H_{L+1}(W_{0:L})), W_{1:L}),
        \qquad
        \theta' =
        (h(W'_{1:L}, H_{L+1}(W'_{0:L})), W'_{1:L}).
    \end{equation}
    Since the set of full-rank matrices is connected, $\forall l \in [L]$ there is a continuous curve $W_l(t)$ for which $W_l(0) = W_l$, $W_l(1) = W'_l$, and $W_l(t)$ is full-rank.
    Hence we can travel from $\theta$ to the following point in the weight space:
    \begin{equation}
        \theta =
        (h(W_{1:L}(1), H_{L+1}(W_{0:L})), W_{1:L}(1)) =        
        (h(W'_{1:L}, H_{L+1}(W_{0:L})), W'_{1:L}).
    \end{equation}
    Since we do not alter the model output while traveling throughout the curve, we do not alter the loss as well.

    Consider some $\tilde H_{L+1} \in \RR^{n_{L+1} \times m}$ such that corresponding loss is less than $\min(\epsilon, \LL, \LL')$.
    Consider a curve $H_{L+1}(t) = (1-t) H_{L+1}(W_{0:L}) + t \tilde H_{L+1}$, and a corresponding curve in the weight space:
    \begin{equation}
        \theta(t) =
        (h(W'_{1:L}, H_{L+1}(t)), W'_{1:L}).
    \end{equation}
    Note that
    \begin{equation}
        \LL(\theta(t)) = 
        \LL(H_{L+1}(\theta(t))) = 
        \LL((1-t) H_{L+1}(W_{0:L}) + t \tilde H_{L+1}) \leq
        (1-t) \LL(H_{L+1}(W_{0:L})) + t \LL(\tilde H_{L+1}) \leq
        \LL.
    \end{equation}
    Hence the curve $\theta(t)$ is fully contained in any sublevel set conatining initial $\theta$.
    The same curve starting from $\theta'$ arrives at the same point.
    Recall that the endpoint has loss less than $\epsilon$.
    Hence all sublevel sets are connected and can attain loss less than any positive $\epsilon$.

    \section{Linear nets}

    The second case for which one can prove globality of local minima is the case of $\phi(z) = z$.
    Consider:
    \begin{equation}
        f(x; W_{0:L}) = 
        W_L \ldots W_0 x,
    \end{equation}
    where $W_l \in \RR^{n_{l+1} \times n_l}$.
    We are going to prove the following result which is due to~\cite{laurent2018deep}:
    \begin{theorem}
        \label{thm:all_local_minima_are_global_linear_net}
        Let $\ell$ be convex and differentiable, and there are no bottlenecks in the architecture: $\min_{l \in [L]_0} n_l = \min\{n_0,n_{L+1}\}$.
        Then all local minima of $\LL(W_{0:L}) = \EE_{x,y} \ell(y,f(x; W_{0:L}))$ are global.
    \end{theorem}

    This theorem follows from the result below:
    \begin{theorem}
        \label{thm:all_local_minima_for_deep_are_critical_for_shallow}
        Assume $\tilde \LL$ is a scalar differentiable function of $n_{L+1} \times n_0$ matrices.
        Let $\LL(W_{0:L}) = \tilde \LL(W_L \ldots W_0)$ and let $\min_{l \in [L]_0} n_l = \min\{n_0,n_{L+1}\}$.
        Then any local minimizer $\hat W_{0:L}$ of $\LL$ satisfies $\nabla \tilde\LL(\hat A) = 0$ for $\hat A = \hat W_L \ldots \hat W_0$.
    \end{theorem}
    Indeed, consider $\tilde \LL(A) = \EE_{x,y} \ell(y, Ax)$.
    The corresponding $\LL$ writes as follows: $\LL(W_{0:L}) = \EE_{x,y} \ell(y,f(x; W_{0:L}))$; hence we are in the scope of Theorem~\ref{thm:all_local_minima_are_global_linear_net}.
    Take a local minimizer $\hat W_{0:L}$ of $\LL$.
    From Theorem~\ref{thm:all_local_minima_for_deep_are_critical_for_shallow} $\hat W_L \ldots \hat W_0$ is a critical point of $\tilde \LL$.
    It follows from convexity of $\ell$ that $\hat W_L \ldots \hat W_0$ is a global minimum of $\tilde\LL$.
    Since $\LL(\hat W_{0:L}) = \tilde\LL(\hat A)$, $\hat W_{0:L}$ is a global minimum of $\LL$.

    Let us now prove Theorem~\ref{thm:all_local_minima_for_deep_are_critical_for_shallow}.
    Define $W_{l,+} = W_L \ldots W_l$, $W_{l,-} = W_l \ldots W_0$, and $A = W_L \ldots W_0$.
    Note that
    \begin{equation}
        \nabla_l \LL(W_{0:L}) =
        W_{l+1,+}^T \nabla \tilde\LL(A) W_{l-1,-}^T
        \quad
        \forall l \in [L]_0.
    \end{equation}

    Since $\hat W_{0:L}$ is a local minimum of $\LL$, we have:
    \begin{equation}
        0 =
        \nabla_L \LL(\hat W_{0:L}) =
        \nabla \tilde\LL(\hat A) \hat W_{L-1,-}^T.
    \end{equation}
    If $\ker W_{L-1,-} = \{0\}$ then $\tilde\LL(\hat A) = 0$ as required.
    Consider the case when the kernel is non-trivial.
    We shall prove that there exist perturbed matrices $\tilde W_{0:L}$ such that $\tilde A = \hat A$ and $\tilde W_{0:L}$ is a local minimizer of $\LL$, and for some $l \in [L-1]_0$ kernels of both $\tilde W_{l-1,-}$ and $\tilde W_{l+1,+}^T$ are trivial.
    This gives $\tilde\LL(\tilde A) = 0$ which is equivalent to $\tilde\LL(\hat A) = 0$.

    By definition of a local minimizer, $\exists \epsilon > 0:$ $\|W_l - \hat W_l\|_F \leq \epsilon$ $\forall l \in [L]_0$ implies $\LL(W_{0:L}) \geq \LL(\hat W_{0:L})$.
    \begin{prop}
        \label{prop:perturbed_local_minima}
        Let $\tilde W_{0:L}$ satisfy the following:
        \begin{enumerate}
            \item $\|\tilde W_l - \hat W_l\|_F \leq \epsilon/2$ $\forall l \in [L]_0$;
            \item $\tilde A = \hat A$.
        \end{enumerate}
        Then $\tilde W_{0:L}$ is a local minimizer of $\LL$.
    \end{prop}
    \begin{proof}
        Let $\|W_l - \tilde W_l\|_F \leq \epsilon/2$ $\forall l \in [L]_0$.
        Then $\|W_l - \hat W_l\|_F \leq \|W_l - \tilde W_l\|_F + \|\tilde W_l - \hat W_l\|_F \leq \epsilon$ $\forall l \in [L]_0$.
        Hence $\LL(W_{0:L}) \geq \LL(\hat W_{0:L}) = \LL(\tilde W_{0:L})$.
    \end{proof}

    Since $W_{l+1,-} = W_{l+1} W_{l,-}$, we have $\ker(W_{l+1,-}) \supseteq \ker(W_{l,-})$.
    Hence there is a chain of inclusions:
    \begin{equation}
        \ker(\hat W_{0,-}) \subseteq \ldots \subseteq \ker(\hat W_{L-1,-}).
    \end{equation}
    Since the $(L-1)$-th kernel is non-trivial, there exists $l_* \in [L-1]_0$ such that $\ker(\hat W_{l,-})$ is non-trivial for any $l \geq l_*$, while for $l < l_*$ the $l$-th kernel is trivial.
    This gives the following:
    \begin{equation}
        0 =
        \nabla_{l_*} \LL(\hat W_{0:L}) =
        \hat W_{l_*+1,+}^T \nabla \tilde\LL(\hat A) \hat W_{l_*-1,-}^T
        \quad
        \text{implies}
        \quad
        0 =
        \hat W_{l_*+1,+}^T \nabla \tilde\LL(\hat A).
        \label{eq:local_minimum_of_L}
    \end{equation}
    We cannot guarantee that $\ker(\hat W_{l_*+1,+}^T)$ is trivial.
    However, we can try to construct a perturbation that does not alter the loss and such that the corresponding kernel is trivial.

    First, without loss of generality assume that $n_{L+1} \geq n_0$.
    Indeed, if Theorem~\ref{thm:all_local_minima_for_deep_are_critical_for_shallow} is already proven for $n_{L+1} \geq n_0$, we can get the same result for $n_{L+1} < n_0$ by applying this theorem to $\tilde \LL'(A) = \tilde \LL(A^T)$.
    This gives that all local minima of $\LL'(W_{L:0}^T) = \tilde\LL'(W_0^T \ldots W_L^T) = \tilde\LL(W_L \ldots W_0)$ correspond to critical points of $\tilde\LL'(W_0^T \ldots W_L^T)$.
    This is equivalent to saying that all local minima of $\LL(W_{0:L}) = \LL'(W_{L:0}^T)$ correspond to critical points of $\tilde\LL(W_L \ldots W_0) = \tilde\LL'(W_0^T \ldots W_L^T)$.
    Combination of assumptions $n_{L+1} \geq n_0$ and $\min_l n_l = \min\{n_0,n_{L+1}\}$ gives $n_l \geq n_0$ $\forall l \in [L+1]$.

    Note that $W_{l,-} \in \RR^{n_{l+1} \times n_0}$ $\forall l \in [L]_0$.
    Since $n_{l+1} \geq n_0$, it is a "column" matrix.
    Consider an SVD decomposition of $\hat W_{l,-}$:
    \begin{equation}
        \hat W_{l,-} =
        \hat U_l \hat \Sigma_l \hat V_l^T.
    \end{equation}
    Here $\hat U_l$ is an orthogonal $n_{l+1} \times n_{l+1}$ matrix, $\hat V_l$ is an orthogonal $n_0 \times n_0$ matrix, and $\hat\Sigma_l$ is a diagonal $n_{l+1} \times n_0$ matrix.
    Since for $l \geq l_*$ $\hat W_{l,-}$ has a non-trivial kernel, its least singular value is zero: $\hat\sigma_{l,n_0} = 0$.
    Let $\hat u_l$ be the $n_0$-th column of $\hat U_l$, which exists since $n_0 \leq n_{l+1}$.
    Let us now define a family of perturbations satisfying the conditions of Proposition~\ref{prop:perturbed_local_minima}:
    \begin{prop}
        \label{prop:perturbation_family}
        Let $w_{l_*+1},\ldots,w_L$ be any collections of vectors and $\delta_{l_*+1},\ldots,\delta_L$ be any collection of scalars satisfying:
        \begin{enumerate}
            \item $w_l \in \RR^{n_{l+1}}$, $\|w_l\|_2 = 1$;
            \item $\delta_l \in [0,\epsilon/2]$.
        \end{enumerate}
        Then the tuples $\tilde W_{0:L}$ defined by
        \begin{equation}
            \tilde W_l =
            \hat W_l + \delta_l w_l \hat u_{l-1}^T
            \quad \text{for $l > l_*$,} \quad \text{and} \quad
            \tilde W_l =
            \hat W_l
            \quad \text{otherwise,}
        \end{equation}
        satisfy the conditions of Proposition~\ref{prop:perturbed_local_minima}.
    \end{prop}
    \begin{proof}
        For $l \leq l_*$ the first condition is trivial.
        In the opposite case we have:
        \begin{equation}
            \|\tilde W_l - \hat W_l\|_F^2 =
            \|\delta_l w_l \hat u_{l-1}^T\|_F^2 =
            \delta_l^2 \|w_l\|_2^2 \|\hat u_{l-1}\|_2^2 \leq
            \epsilon^2 / 4,
        \end{equation}
        which gives the first condition of Proposition~\ref{prop:perturbed_local_minima}.

        Let us now prove that $\tilde W_{l,-} = \hat W_{l,-}$ $\forall l \geq l_*$ (for $l < l_*$ the statement is trivial).
        For $l = l_*$ the statement goes from the definition; this gives the induction base.
        The induction step is given as follows:
        \begin{equation}
            \tilde W_{l+1,-} =
            \tilde W_{l+1} \tilde W_{l,-} =
            \tilde W_{l+1} \hat W_{l,-} =
            (\hat W_{l+1} + \delta_{l+1} w_{l+1} \hat u_{l}^T) \hat W_{l,-} =
            \hat W_{l+1} \hat W_{l,-} =
            \hat W_{l+1,-}.
        \end{equation}
    \end{proof}

    Hence by Proposition~\ref{prop:perturbed_local_minima} for any $\delta_l$ and $w_l$ satisfying the conditions of Proposition~\ref{prop:perturbation_family}, $\tilde W_{0:L}$ is a local minimum of $\LL$.
    Then we have an equation similar to~(\ref{eq:local_minimum_of_L}):
    \begin{equation}
        0 =
        \nabla_{l_*} \LL(\tilde W_{0:L}) =
        \tilde W_{l_*+1,+}^T \nabla \tilde\LL(\tilde A) \hat W_{l_*-1,-}^T.
    \end{equation}
    As before, this implies:
    \begin{equation}
        0 =
        \nabla \tilde\LL^T(\tilde A) \tilde W_{l_*+1,+}.
    \end{equation}
    For $\delta_{l_*+1} = 0$ we have:
    \begin{equation}
        0 =
        \nabla \tilde\LL^T(\tilde A) \tilde W_L \ldots \tilde W_{l_*+2} \hat W_{l_*+1}.
    \end{equation}
    Substracting the latter equation to the pre-latter one gives:
    \begin{equation}
        0 =
        \nabla \tilde\LL^T(\tilde A) \tilde W_L \ldots \tilde W_{l_*+2} (\tilde W_{l_*+1} - \hat W_{l_*+1}) =
        \nabla \tilde\LL^T(\tilde A) \tilde W_L \ldots \tilde W_{l_*+2} (\delta_{l_*+1} w_{l_*+1} \hat u_{l_*}^T).
    \end{equation}
    Right-multiplying this equation by $\hat u_{l_*}$ gives:
    \begin{equation}
        0 =
        \delta_{l_*+1} \nabla \tilde\LL^T(\tilde A) \tilde W_L \ldots \tilde W_{l_*+2} w_{l_*+1},
    \end{equation}
    which holds for any sufficiently small non-zero $\delta_{l_*+1}$ and any unit $w_{l_*+1}$.
    Hence
    \begin{equation}
        0 =
        \nabla \tilde\LL^T(\tilde A) \tilde W_L \ldots \tilde W_{l_*+2}.
    \end{equation}
    Proceeding in the same manner gives finally $\nabla \tilde\LL(\tilde A) = 0$.
    The proof concludes with noting that $\nabla \tilde\LL(\hat A) = \nabla \tilde\LL(\tilde A)$ by construction of $\tilde W_{0:L}$.

    \section{Local convergence guarantees}

    Let $\LL \in C^2(\RR^{\dim \theta})$ and $\LL$ has $L$-Lipschitz gradient:
    \begin{equation}
        \| \nabla\LL(\theta_1) - \nabla\LL(\theta_2) \|_2 \leq
        L \| \theta_1 - \theta_2 \|_2
        \quad
        \forall \theta_{1,2}.
    \end{equation}
    Consider a GD update rule:
    \begin{equation}
        \theta_{k+1} =
        \theta_k - \eta \nabla\LL(\theta_k) =
        g(\theta_k).
    \end{equation}
    Let $\theta^*$ be a strict saddle:
    \begin{equation}
        \nabla\LL(\theta^*) = 0,
        \quad
        \lambda_{min} (\nabla^2 \LL(\theta^*)) < 0.
    \end{equation}
    Let $\theta_0 \sim P_{init}$.
    We shall prove the following result which is due to~\cite{lee2016gradient}:
    \begin{theorem}
        \label{thm:non_convergence_to_strict_saddles}
        Suppose $P_{init}$ is absolutely continuous with respect to the Lebesgue measure $\mu$ on $\RR^{\dim\theta}$.
        Then for $\eta \in (0,L^{-1})$, $\PP(\{\lim_{k \to \infty} \theta_k = \theta^*\}) = 0$.
    \end{theorem}
    \begin{proof}
        The proof starts with the definition of global stable sets.
        Define a global stable set of a critical point as a set of initial conditions that lead to convergence to this critical point:
        \begin{equation}
            \Theta^s(\theta^*) =
            \{\theta_0: \; \lim_{k\to\infty} \theta_k = \theta^*\}.
        \end{equation}
        In order to prove the theorem it suffices to show that $\mu(\Theta^s(\theta^*)) = 0$.

        The proof relies on the following result of the theory of dynamical systems:
        \begin{theorem}
            \label{thm:local_manifolds}
            Let $0$ be a stable point of a local diffeomorphism $\phi: \; U \to E$, where $U$ is a vicinity of zero in a Banach space $E$.
            Suppose that $E = E_s \oplus E_u$, where $E_s$ is a span of eigenvectors that correspond to eigenvalues of $D\phi(0)$ less or equal to one, while $E_u$ is a span of eigenvectors that correspond to the eigenvalues greater than one.
            Then there exists a disk $\Theta^{sc}_{loc}$ tangent to $E_s$ at $0$ called the \emph{local stable center manifold.}
            Moreover, there exists a neighborhood $B$ of $0$, such that $\phi(\Theta^{sc}_{loc}) \cap B \subset \Theta^{sc}_{loc}$ and $\bigcap_{k=0}^\infty \phi^{-k}(B) \subset \Theta^{sc}_{loc}$.
        \end{theorem}
        In order to apply this theorem, we have to prove that $g$ is a diffeomorphism:
        \begin{prop}
            \label{prop:g_is_diffeomorphism}
            For $\eta \in (0,L^{-1})$ $g$ is a diffeomorphism.
        \end{prop}
        Given this, we apply the theorem above to $\phi(\theta) = g(\theta+\theta^*)$: its differential at zero is $D\phi(0) = I - \eta \nabla^2\LL(\theta^*)$.
        Since $\lambda_{min}(\nabla^2\LL(\theta^*)) < 0$, $\dim E_u > 0$, hence $\dim E_s < \dim \theta$.
        This means that $\mu(\Theta^{sc}_{loc}) = 0$.

        Let $B$ be a vicinity of zero promised by Theorem~\ref{thm:local_manifolds}.
        Given $\theta_0 \in \Theta^s(\theta^*)$, $\exists K \geq 0:$ $\forall k \geq K$ $\theta_k \in B$.
        Equivalently, $\forall l \geq 0$ $g^l(\theta_K) \in B$.
        Hence $\theta_K \in \bigcap_{l=0}^\infty g^{-l}(B) \subset \Theta^{sc}_{loc}$.
        This gives the following:
        \begin{equation}
            \Theta^s(\theta^*) \subseteq
            \bigcup_{K=0}^\infty g^{-K} (\Theta^{sc}_{loc}).
        \end{equation}
        The proof concludes by noting that $\mu(\Theta^s(\theta^*)) \leq \sum_{K=0}^\infty \mu(g^{-K} (\Theta^{sc}_{loc})) = \sum_{K=0}^\infty \mu(\Theta^{sc}_{loc}) = 0$ since $g$ is a diffeomorphism.
    \end{proof}

    \begin{proof}[Proof of Proposition~\ref{prop:g_is_diffeomorphism}]
        Being a diffeomorphism is equivalent to be injective, surjective, continuously differentiable, and having a continuously differentiable inverse.

        Suppose $g(\theta) = g(\theta')$.
        Then $\theta - \theta' = \eta(\nabla\LL(\theta') - \nabla\LL(\theta))$.
        Hence:
        \begin{equation}
            \|\theta - \theta'\|_2 =
            \eta \|\nabla\LL(\theta') - \nabla\LL(\theta)\|_2 \leq
            \eta L \|\theta - \theta'\|_2.
        \end{equation}
        Since $\eta < 1/L$, this implies $\theta = \theta'$.
        Hence $g$ is injective.

        Given some point $\theta_2$, we shall construct $\theta_1$ such that $\theta_2 = g(\theta_1)$.
        Consider:
        \begin{equation}
            h(\theta_{1,2}) =
            \frac{1}{2} \|\theta_1 - \theta_2\|_2^2 - \eta \LL(\theta_1).
        \end{equation}
        Note that $h(\theta_{1,2})$ is strongly convex with respect to $\theta_1$:
        \begin{equation}
            \lambda_{min}(\nabla_{\theta_1}^2 h(\theta_{1,2})) \geq
            1 - \eta L > 0.
        \end{equation}
        Hence it has a unique global minimizer which is a critical point:
        \begin{equation}
            0 =
            \nabla_{\theta_1} h(\theta_{1,2}) =
            \theta_1 - \theta_2 - \eta \nabla\LL(\theta_1).
        \end{equation}
        Hence a unique element of $\argmin_{\theta_1} h(\theta_{1,2})$ satisfies $\theta_2 = g(\theta_1)$.
        Hence $g$ is surjective.

        The fact that $g \in C^1(\RR^{\dim\theta})$ follows from the fact that $g(\theta) = \theta - \eta \nabla\LL(\theta)$ and $\LL \in C^2(\RR^{\dim\theta})$.
        By the virtue of the inverse function theorem, in order to prove that $g$ has $C^1$ inverse, it suffices to show that $g$ is itself $C^1$ and its jacobian is non-singular everywhere.
        The jacobian is given as $J g(\theta) = I - \eta \nabla^2\LL(\theta)$; hence its minimal eigenvalue $\geq 1 - \eta L > 0$ which means that the jacobian is non-singular.
        The latter statement concludes the proof that $g$ is a diffeomorphism.
    \end{proof}

    \subsection{Limitations of the result}

    Note that Theorem~\ref{thm:non_convergence_to_strict_saddles} applies in the following assumptions:
    \begin{enumerate}
        \item $\LL \in C^2$;
        \item $\nabla\LL$ is $L$-Lipschitz and $\eta \in (0,L^{-1})$;
        \item No gradient noise;
        \item The saddle point is strict;
        \item The saddle point is isolated.
    \end{enumerate}

    The first assumption is necessary to ensure that $g$ is a diffeomorphism.
    ReLU nets violate this assumption, and hence require a generalization of Theorem~\ref{thm:local_manifolds}.

    The second assumption is a standard assumption for the optimization literature.
    Note however that for, say, a quadratic loss, a network with at least one hidden layer results in a loss surface which is not globally Lipschitz.
    Fortunately, if we show that there exists a subset $S \in \RR^{\dim\theta}$ such that $g(S) \subseteq S$ and restrict initializations to this subset, one can substitite a global Lipschitzness requirement to local Lipschitzness in $S$; this is done in~\cite{panageas2017gradient}.
    Note that ReLU nets break gradient Lipschitzness anyway.

    A full-batch gradient descent is rare in practice; a typical procedure is a stochastic gradient descent which introduces a zero-centered noise to gradient updates.
    Existence of such noise pulls us away from the scope of the dynamical systems theory.
    Nevertheless, intuitively, this noise should help us to escape a stable manifold associated with a saddle point at hand.
    In turns out that the presence of noise allows one to have guarantees not only for convergence itself, but even for convergence rates: see e.g.~\cite{jin2017escape}.

    Strictness of saddle points is necessary to ensure that the second order information about the hessian of the loss is enough to identify $E_u$.
    We hypothesize that the generalization of Theorem~\ref{thm:local_manifolds} to high-order saddles is still possible (but out of the scope of the conventional dynamical systems theory).

    Note that Theorem~\ref{thm:non_convergence_to_strict_saddles} says essentially that we cannot converge to any \emph{a-priori given} saddle point.
    If the set of all saddle points is at most countable, this will imply that we cannot converge to any saddle points.
    However, if this set is uncountable, Theorem~\ref{thm:non_convergence_to_strict_saddles} does not guarantee that we do not converge to any of them.
    Moreover, e.g. for ReLU nets, there is a continuous family of weight-space symmetries that keep criticality (and hence keep negativity of the least eigenvalue of the hessian).
    Indeed, substituting $(W_{l+1},W_l)$ with $(\alpha^{-1} W_{l+1}, \alpha W_l)$ for any positive $\alpha$ keeps $f(x)$ unchanged.
    Moreover, if $\nabla_{l'} = 0$ $\forall l'$ then
    \begin{equation}
        \nabla_{l+1}^{(\alpha)} =
        g_{l+2} x_{l+1}^{(\alpha),T} =
        \alpha g_{l+2} x_{l+1}^T =
        \alpha \nabla_{l+1} =
        0,
    \end{equation}
    and all other $\nabla_{l'}^{(\alpha)} = 0$ by a similar reasoning.

    A generalization of Theorem~\ref{thm:non_convergence_to_strict_saddles} to non-isolated critical points is given in~\cite{panageas2017gradient}.
    Intuitively, if we have a manifold of strict saddle points, the global stable set associated with this manifold is still of measure zero due to the existence of the unstable manifold.
    Nevertheless, one have to again generalize Theorem~\ref{thm:local_manifolds}.

    \chapter{Generalization}

    The goal of learning is to minimize a population risk $R$ over some class of predictors $\FF$:
    \begin{equation}
        f^* \in 
        \Argmin_{f \in \FF} R(f),
    \end{equation}
    where $R(f) = \EE_{x,y\sim\DD} r(y, f(x))$; here $\DD$ is a data distribution and $r(y,z)$ is a risk; we shall assume that $r(y,z) \in [0,1]$.
    A typical notion of risk for binary classification problems is $0/1$-risk: $r_{0/1}(y,z) = [y z < 0]$, where the target $y \in \{-1,1\}$ and the logit $z \in \RR$; in this case $1 - R(f)$ is an accuracy of $f$.
    Since we do not have an access to the true data distribution $\DD$, we cannot minimize the true risk.
    Instead, we can hope to minimize an empirical risk $\hat R_m$ over a set of $m$ i.i.d. samples $S_m$ from distribution $\DD$:
    \begin{equation}
        \hat f_m \in 
        \Argmin_{f \in \FF} \hat R_m(f),
        \label{eq:emp_risk_minimization}
    \end{equation}
    where $\hat R_m(f) = \EE_{x,y\sim S_m} r(y, f(x))$.
    Since a risk function is typically non-convex and suffer from poor gradients, one cannot solve problem~(\ref{eq:emp_risk_minimization}) directly with gradient methods.
    A common solution is to consider a convex differentiable surrogate $\ell$ for a risk $r$, and substitute problem~(\ref{eq:emp_risk_minimization}) with a train loss minimization problem:
    \begin{equation}
        \hat f_m \in
        \Argmin_{f \in \FF} \hat \LL_m(f),
        \label{eq:train_loss_minimization}
    \end{equation}
    where $\hat\LL_m(f) = \EE_{x,y\sim S_m} \ell(y, f(x))$; this problem can be attacked directly with gradient methods.

    Unfortunately, it is hard to obtain any guarantees for finding solutions even for problem~(\ref{eq:train_loss_minimization}).
    Nevertheless, suppose we have a learning algorithm $\AA$ that takes a dataset $S_m$ and outputs a model $\hat f_m$.
    This algorithm may aim to solve problem~(\ref{eq:train_loss_minimization}) or to tackle problem~(\ref{eq:emp_risk_minimization}) directly, but its purpose does not matter; what matters is the fact that it conditions a model $\hat f_m$ on a dataset $S_m$.
    Our goal is to upper-bound some divergence of $\hat R_m(\hat f_m)$ with respect to $R(\hat f_m)$.
    Since the dataset $S_m$ is random, $\hat f_m$ is also random, and the bound should have some failure probability $\delta$ with respect to $S_m$.

    \section{Uniform bounds}

    First of all, note that $R(f) = \EE_{S_m \sim \DD^m} \hat R_m(f)$.
    This fact suggests applying the Hoeffding's inequality for upper-bounding $\hat R_m(f) - R(f)$:
    \begin{theorem}[Hoeffding's inequality \cite{hoeffding1963probability}]
        Let $X_{1:m}$ be i.i.d. random variables supported on $[0,1]$.
        Then, given $\epsilon > 0$,
        \begin{equation}
            \PP\left(\sum_{i=1}^m X_i - \EE \sum_{i=1}^m X_i \geq \epsilon \right) \leq
            e^{-\frac{2\epsilon^2}{m}},
            \qquad
            \PP\left(\EE \sum_{i=1}^m X_i - \sum_{i=1}^m X_i \geq \epsilon \right) \leq
            e^{-\frac{2\epsilon^2}{m}}.
        \end{equation}
    \end{theorem}
    This gives us the following:
    \begin{equation}
        \PP(R(f) - \hat R_m(f) \geq \epsilon) \leq
        e^{-2 m \epsilon^2}
        \quad
        \forall \epsilon > 0
        \quad
        \forall f \in \FF.
    \end{equation}
    Hence for any $f \in \FF$,
    \begin{equation}
        R(f) - \hat R_m(f) \leq
        \sqrt{\frac{1}{2m} \log \frac{1}{\delta}}
        \quad
        \text{w.p. $\geq 1 - \delta$ over $S_m$.}
    \end{equation}
    However, the bound above does not make sense since $f$ there is given a-priori and does not depend on $S_m$.
    Our goal is to bound the same difference but with $f = \hat f_m$.
    The simplest way to do this is to upper-bound this difference uniformly over $\FF$:
    \begin{equation}
        R(\hat f_m) - \hat R_m(\hat f_m) \leq
        \sup_{f\in\FF} (R(f) - \hat R_m(f)).        
    \end{equation}

    \paragraph{A note on the goodness of uniform bounds.}

    One may worry about how large the supremum over $\FF$ can be.
    If the model class $\FF$ includes a "bad model" which has low train error for a given $S_m$ but large true error, the bound becomes too pessimistic.
    Unfortunately, in the case of realistic neural nets, one can explicitly construct such a bad model.
    For instance, for a given $S_m$ consider $\hat f_{m,m'} = \AA(S_m \cup \bar S_{m'})$ with $\bar S_{m'}$ being a dataset with random labels --- it is independent on $S_m$ and taken in advance.
    For $m' \gg m$, $\hat f_{m,m'} \approx \AA(\bar S_{m'})$ --- a model learned on random labels; see~\cite{zhang2016understanding}.
    Hence for binary classification with balanced classes $R(\hat f_{m,m'}) \approx 0.5$, while $\hat R_m(\hat f_{m,m'}) \approx 0$ whenever the algorithm is able to learn the data perfectly, which is empirically the case for gradient descent applied to realistic neural nets.

    Nevertheless, taking $\FF$ to be a set of all models realizable with a given architecture is not necessary.
    Indeed, assume the data lies on a certain manifold: $\supp \DD \subseteq \MM$.
    Then for sure, $\hat f_m \in \AA(\MM^m)$.
    Taking $\FF = \AA(\MM^m)$ ensures that $\FF$ contains only those models that are realizable by our algorithm on realistic data --- this excludes the situation discussed above.
    One can hope then that if our learning algorithm is good for any realistic data, the bound will be also good.
    The problem then boils to upper-bounding the supremum as well as possible.

    Unfortunately, bounding the supremum for $\FF = \AA(\MM^m)$ is problematic since it requires taking the algorithm dynamics into account, which is complicated for gradient descent applied to neural nets.
    As a trade-off, one can consider some larger $\FF \supseteq \AA(\MM^m)$, for which the supremum can be upper-bounded analytically.

    \subsection{Upper-bounding the supremum}

    When $\FF$ is finite, we can still apply our previous bound:
    \begin{multline}
        \PP\left(\sup_{f\in\FF} (R(f) - \hat R_m(f)) \geq \epsilon\right) =
        \PP(\exists f\in\FF: \; (R(f) - \hat R_m(f)) \geq \epsilon) \leq\\\leq
        \sum_{f\in\FF} \PP(R(f) - \hat R_m(f) \geq \epsilon) \leq
        |\FF| e^{-2 m \epsilon^2}
        \quad
        \forall \epsilon > 0.
    \end{multline}
    Hence
    \begin{equation}
        \sup_{f\in\FF} (R(f) - \hat R_m(f)) \leq
        \sqrt{\frac{1}{2m} \left(\log \frac{1}{\delta} + \log |\FF| \right)}
        \quad
        \text{w.p. $\geq 1 - \delta$ over $S_m$.}
    \end{equation}
    In the case when $\FF$ is infinite, we can rely on a certain generalization of Hoeffding's inequality:
    \begin{theorem}[McDiarmid's inequality \cite{mcdiarmid1989method}]
        Let $X_{1:m}$ be i.i.d. random variables and $g$ is a scalar function of $m$ arguments such that
        \begin{equation}
            \sup_{x_{1:m}, \hat x_i} |g(x_{1:m}) - g(x_{1:i-1}, \hat x_i, x_{i+1,m})| \leq c_i
            \quad
            \forall i \in [m].
        \end{equation}
        Then, given $\epsilon > 0$,
        \begin{equation}
            \PP\left( g(X_{1:m}) - \EE g(X_{1:m}) \geq \epsilon \right) \leq
            e^{-\frac{2\epsilon^2}{\sum_{i=1}^m c_i^2}}.
        \end{equation}
    \end{theorem}
    Applying this inequality to $g(\{(x_i,y_i)\}_{i=1}^m) = \sup_{f \in \FF}(R(f) - \hat R_m(f))$ gives:
    \begin{equation}
        \PP_{S_m} \left(\sup_{f \in \FF}(R(f) - \hat R_m(f)) - \EE_{S'_m} \sup_{f \in \FF}(R(f) - \hat R'_m(f)) \geq \epsilon\right) \leq
        e^{-2 m \epsilon^2},
    \end{equation}
    which is equivalent to:
    \begin{equation}
        \sup_{f \in \FF}(R(f) - \hat R_m(f)) \leq
        \EE_{S'_m} \sup_{f \in \FF}(R(f) - \hat R'_m(f)) + \sqrt{\frac{1}{2m} \log \frac{1}{\delta}}
        \quad
        \text{w.p. $\geq 1 - \delta$ over $S_m$.}
        \label{eq:mcdiarmid_bound_for_sup}
    \end{equation}

    Let us upper-bound the expectation:
    \begin{multline}
        \EE_{S'_m} \sup_{f \in \FF}(R(f) - \hat R'_m(f)) =
        \EE_{S'_m} \sup_{f \in \FF}(\EE_{S''_m} \hat R''_m(f) - \hat R'_m(f)) \leq
        \EE_{S'_m} \EE_{S''_m} \sup_{f \in \FF}(\hat R''_m(f) - \hat R'_m(f)) =
        \\=
        \EE_{S'_m} \EE_{S''_m} \sup_{f \in \FF} \left(\frac{1}{m} \sum_{i=1}^m (r(y''_i,f(x''_i)) - r(y'_i,f(x'_i)))\right) =
        \EE_{S'_m} \EE_{S''_m} \sup_{f \in \FF} \left(\frac{1}{m} \sum_{i=1}^m (r''_i(f) - r'_i(f))\right) =
        \\=
        \EE_{S'_m} \EE_{S''_m} \EE_{\sigma_m \sim \{-1,1\}^m} \sup_{f \in \FF} \left(\frac{1}{m} \sum_{i=1}^m \sigma_i (r''_i(f) - r'_i(f))\right) \leq
        \\\leq
        \EE_{S'_m} \EE_{S''_m} \EE_{\sigma_m \sim \{-1,1\}^m} \sup_{f \in \FF} \left|\frac{1}{m} \sum_{i=1}^m \sigma_i (r''_i(f) - r'_i(f))\right| =
        \\\leq
        \EE_{S'_m} \EE_{S''_m} \EE_{\sigma_m \sim \{-1,1\}^m} \sup_{f \in \FF} \left(\left|\frac{1}{m} \sum_{i=1}^m \sigma_i r''_i(f)\right| + \left|\frac{1}{m} \sum_{i=1}^m \sigma_i r'_i(f)\right|\right) =
        \\=
        2 \EE_{S'_m} \EE_{\sigma_m \sim \{-1,1\}^m} \sup_{f \in \FF} \left|\frac{1}{m} \sum_{i=1}^m \sigma_i r(y'_i,f(x'_i))\right| =
        2 \EE_{S'_m} \Rad{r \circ \FF}{S'_m},
    \end{multline}
    where we have defined a function class $r \circ \FF$ such that $\forall h \in r \circ \FF$ $h(x,y) = r(y,f(x))$, and the Rademacher complexity of a class $\HH$ of functions supported on $[0,1]$ conditioned on a dataset $z_{1:m}$:
    \begin{equation}
        \Rad{\HH}{z_{1:m}} =
        \EE_{\sigma_{1:m} \sim \{-1,1\}^m} \sup_{h \in \HH} \left|\frac{1}{m} \sum_{i=1}^m \sigma_i h(z_i)\right|.
    \end{equation}

    \subsection{Upper-bounding the Rademacher complexity}

    \subsubsection{The case of zero-one risk}

    Consider $r(y,z) = r_{0/1}(y,z) = [y z < 0]$.
    In this case we have:
    \begin{multline}
        \Rad{r \circ \FF}{S_m} =
        \EE_{\sigma_m} \sup_{f \in \FF} \left|\frac{1}{m} \sum_{i=1}^m \sigma_i r_i(f) \right| =
        \EE_{\sigma_m} \max_{f \in \FF_{S_m}} \left|\frac{1}{m} \sum_{i=1}^m \sigma_i r_i(f) \right| =
        \\=
        \frac{1}{m s} \log \exp \left(s \EE_{\sigma_m} \max_{f \in \FF_{S_m}} \left|\sum_{i=1}^m \sigma_i r_i(f) \right|\right) \leq
        \frac{1}{m s} \log \EE_{\sigma_m} \exp \left(s \max_{f \in \FF_{S_m}} \left|\sum_{i=1}^m \sigma_i r_i(f) \right|\right) =
        \\=
        \frac{1}{m s} \log \EE_{\sigma_m} \exp \left(s \max_{h \in r \circ \FF_{S_m} \cup (-r) \circ \FF_{S_m}} \sum_{i=1}^m \sigma_i h_i \right) \leq
        \frac{1}{m s} \log \sum_{h \in r \circ \FF_{S_m} \cup (-r) \circ \FF_{S_m}} \EE_{\sigma_m} \exp \left(s \sum_{i=1}^m \sigma_i h_i \right) \leq
        \\\leq
        \frac{1}{m s} \log \sum_{h \in r \circ \FF_{S_m} \cup (-r) \circ \FF_{S_m}} e^{\frac{m s^2}{2}} =
        \frac{1}{m s} \log \left(2 |\FF_{S_m}| e^{\frac{m s^2}{2}}\right) =
        \frac{1}{m s} \log (2 |\FF_{S_m}|) + \frac{s}{2},
        \label{eq:rad_bound_zero-one}
    \end{multline}
    where $s$ is any positive real number and $\FF_{S_m}$ is an equivalence class of functions from $\FF$, where two functions are equivalent iff their images on $S_m$ have identical signs; note that this class is finite: $|\FF_{S_m}| \leq 2^m$.
    We have also used the following lemma:
    \begin{lemma}[Hoeffding's lemma \cite{hoeffding1963probability}]
        Let $X$ be a random variable a.s.-supported on $[a,b]$ with zero mean.
        Then, for any positive real $s$,
        \begin{equation}
            \EE e^{sX} \leq
            e^{\frac{(b-a)^2 s^2}{8}}.
        \end{equation}
    \end{lemma}

    Since the upper bound~(\ref{eq:rad_bound_zero-one}) is valid for any $s > 0$, we can minimize it with respect to $s$.
    One can easily deduce that the optimal $s$ is $\sqrt{(2 / m) \log (2 |\FF_{S_m}|)}$; plugging it into eq.~(\ref{eq:rad_bound_zero-one}) gives:
    \begin{equation}
        \Rad{r \circ \FF}{S_m} \leq
        \sqrt{\frac{2}{m} \log (2 |\FF_{S_m}|)}.
        \label{eq:rad_bound_zero-one_final}
    \end{equation}
    Define $\Pi_\FF(m) = \max_{S_m} |\FF_{S_m}|$ --- a growth function of a function class $\FF$.
    The growth function shows how many distinct labelings a function class $\FF$ induces on datasets of varying sizes.
    Obviously, $\Pi_\FF(m) \leq 2^m$ and $\Pi_\FF(m)$ is monotonically non-increasing.
    We say "$\FF$ shatters $S_m$" whenever $|\FF_{S_m}| = 2^m$.
    Define a VC-dimension~\cite{vapnik1971} as a maximal $m$ for which $\FF$ shatters any $S_m$:
    \begin{equation}
        \VC{\FF} =
        \max \{m: \; \Pi_\FF(m) = 2^m\}.
    \end{equation}
    One can relate a growth function with a VC-dimension using the following lemma:
    \begin{lemma}[Sauer's lemma \cite{sauer1972density}]
        $\Pi_\FF(m) \leq \sum_{k=0}^{\VC{\FF}} \binom{m}{k}$.
    \end{lemma}
    Now we need to express the asymptotic behavior as $m\to\infty$ in a convenient way.
    Let $d = \VC{\FF}$.
    For $m \leq d$ $\Pi_\FF(m) = 2^m$; consider $m > d$:
    \begin{equation}
        \Pi_\FF(m) \leq
        \sum_{k=0}^{d} \binom{m}{k} \leq
        \left(\frac{m}{d}\right)^{d} \sum_{k=0}^{d} \binom{m}{k} \left(\frac{d}{m}\right)^{k} \leq
        \left(\frac{m}{d}\right)^{d} \sum_{k=0}^{m} \binom{m}{k} \left(\frac{d}{m}\right)^{k} =
        \left(\frac{m}{d}\right)^{d} \left(1 + \frac{d}{m}\right)^m \leq
        \left(\frac{e m}{d}\right)^d.
    \end{equation}
    Substituting it into~(\ref{eq:rad_bound_zero-one_final}) gives the final bound:
    \begin{equation}
        \Rad{r \circ \FF}{S_m} \leq
        \sqrt{\frac{2}{m} \left(\log 2 + \VC{\FF} \left(1 + \log m - \log \VC{\FF}\right)\right)} =
        \Theta_{m\to\infty}\left(\sqrt{2 \VC{\FF} \frac{\log m}{m}}\right).
    \end{equation}
    Hence for the bound to be non-vacuous having $\VC{\FF} < m/(2\log m)$ is necessary.
    According to \cite{bartlett2019nearly}, whenever $\FF$ denotes a set of all models realizable by a fully-connected network of width $U$ with $W$ parameters, $\VC{\FF} = \Theta(WU)$.
    While the constant is not give here, this results suggests that the corresponding bound will be vacuous for realistic nets trained on realistic datasets since $W \gg m$ there.

    \subsubsection{The case of a margin risk}

    Suppose now $r$ is a $\gamma$-margin risk: $r(y,z) = r_\gamma(y,z) = [y z < \gamma]$.
    In this case we can bound the true $0/1$-risk as:
    \begin{equation}
        R_{0/1}(\hat f_m) \leq
        R_\gamma(\hat f_m) \leq
        \hat R_{m,\gamma}(\hat f_m) + 2 \EE_{S_m'} \Rad{r_\gamma \circ \FF}{S_m'} + \sqrt{\frac{1}{2m} \log \frac{1}{\delta}}
        \quad
        \text{w.p. $\geq 1 - \delta$ over $S_m$.}
    \end{equation}
    Here we have a trade-off between the train risk and the Rademacher complexity: as $\gamma$ grows larger the former term grows too, but the latter one vanishes.
    One can hope that a good enough model $\hat f_m$ should be able to classify the dataset it was trained on with a sufficient margin, i.e. $\hat R_{m,\gamma}(\hat f_m) \approx 0$ for large enough $\gamma$.

    In the case of a margin loss, a Rademacher complexity is upper-bounded with covering numbers:
    \begin{equation}
        \NN_p(\HH, \epsilon, S_m) =
        \inf_{\bar\HH \subseteq \HH} \left\{ |\bar\HH|: \; \forall h \in \HH \; \exists \bar h \in \bar\HH: \; \left(\sum_{k=1}^m |h(z_k) - \bar h(z_k)|^p\right)^{1/p} < \epsilon \right\}.
    \end{equation}
    Note that $\NN_p(\HH, \epsilon, S_m)$ grows as $\epsilon \to 0$ and decreases with $p$.
    There are several statements that link the Rademacher complexity with covering numbers; we start with a simple one:
    \begin{theorem}
        Suppose $\HH$ is a class of hypotheses supported on $[0,1]$.
        Then $\forall S_m$ and $\forall \epsilon > 0$,
        \begin{equation}
            \Rad{\HH}{S_m} \leq
            \frac{\epsilon}{m} + \sqrt{\frac{2}{m} \log(2 \NN_1(\HH, \epsilon, S_m))}.
        \end{equation}
    \end{theorem}
    \begin{proof}
        Take $\epsilon > 0$.
        Let $\bar\HH \subseteq \HH$: $\forall h \in \HH$ $\exists \bar h \in \bar\HH$: $\sum_{k=1}^m |h(z_k) - \bar h(z_k)| < \epsilon$.
        Let $\bar h[h] \in \bar\HH$ be a representative of $h \in \HH$.
        Then,
        \begin{multline}
            \Rad{\HH}{S_m} = 
            \EE_{\sigma_{1:m}} \sup_{h \in \HH} \left| \frac{1}{m} \sum_{k=1}^m \sigma_k h(z_k) \right| \leq
            \\\leq
            \EE_{\sigma_{1:m}} \sup_{h \in \HH} \left| \frac{1}{m} \sum_{k=1}^m \sigma_k (h(z_k) - \bar h[h](z_k)) \right| +
            \EE_{\sigma_{1:m}} \sup_{h \in \HH} \left| \frac{1}{m} \sum_{k=1}^m \sigma_k \bar h[h](z_k) \right|
            \\\leq
            \frac{\epsilon}{m} +
            \EE_{\sigma_{1:m}} \sup_{\bar h \in \bar\HH} \left| \frac{1}{m} \sum_{k=1}^m \sigma_k \bar h(z_k) \right| =
            \frac{\epsilon}{m} + \Rad{\bar\HH}{S_m} \leq
            \frac{\epsilon}{m} + \sqrt{\frac{2}{m} \log(2|\bar\HH|_{S_m})}.
        \end{multline}
        Taking infium with respect to $\bar\HH$ concludes the proof.
    \end{proof}
    Note that for $\gamma = 0$ $r_\gamma = r_{0/1}$ and $\NN_1(r_{0/1} \circ \FF, \epsilon, S_m) \to |\FF_{S_m}|$ as $\epsilon \to 0$; hence we get~(\ref{eq:rad_bound_zero-one_final}).

    The next theorem is more involved:
    \begin{theorem}[Dudley entropy integral \cite{dudley1967sizes}]
        \label{thm:dudley}
        Suppose $\HH$ is a class of hypotheses supported on $[-1,1]$.
        Then $\forall S_m$ and $\forall \epsilon > 0$,
        \begin{equation}
            \Rad{\HH}{S_m} \leq
            \frac{4\epsilon}{\sqrt{m}} + \frac{12}{m} \int_{\epsilon}^{\sqrt{m}/2} \sqrt{\log\NN_2(\HH, t, S_m)} \, dt.
        \end{equation}
    \end{theorem}

    Now the task is to upper-bound the covering number for $\HH = r_\gamma \circ \FF$.
    It is easier, however, to upper-bound $\tilde r_\gamma \circ \FF$ instead, where $\tilde r_\gamma$ is a soft $\gamma$-margin risk.
    Indeed,
    \begin{equation}
        \NN_p(\tilde r_\gamma \circ \FF, \epsilon, S_m) \leq
        \NN_p(\gamma^{-1} \FF, \epsilon, S_m) =
        \NN_p(\FF, \gamma \epsilon, S_m).
        \label{eq:N_p_bound}
    \end{equation}
    In this case it suffices to upper-bound the covering number for the model class $\FF$ itself.
    Note also that we still have an upper-bound for the true $0/1$-risk:
    \begin{equation}
        R_{0/1}(\hat f_m) \leq
        \tilde R_\gamma(\hat f_m) \leq
        \hat R_{m,\gamma}(\hat f_m) + 2 \EE_{S_m'} \Rad{\tilde r_\gamma \circ \FF}{S_m'} + \sqrt{\frac{1}{2m} \log \frac{1}{\delta}}
        \quad
        \text{w.p. $\geq 1 - \delta$ over $S_m$.}
    \end{equation}
    
    When $\FF$ is a set of all models induced by a neural network of a given architecture, $\NN_p(\FF, \gamma \epsilon, S_m)$ is infinite.
    Nevertheless, if restrict $\FF$ to a class of functions with uniformly bounded Lipschitz constant, the covering number becomes finite, which implies the finite conditional Rademacher complexity.
    If we moreover assume that the data have bounded support then the expected Rademacher complexity becomes finite as well.

    A set of all neural nets of a given architecture does not have a uniform Lipschitz constant, however, this is the case if we assume weight norms to be a-priori bounded.
    For instance, consider a fully-connected network $f(\cdot;W_{0:L})$ with $L$ hidden layers without biases.
    Assume the activation function $\phi$ to have $\phi(0) = 0$ and to be 1-Lipschitz.
    Define:
    \begin{equation}
        \FF_{s_{0:L}, b_{0:L}} =
        \{f(\cdot;W_{0:L}): \; \forall l \in [L]_0 \; \|W_l\|_2 \leq s_l, \; \|W_l^T\|_{2,1} \leq b_l\}.
    \end{equation}
    \begin{theorem}[\cite{bartlett2017spectrally}]
        \begin{equation}
            \log \NN_2(\FF_{s_{0:L}, b_{0:L}}, \epsilon, S_m) \leq
            C^2 \frac{\|X_m\|_F^2}{\epsilon^2} \mathcal{R}^2_{s_{0:L}, b_{0:L}},
        \end{equation}
        where $C = \sqrt{\log(2\max n_l^2)}$ and we have introduced a spectral complexity:
        \begin{equation}
            \mathcal{R}_{s_{0:L}, b_{0:L}} =
            \left(\prod_{l=0}^L s_l \right) \times \left(\sum_{l=0}^L (b_l/s_l)^{2/3}\right)^{3/2} =
            \left(\sum_{l=0}^L \left(b_l \prod_{l'\neq l} s_{l'} \right)^{2/3} \right)^{3/2}.
        \end{equation}
    \end{theorem}
    Plugging this result into Theorem~\ref{thm:dudley} and noting eq.~(\ref{eq:N_p_bound}) gives:
    \begin{equation}
        \Rad{\tilde r_\gamma \circ \FF_{s_{0:L}, b_{0:L}}}{S_m} \leq
        \frac{4\epsilon}{\sqrt{m}} + \frac{12}{m} \int_{\epsilon}^{\sqrt{m}/2} C \frac{\|X_m\|_F}{\gamma t} \mathcal{R}_{s_{0:L}, b_{0:L}} \, dt =
        \frac{4\epsilon}{\sqrt{m}} + \frac{12}{m} C \frac{\|X_m\|_F}{\gamma} \mathcal{R}_{s_{0:L}, b_{0:L}} \log\left(\frac{\sqrt{m}}{2\epsilon}\right).
    \end{equation}
    Differentiating the right-hand side wrt $\epsilon$ gives:
    \begin{equation}
        \frac{d (rhs.)}{d\epsilon} =
        \frac{4}{\sqrt{m}} - \frac{12}{m} C \frac{\|X_m\|_F}{\gamma \epsilon} \mathcal{R}_{s_{0:L}, b_{0:L}}.
    \end{equation}
    Hence the optimal $\epsilon$ is given by:
    \begin{equation}
        \epsilon_{opt} =
        \frac{3}{\sqrt{m}} C \frac{\|X_m\|_F}{\gamma} \mathcal{R}_{s_{0:L}, b_{0:L}}.
    \end{equation}
    Plugging it back into the bound gives:
    \begin{equation}
        \Rad{\tilde r_\gamma \circ \FF_{s_{0:L}, b_{0:L}}}{S_m} \leq
        \frac{12}{m} C \frac{\|X_m\|_F}{\gamma} \mathcal{R}_{s_{0:L}, b_{0:L}} \left(1 - \log\left(\frac{6}{m} C \frac{\|X_m\|_F}{\gamma} \mathcal{R}_{s_{0:L}, b_{0:L}}\right)\right).
    \end{equation}

    \subsubsection{From an a-priori bound to an a-posteriori bound}
    \label{sec:a_posteriori_uniform_bound}

    We thus have obtained a bound for a class of neural networks with a-priori bounded weight norms.
    Let $\theta$ be the following set of network weights:
    \begin{equation}
        \theta(i_{0:L}, j_{0:L}) =
        \{W_{0:L}: \; \forall l \in [L]_0 \; \|W_l\|_2 \leq s_l(i_l), \; \|W_l^T\|_{2,1} \leq b_l(j_l)\},
    \end{equation}
    where $s_l(\cdot)$ and $b_l(\cdot)$ are strictly monotonic functions on $\mathbb{N}$ growing to infinity.
    Correspondingly, define a set of failure probabilities:
    \begin{equation}
        \delta(i_{0:L}, j_{0:L}) =
        \frac{\delta}{\prod_{l=0}^L(i_l (i_l+1) j_l (j_l+1))}.
    \end{equation}
    Note the following equality:
    \begin{equation}
        \sum_{j=1}^\infty \frac{1}{j (j+1)} =
        \sum_{j=1}^\infty \left(\frac{1}{j} - \frac{1}{j+1}\right) =
        1.
    \end{equation}
    This implies the following:
    \begin{equation}
        \sum_{i_0=1}^\infty \ldots \sum_{i_L=1}^\infty \sum_{j_0=1}^\infty \ldots \sum_{j_L=1}^\infty \delta(i_{0:L}, j_{0:L}) =
        \delta.
    \end{equation}
    Hence by applying the union bound, the following holds with probability $\geq 1 - \delta$ over $S_m$:
    \begin{equation}
        \sup_{f \in \FF_{s_{0:L}(i_{0:L}), b_{0:L}(j_{0:L})}}(R(f) - \hat R_m(f)) \leq
        \EE_{S'_m} \Rad{\tilde r_\gamma \circ \FF_{s_{0:L}(i_{0:L}), b_{0:L}(j_{0:L})}}{S_m'} + \sqrt{\frac{1}{2m} \log \frac{1}{\delta(i_{0:L}, j_{0:L})}}
        \quad
        \forall i_l, j_l \in \mathbb{N}.
    \end{equation}
    Take the set of smallest $i_{0:L}$ and $j_{0:L}$ such that $\|\hat W_l\|_2 < i_l / L$ and $\|\hat W_l^T\|_{2,1} < j_l / L$ $\forall l \in [L]_0$ for $\hat W_{0:L}$ being the weights of a learned network $\hat f_m = \AA(S_m)$.
    Denote 
    by $i^*_{0:L}$ and $j^*_{0:L}$ the sets mentioned above; let $s^*_{0:L} = s_{0:L}(i_{0:L}^*)$ and $b^*_{0:L} = b_{0:L}(j_{0:L}^*)$.
    Given this, $\hat f_m \in \FF_{s^*_{0:L}, b^*_{0:L}}$, and we have with probability $\geq 1 - \delta$ over $S_m$:
    \begin{equation}
        R(\hat f_m) - \hat R_m(\hat f_m) \leq
        \sup_{f \in \FF_{s_{0:L}^*, b_{0:L}^*}}(R(f) - \hat R_m(f)) \leq
        \EE_{S'_m} \Rad{\tilde r_\gamma \circ \FF_{s_{0:L}^*, b_{0:L}^*}}{S_m'} + \sqrt{\frac{1}{2m} \log \frac{1}{\delta(i_{0:L}^*, j_{0:L}^*)}}.
    \end{equation}
    Let us express the corresponding spectral complexity in a more convenient form.
    \begin{equation}
        \mathcal{R}_{s_{0:L}^*, b_{0:L}^*} =
        \left(\sum_{l=0}^L \left(b_l^* \prod_{l'\neq l} s_{l'}^* \right)^{2/3} \right)^{3/2} \leq
        \left(\sum_{l=0}^L \left((\|\hat W_l^T\|_{2,1} + \Delta b_l^*) \prod_{l'\neq l} (\|\hat W_{l'}\|_2 + \Delta s_l^*) \right)^{2/3} \right)^{3/2},
    \end{equation}
    where $\Delta s_l^* = s_{l+1}^* - s_l^*$ and $\Delta b_l^* = b_{l+1}^* - b_l^*$.
    At the same time,
    \begin{equation}
        \mathcal{R}_{s_{0:L}^*, b_{0:L}^*} \geq
        \left(\sum_{l=0}^L \left(\|\hat W_l^T\|_{2,1} \prod_{l'\neq l} \|\hat W_{l'}\|_2 \right)^{2/3} \right)^{3/2}.
    \end{equation}
    These two bounds together give an upper-bound for $\Rad{\tilde r_\gamma \circ \FF_{s_{0:L}^*, b_{0:L}^*}}{S_m'}$ that depends explicitly on learned weight norms but not on $s_{0:L}^*$ and $b_{0:L}^*$.

    Note that $i_l^* = i_l(s_l^*) \leq i_l(\|\hat W_l\|_2) + 1$ and $j_l^* = j_l(b_l^*) \leq j_l(\|\hat W_l^T\|_{2,1}) + 1$ $\forall l \in [L]_0$, where $i_l(s_l)$ and $j_l(b_l)$ are inverse maps for $s_l(i_l)$ and $b_l(j_l)$, respectively.
    This gives an upper-bound for $\log (\delta(i_{0:L}^*, j_{0:L}^*)^{-1}):$
    \begin{multline}
        \log \frac{1}{\delta(i_{0:L}^*, j_{0:L}^*)} \leq
        \\\leq
        \log\frac{1}{\delta} + \sum_{l=0}^L \left(\log(1 + i_l(\|\hat W_l\|_2)) + \log(2 + i_l(\|\hat W_l\|_2)) + \log(1 + j_l(\|\hat W_l^T\|_{2,1})) + \log(2 + j_l(\|\hat W_l^T\|_{2,1}))\right).
    \end{multline}

    To sum up, we have expressed the bound on test-train risk difference in terms of the weights of the learned model $\hat f_m = \AA(S_m)$, thus arriving at an a-posteriori bound.
    Note that the bound is valid for any sequences $s_l(i_l)$ and $b_l(j_l)$ taken before-hand.
    Following \cite{bartlett2017spectrally}, we can take, for instance, $s_l(i_l) = i_l / L$ and $b_l(j_l) = j_l / L$.

    \subsection{Failure of uniform bounds}

    Recall the general uniform bound:
    \begin{equation}
        R(\hat f_m) - \hat R_m(\hat f_m) \leq
        \sup_{f\in\FF} (R(f) - \hat R_m(f)),
        \label{eq:general_uniform_bound}
    \end{equation}
    where $\hat f_m = \AA(S_m)$.
    We have already discussed that the bound fails if $\FF$ contains a "bad model" for which $R(f)$ is large, while $\hat R_m(f)$ is small; hence we are interested to take $\FF$ as small as possible.
    We have also noted that the smallest $\FF$ we can consider is $\AA(\MM)$, where $\MM = \supp(\DD^m)$.

    Consider now the ideal case: 
    $\exists \epsilon > 0:$ $R(f) < \epsilon$ $\forall f \in \FF$.
    In other words, all models of the class $\FF$ generalize well.
    In this case the bound~(\ref{eq:general_uniform_bound}) becomes simply:
    \begin{equation}
        R(\hat f_m) - \hat R_m(\hat f_m) \leq
        \epsilon
        \quad
        \text{w.p. $\geq 1 - \delta$ over $S_m$,}
    \end{equation}
    which is perfect.
    Our next step was to apply McDiarmid's inequality: see eq.~(\ref{eq:mcdiarmid_bound_for_sup}); in our case this results in:
    \begin{equation}
        \sup_{f \in \FF}(R(f) - \hat R_m(f)) \leq
        \epsilon + \sqrt{\frac{1}{2m} \log \frac{1}{\delta}}
        \quad
        \text{w.p. $\geq 1 - \delta$ over $S_m$,}
    \end{equation}
    which is almost perfect as well.
    What happened then, is we tried to upper-bound the expected supremum:
    \begin{equation}
        \EE_{S'_m} \sup_{f \in \FF}(R(f) - \hat R'_m(f)) =
        \EE_{S'_m} \sup_{f \in \FF}(\EE_{S''_m} \hat R''_m(f) - \hat R'_m(f)) \leq
        \EE_{S'_m} \EE_{S''_m} \sup_{f \in \FF}(\hat R''_m(f) - \hat R'_m(f)).
    \end{equation}
    The last step is called "symmetrization".
    Note that having small true error does not imply having small empirical error on any train dataset.
    \cite{nagarajan2019uniform} constructed a learning setup for which for any $S_m''$ there exists a model $\tilde f_m \in \FF$ such that $\hat R''_m(\tilde f_m) \approx 1$; this is true even for $\FF = \AA(\MM^m)$.
    Specifically, they provided a simple algorithm to construct a specific dataset $\neg(S_m'')$ and take $\tilde f_m = \AA(\neg (S_m''))$.

    \section{PAC-bayesian bounds}

    \subsection{At most countable case}

    Recall the following bound for finite $\FF$:
    \begin{multline}
        \PP\left(\sup_{f\in\FF} (R(f) - \hat R_m(f)) \geq \epsilon\right) =
        \PP\left(\exists f\in\FF: \; (R(f) - \hat R_m(f)) \geq \epsilon\right) \leq\\\leq
        \sum_{f\in\FF} \PP(R(f) - \hat R_m(f) \geq \epsilon) \leq
        |\FF| e^{-2 m \epsilon^2}
        \quad
        \forall \epsilon > 0.
    \end{multline}
    When $\FF$ has infinite cardinality, the bound above still holds, but it is vacuous.
    Consider at most countable $\FF$ and $\epsilon$ that depends on $f$.
    If we take $\epsilon(f)$ for which $\sum_{f \in \FF} e^{-2 m \epsilon^2(f)}$ is finite, then we arrive into the finite bound:
    \begin{equation}
        \PP\left(\exists f\in\FF: \; (R(f) - \hat R_m(f)) \geq \epsilon(f)\right) \leq
        \sum_{f\in\FF} \PP\left(R(f) - \hat R_m(f) \geq \epsilon(f)\right) \leq
        \sum_{f \in \FF} e^{-2 m \epsilon^2(f)}
        \quad
        \forall \epsilon > 0.
    \end{equation}
    For instance, consider some probability distribution $P(f)$ on $\FF$.
    Take $\epsilon(f)$ such that $e^{-2 m \epsilon^2(f)} = P(f) e^{-2 m \tilde\epsilon^2}$ for some $\tilde\epsilon \in \RR_+$.
    Solving this equation gives:
    \begin{equation}
        \epsilon(f) =
        \tilde\epsilon + \sqrt{\frac{1}{2m} \log\frac{1}{P(f)}}.
    \end{equation}
    Hence we have $\forall \tilde\epsilon > 0$:
    \begin{equation}
        \PP\left(\exists f\in\FF: \; (R(f) - \hat R_m(f)) \geq \tilde\epsilon + \sqrt{\frac{1}{2m} \log\frac{1}{P(f)}}\right) \leq
        e^{-2 m \tilde\epsilon^2}.
    \end{equation}
    Or, equivalently, w.p. $\geq 1 - \delta$ over $S_m$ we have $\forall f \in \FF$:
    \begin{equation}
        R(f) - \hat R_m(f) \leq
        \sqrt{\frac{1}{2m} \left(\log \frac{1}{\delta} + \log\frac{1}{P(f)} \right)}.
        \label{eq:discrete_pac_bayesian_bound}
    \end{equation}

    \subsection{General case}

    Let us refer $P(f)$ as a "prior distribution".
    Suppose our learning algorithm is stochastic and outputs a model distribution $Q(f)$ which we shall refer as a "posterior":
    \begin{equation}
        \hat f_m \sim 
        \hat Q_m =
        \AA(S_m).
    \end{equation}
    We shall now prove the following theorem:
    \begin{theorem}[\cite{mcallester1999pac}]
        \label{thm:mc_allester}
        For any $\delta \in (0,1)$ w.p. $\geq 1 - \delta$ over $S_m$ we have:
        \begin{equation}
            R(\hat Q_m) - \hat R_m(\hat Q_m) \leq
            \sqrt{\frac{1}{2m-1} \left(\log \frac{4m}{\delta} + \kld{\hat Q_m}{P} \right)},
        \end{equation}
        where $R(Q) = \EE_{f \sim Q} R(f)$ and $\hat R_m(Q) = \EE_{f \sim Q} \hat R_m(f)$.
    \end{theorem}
    \begin{proof}

        The proof relies on the following lemmas:
        \begin{lemma}[\cite{mcallester1999pac}]
            \label{lemma:mc_allester}
            For any probability distribution $P$ on $\FF$, for any $\delta \in (0,1)$ w.p. $\geq 1 - \delta$ over $S_m$ we have:
            \begin{equation}
                \EE_{f \sim P} e^{(2m-1) (\Delta_m(f))^2} \leq
                \frac{4m}{\delta},
            \end{equation}
            where $\Delta_m(f) = |R(f) - \hat R_m(f)|$.
        \end{lemma}
        \begin{lemma}[\cite{donsker1985large}]
            \label{lemma:dv}
            Let $P$ and $Q$ be probability distributions on $X$.
            Then for any $h: \; X \to \RR$
            \begin{equation}
                \EE_{x\sim Q} h(x) \leq
                \log\EE_{x \sim P} e^{h(x)} + \kld{Q}{P}.
            \end{equation}
        \end{lemma}

        From D-V lemma, taking $X = \FF$, $h = (2m-1) \Delta_m^2$, and $Q = \hat Q_m$:
        \begin{equation}
            \EE_{f \sim \hat Q_m} (2m-1) (\Delta_m(f))^2 \leq
            \log\EE_{f \sim P} e^{(2m-1) (\Delta_m(f))^2} + \kld{\hat Q_m}{P}.
        \end{equation}
        Hence from Lemma~\ref{lemma:mc_allester}, w.p. $\geq 1 - \delta$ over $S_m$ we have:
        \begin{equation}
            \EE_{f \sim \hat Q_m} (2m-1) (\Delta_m(f))^2 \leq
            \log\frac{4m}{\delta} + \kld{\hat Q_m}{P}.
        \end{equation}
        A simple estimate concludes the proof:
        \begin{multline}
            R(\hat Q_m) - \hat R_m(\hat Q_m) \leq
            |\EE_{f\sim \hat Q_m} (R(f) - \hat R_m(f))| \leq
            \EE_{f\sim \hat Q_m} |R(f) - \hat R_m(f)| =
            \\=
            \EE_{f\sim \hat Q_m} \Delta_m(f) \leq
            \sqrt{\EE_{f\sim \hat Q_m} (\Delta_m(f))^2} \leq
            \sqrt{\frac{1}{2m-1} \left(\log \frac{4m}{\delta} + \kld{\hat Q_m}{P} \right)}.
        \end{multline}
            
    \end{proof}

    Let us prove D-V lemma first; we shall prove in the case when $P \ll Q$ and $Q \ll P$:
    \begin{proof}[Proof of Lemma~\ref{lemma:dv}]
        \begin{multline}
            \EE_{x \sim Q} h(x) - \kld{Q}{P} =
            \EE_{x \sim Q} \left(h(x) - \log\left(\frac{dQ}{dP}(x)\right)\right) =
            \\=
            \EE_{x \sim Q} \log\left(e^{h(x)} \frac{dP}{dQ}(x)\right) \leq
            \log \EE_{x \sim Q} \left(e^{h(x)} \frac{dP}{dQ}(x)\right) =
            \log \EE_{x \sim P} e^{h(x)},
        \end{multline}
        where $dQ/dP$ is a Radon-Nikodym derivative.
    \end{proof}

    We now proceed with proving Lemma~\ref{lemma:mc_allester}:
    \begin{proof}[Proof of Lemma~\ref{lemma:mc_allester}]
        Recall Markov's inequality:
        \begin{theorem}[Markov's inequality]
            Let $X$ be a non-negative random variable.
            Then $\forall a > 0$
            \begin{equation}
                \PP(X \geq a) \leq
                \frac{\EE X}{a}.
            \end{equation}
        \end{theorem}
        Hence taking $a = 4m / \delta$, it suffices to show that
        \begin{equation}
            \EE_{S_m} \EE_{f \sim P} e^{(2m-1) (\Delta_m(f))^2} \leq
            4m.
        \end{equation}
        We are going to prove a stronger property:
        \begin{equation}
            \EE_{S_m} e^{(2m-1) (\Delta_m(f))^2} \leq
            4m
            \quad
            \forall f \in \FF.
            \label{eq:markov_for_delta}
        \end{equation}
        Note that from Hoeffding's inequality we get:
        \begin{equation}
            \PP_{S_m}(\Delta_m(f) \geq \epsilon) \leq
            2 e^{-2m \epsilon^2}
            \quad
            \forall \epsilon > 0
            \quad
            \forall f \in \FF.
            \label{eq:hoeffding_for_delta}
        \end{equation}

        First, assume that the distribution of $\Delta_m(f)$ has density $\forall f \in \FF$; denote it by $p_f(\Delta)$.
        In this case we can directly up`per-bound the expectation over $S_m$:
        \begin{multline}
            \EE_{S_m} e^{(2m-1) (\Delta_m(f))^2} =
            \int_0^\infty e^{(2m-1) \epsilon^2} p_f(\epsilon) \, d\epsilon =
            \int_0^\infty e^{(2m-1) \epsilon^2} \frac{d}{d\epsilon} \left(-\int_\epsilon^\infty p_f(\Delta) \, d\Delta\right) \, d\epsilon =
            \\=
            \left.-\left(e^{(2m-1) \epsilon^2} \int_\epsilon^\infty p_f(\Delta) \, d\Delta\right) \right|_{\epsilon=0}^\infty +
            2 (2m-1) \int_0^\infty \epsilon e^{(2m-1) \epsilon^2} \int_\epsilon^\infty p_f(\Delta) \, d\Delta \, d\epsilon \leq
            \\\leq
            \int_0^\infty p_f(\Delta) \, d\Delta +
            2 (2m-1) \int_0^\infty \epsilon e^{(2m-1) \epsilon^2} \int_\epsilon^\infty p_f(\Delta) \, d\Delta \, d\epsilon \leq
            \\\leq
            2 + 4 (2m-1) \int_0^\infty \epsilon e^{(2m-1) \epsilon^2} e^{-2m \epsilon^2} \, d\epsilon =
            2 + 4 (2m-1) \int_0^\infty \epsilon e^{-\epsilon^2} \, d\epsilon =
            2 + 2 (2m-1) =
            4m.
            \label{eq:markov_for_delta_proof}
        \end{multline}

        We now relax our assumption of density existence.
        Let $\mu_f$ be a distribution of $\Delta_m(f)$.
        Consider a class $\MM$ of all non-negative sigma-additive measures on $\RR_+$ such that a property similar to~(\ref{eq:hoeffding_for_delta}) holds:
        \begin{equation}
            \mu([\epsilon,\infty)) \leq
            2 e^{-2m \epsilon^2}
            \quad
            \forall \epsilon > 0
            \quad
            \forall \mu \in \MM.
            \label{eq:hoeffding_for_mu}
        \end{equation}
        Note that $\MM$ contains a probability distribution of $\Delta_m(f)$ for any $f \in \FF$.
        Among these measures we shall choose a specific one that maximizes an analogue of the left-hand sise of~(\ref{eq:markov_for_delta}):
        \begin{equation}
            \mu^* \in
            \Argmax_{\mu \in \MM} \int_0^\infty e^{(2m-1) \Delta^2} \, \mu(d\Delta).
        \end{equation}
        Note that constraint~(\ref{eq:hoeffding_for_mu}) states that a measure of a right tail of a real line should be upper-bounded.
        However, $\mu^*$ should have as much mass to the right as possible.
        Hence constraint~(\ref{eq:hoeffding_for_mu}) should become an equality for this specific $\mu^*$:
        \begin{equation}
            \mu^*([\epsilon,\infty)) =
            2 e^{-2m \epsilon^2}
            \quad
            \forall \epsilon > 0.
        \end{equation}
        From this follows that $\mu^*$ has density $\tilde p^*(\Delta) = 8m \Delta e^{-2m \Delta^2}$.

        Note that an inequality similar to~(\ref{eq:markov_for_delta_proof}) holds for $\tilde p^*$.
        Moreover, since $\mu^*$ maximizes $\int_0^\infty e^{(2m-1) \Delta^2} \, \mu(d\Delta)$, we have the following bound:
        \begin{equation}
            \EE_{S_m} e^{(2m-1) (\Delta_m(f))^2} =
            \EE_{\Delta \sim \mu_f} e^{(2m-1) \Delta^2} =
            \int_0^\infty e^{(2m-1) \Delta^2} \, \mu_f(d\Delta) \leq
            \int_0^\infty e^{(2m-1) \Delta^2} \tilde p^*(\Delta) \, d\Delta \leq
            4m.
        \end{equation}
    \end{proof}

    \subsection{Applying PAC-bayesian bounds to deterministic algorithms}

    Consider a deterministic learning rule $\AA(S_m) \sim \hat Q_m$, where $\hat Q_m$ is a Kronecker delta.
    While this situation is fine for at most countable case, whenever $\FF$ is uncountable and $P(f) = 0$ $\forall f \in \FF$, $\kld{\hat Q_m}{P} = \infty$ and we arrive into a vacuous bound.

    \subsubsection{Compression and coding}

    One work-around is to consider some discrete coding $c$, with $\mathrm{enc}_c()$ being an encoder and $\mathrm{dec}_c()$ being a decoder.
    We assume that $\mathrm{dec}_c(\mathrm{enc}_c(f)) \approx f$ $\forall f \in \FF$ and instantiate a bound of the form~(\ref{eq:discrete_pac_bayesian_bound}) for $\mathrm{enc}_c(f)$.
    Equivalently, we shall write $f_c$ for $\mathrm{enc}_c(f)$.
    Following~\cite{zhou2018nonvacuous}, we take a prior that prioritize models of small code-length:
    \begin{equation}
        P_c(f_c) =
        \frac{1}{Z} m(|f_c|) 2^{-|f_c|},
    \end{equation}
    where $|f_c|$ is a code-length for $f$, $m(k)$ is some probability distribution on $\mathbb{N}$, and $Z$ is a normalizing constant.
    In this case a KL-divergence is given as:
    \begin{equation}
        \kld{\delta_{f_c}}{P_c} =
        \log Z + |f_c| \log 2 - \log(m(|f_c|)).
    \end{equation}

    In order to make our bound as small as possible, we need to ensure that our learning algorithm, when fed realistic data, outputs models of small code-length.
    One can esnure this by coding not the model $f$ itself, but rather a result of its compression via a compression algorithm $\CC$.
    We assume that a compressed model $\CC(f)$ is still a model from $\FF$. 
    We also hope that its risk does not change sufficiently $R(\CC(f)) \approx R(f)$ and a learning algorithm tends to output models which in a compressed form have small code-length.
    In this case we are able to upper-bound a test-train risk difference for an encoded compressed model $\CC(f)_c$ instead of the original one.

    When our models are neural nets paramaterized with a set of weights $\theta$, a typical form of a compressed model is a tuple $(S,Q,C)$, where
    \begin{itemize}
        \item $S = s_{1:k} \subset [\dim\theta]$ are locations of non-zero weights;
        \item $C = c_{1:r} \subset \RR$ is a codebook;
        \item $Q = q_{1:k}$, $q_i \in [r]$ $\forall i \in [k]$ are quantized values.
    \end{itemize}
    Then $\CC(\theta)_i = c_{q_j}$ if $i = s_j$ else $0$.
    In this case a naive coding for 32-bit precision gives:
    \begin{equation}
        |\CC(\theta)|_c =
        |S|_c + |Q|_c + |C|_c \leq
        k (\log\dim\theta + \log r) + 32 r.
    \end{equation}

    \subsubsection{Stochastization}

    Another work-around is to volunteerly substitute $\hat f_m$ with some $\tilde Q_m$, presumably satisfying $\EE_{f \sim \tilde Q_m} f = \hat f_m$, such that $\kld{\tilde Q_m}{P}$ is finite.
    In this case we get the upper-bound for $R(\tilde Q_m)$ instead of $R(\hat f_m)$.
    One possible goal may be to obtain as better generalization guarantee as possible; in this case one can optimize the upper-bound on $R(\tilde Q_m)$ wrt $\tilde Q_m$.
    Another goal may be to get a generalization guarantee for $\hat f_m$ itself; in this case we have to somehow relate it with a generalization gurantee for $\hat Q_m$.

    Let us discuss the former goal first.
    Our goal is to optimize the upper-bound on test risk wrt a stochastic model $Q$:
    \begin{equation}
        R(Q) \leq
        \hat R_m(Q) + \sqrt{\frac{1}{2m-1} \left(\log \frac{4m}{\delta} + \kld{Q}{P} \right)} \to \min_{Q}.
    \end{equation}
    In order to make optimization via GD possible, we first substitute $\hat R_m$ with its differentiable convex surrogate $\hat \LL_m$:
    \begin{equation}
        R(Q) \leq
        \hat\LL_m(Q) + \sqrt{\frac{1}{2m-1} \left(\log \frac{4m}{\delta} + \kld{Q}{P} \right)} \to \min_{Q}.
    \end{equation}
    The second thing we have to do in order to make GD optimization feasible is switching from searching in an abstract model distribution space to searching in some euclidian space.
    Let $\FF$ be a space of models realizable by a given neural network architecture.
    Let $\theta$ denote a set of weights.
    Following~\cite{dziugaite2017computing}, we consider an optimization problem in a distribution space $\mathcal{Q}$ consisting of non-degenerate diagonal gaussians:
    \begin{equation}
        \mathcal{Q} =
        \{\NN(\mu,\diag(\exp\lambda)): \; \mu \in \RR^{\dim\theta}, \; \lambda \in \RR^{\dim\theta}\}.
    \end{equation}
    In this case we substitute a model class $\FF$ with a set of network weights $\RR^{\dim\theta}$.
    For $Q \in \mathcal{Q}$ and a gaussian prior $P = \NN(\mu^*, \exp\lambda^* I)$ the KL-divergence is given as follows:
    \begin{equation}
        \kld{Q}{P} =
        \frac{1}{2} \left(e^{-\lambda^*} \left(\left\|e^{\lambda}\right\|_1 + \|\mu - \mu^*\|_2^2\right) + \dim\theta (\lambda^* - 1) - 1 \cdot \mu\right).
    \end{equation}
    Since both the KL term and the loss term are differentiable wrt $(\mu,\lambda)$ we can optimize the test risk bound via GD.
    \cite{dziugaite2017computing} suggest starting the optimization process from $\mu^* = \hat\theta_m$, where $\hat\theta_m$ is the set of weights for a model $\hat f_m = \AA(S_m)$, and $\lambda^*$ being a sufficiently large negative number.

    The next question is how to choose the prior.
    Note that the distribution we finally choose as a result of the bound optimization does not take stochasticity of the initialization $\theta^{(0)}$ of the algorithm $\AA$ that finds $\hat \theta_m$ into account.
    For this reason, the prior can depend on $\theta^{(0)}$; following \cite{dziugaite2017computing}, we take $\mu^* = \theta^{(0)}$.
    The rationale for this is that in this case the KL-term depends on $\|\mu-\theta^{(0)}\|_2^2$.
    If we hope that the both optimization processes do not lead us far away from their initializations, the KL-term will not be too large.

    As for the prior log-standard deviation $\lambda^*$, we apply the same technique as for obtaining an a-posteriori uniform bound: see Section~\ref{sec:a_posteriori_uniform_bound}.
    Define $\lambda^*_j = \log c - j / b$, where $c, b > 0$, $j \in \mathbb{N}$.
    Take $\delta_j = 6 \delta / (\pi^2 j^2)$.
    Then we get a valid bound for any $j \geq 1$:
    \begin{equation}
        R(Q) \leq
        \hat\LL_m(Q) + \sqrt{\frac{1}{2m-1} \left(\log \frac{4m}{\delta_j} + \kld{Q}{P(\mu^*,\lambda^*_j)} \right)}
        \quad
        \text{w.p. $\geq 1-\delta_j$ over $S_m$.}
    \end{equation}
    A union bound gives:
    \begin{equation}
        R(Q) \leq
        \hat\LL_m(Q) + \sqrt{\frac{1}{2m-1} \left(\log \frac{4m}{\delta_j} + \kld{Q}{P(\mu^*,\lambda^*_j)} \right)}
        \quad
        \forall j \in \mathbb{N}
        \quad
        \text{w.p. $\geq 1-\delta$ over $S_m$.}
    \end{equation}
    This allows us to optimize the bound wrt $j$.
    However, optimization wrt real numbers is preferable since this allows us applying GD.
    In order to achieve this, we express $j$ as a function of $\lambda^*$: $j = b (\log c - \lambda^*)$.
    This gives us the following expression:
    \begin{equation}
        R(Q) \leq
        \hat\LL_m(Q) + \sqrt{\frac{1}{2m-1} \left(\log \frac{2\pi^2 m b (\log c - \lambda^*)}{3\delta} + \kld{Q}{P(\mu^*,\lambda^*)} \right)}
        \quad
        \forall \lambda^* \in \{\lambda^*_j\}_{j=1}^\infty
        \quad
        \text{w.p. $\geq 1-\delta$.}
        \label{eq:dziugaite_roy_bound}
    \end{equation}
    The expression above allows us to optimize its right-hand side wrt $\lambda^*$ via GD.
    However, we cannot guarantee that the optimization result lies in $\{\lambda^*_j\}_{j=1}^\infty$.
    To remedy this, \cite{dziugaite2017computing} simply round the result to the closest $\lambda^*$ in this set.
    To sum up, we take $\mu^* = \theta^{(0)}$ and optimize the bound~(\ref{eq:dziugaite_roy_bound}) wrt $\mu$, $\lambda$, and $\lambda^*$ via GD.

    \subsubsection{A bound for a deterministic model}

    Recall in the previous section we aimed to search for a stochastic model that optimizes the upper-bound for the test risk.
    In the current section we shall discuss how to obtain a bound for a given model deterministic model $\hat f_m$ in a PAC-bayesian framework.

    Consider a neural network $f_\theta$ with $L-1$ hidden layers with weights $\theta$ without biases; let $\phi(\cdot)$ be an activation function.
    Suppose our learning algorithm $\AA$ outputs weights $\hat\theta_m$ when given a dataset $S_m$.
    In our current framework, both the prior and the posterior are distributions on $\RR^{\dim\theta}$.
    Note that McAllester's theorem (Theorem~\ref{thm:mc_allester}) requires computing KL-divergence between two distributions in model space.
    Nevertheless, noting that weights are mapped to models surjectively, we can upper-bound this term with KL-diveregence in the weight space:
    \begin{corollary}[of Theorem~\ref{thm:mc_allester}]
        \label{cor:of_thm_mc_allester}
        For any $\delta \in (0,1)$ w.p. $\geq 1 - \delta$ over $S_m$ we have:
        \begin{equation}
            R(\hat Q_m) \leq \hat R_m(\hat Q_m) +
            \sqrt{\frac{1}{2m-1} \left(\log \frac{4m}{\delta} + \kld{\hat Q_m}{P} \right)},
        \end{equation}
        where $R(Q) = \EE_{\theta \sim Q} R(f_\theta)$ and $\hat R_m(Q) = \EE_{\theta \sim Q} \hat R_m(f_\theta)$.
    \end{corollary}

    For deterministic $\AA$, our $\hat Q_m$ is degenerate, and the bound is vacuous.
    The bound is, however, valid for any distribution $\tilde Q_m$ in the weight space.
    We take $\tilde Q_m = \NN(\hat\theta_m,\sigma^2 I_{\dim\theta})$ for some $\sigma$ given before-hand.
    We take the prior as $P = \NN(0,\sigma^2 I_{\dim\theta})$; in this case the train risk term and the KL-term in the right-hand side of Corollary~\ref{cor:of_thm_mc_allester} are given as follows:
    \begin{equation}
        \hat R_m(\tilde Q_m) =
        \EE_{\xi \sim \NN(0,\sigma^2 I_{\dim\theta})} \hat R_m\left(f_{\hat\theta_m+\xi}\right),
        \qquad
        \kld{\tilde Q_m}{P} =
        \frac{\|\theta\|_2^2}{2\sigma^2}.
    \end{equation}
    This gives us the upper-bound for $R(\tilde Q_m)$; our goal is, however, to bound $R(\hat f_m)$ instead.
    The following lemma tells us that it is possible to substitute a risk of a stochastic model with a margin risk of a deterministic one:
    \begin{lemma}[\cite{neyshabur2018a}]
        \label{lemma:1_of_neyshabur_et_al}
        Let the prior $P$ has density $p$.
        For any $\delta \in (0,1)$ w.p. $\geq 1 - \delta$ over $S_m$, for any deterministic $\theta$ and a random a.c. $\xi$ such that
        \begin{equation}
            \PP_\xi\left(\max_{x \in \XX} |f_{\theta+\xi}(x) - f_{\theta}(x)| < \frac{\gamma}{2}\right) \geq
            \frac{1}{2}
            \label{eq:lemma_1_condition}
        \end{equation}
        we have:
        \begin{equation}
            R(f_\theta) \leq \hat R_{m,\gamma}(f_\theta) +
            \sqrt{\frac{1}{2m-1} \left(\log \frac{16m}{\delta} + 2 \kld{q'}{p} \right)},
        \end{equation}
        where $q'$ denotes a probability density of $\theta + \xi$.
    \end{lemma}
    This lemma requires the noise $\xi$ to conform some property; the next lemma will help us to choose the standard deviation $\sigma$ accordingly:
    \begin{lemma}[\cite{neyshabur2018a}]
        \label{lemma:2_of_neyshabur_et_al}
        Let $\phi(z) = [z]_+$.
        For any $x \in \XX_B$, where $\XX_B = \{x \in \XX: \; \|x\|_2 \leq B\}$, for any $\theta = \mathrm{vec}(\{W_l\}_{l=1}^L)$, and for any $\xi = \mathrm{vec}(\{U_l\}_{l=1}^L)$ such that $\forall l \in [L]$ $\|U_l\|_2 \leq L^{-1} \|W_l\|_2$,
        \begin{equation}
            |f_{\theta + \xi}(x) - f_\theta(x)| \leq
            e B \left(\prod_{l=1}^L \|W_l\|_2\right) \sum_{l=1}^L \frac{\|U_l\|_2}{\|W_l\|_2}.
        \end{equation}
    \end{lemma}
    These lemmas will allow us to prove the following result:
    \begin{theorem}[\cite{neyshabur2018a}]
        \label{thm:neyshabur_et_al}
        Assume $\supp x = \XX_B$ and $\phi(z) = [z]_+$; let $n = \max_l n_l$.
        For any $\delta \in (0,1)$ w.p. $\geq 1 - \delta$ over $S_m$ we have for any $\hat\theta_m$
        \begin{equation}
            R\left(f_{\hat\theta_m}\right) \leq
            \hat R_{\gamma,m}\left(f_{\hat\theta_m}\right) + 
            \sqrt{\frac{1}{2m-1} \left(\log \frac{8Lm}{\delta} + \frac{1}{2L} \log m + 8 e^4 \left(\frac{B \mathcal{R}(\theta)}{\gamma}\right)^2 L^2 n \log(2 L n) \right)},
        \end{equation}
        where we have defined a spectral complexity:
        \begin{equation}
            \mathcal{R}(\theta) =
            \left(\prod_{l=1}^L \|W_l\|_2\right) \sqrt{\sum_{l=1}^L \frac{\|W_l\|_F^2}{\|W_l\|_2^2}}.
        \end{equation}
    \end{theorem}
    Compare with the result of Bartlett and coauthors:
    \begin{theorem}[\cite{bartlett2017spectrally}]
        \label{thm:bartlett_et_al}
        Assume $\supp x = \XX_B$ and $\phi(z) = [z]_+$; let $n = \max_l n_l$.
        For any $\delta \in (0,1)$ w.p. $\geq 1 - \delta$ over $S_m$ we have for any $\hat\theta_m$
        \begin{equation}
            R\left(f_{\hat\theta_m}\right) \leq
            \hat R_{\gamma,m}\left(f_{\hat\theta_m}\right) + \Rad{\tilde r_\gamma \circ \FF_{\leq\hat\theta_m}}{S_m} + \sqrt{\frac{1}{2m} \log\frac{1}{\delta}},
        \end{equation}
        where we upper-bound the Rademacher complexity as
        \begin{equation}
            \Rad{\tilde r_\gamma \circ \FF_{\leq\theta}}{S_m} \leq
            \frac{C}{\sqrt{m}} \frac{B \mathcal{R}(\theta, L^{-1})}{\gamma} \sqrt{\log(2n)} \left(1 - \log\left(\frac{C}{2\sqrt{m}} \frac{B \mathcal{R}(\theta, 0)}{\gamma} \sqrt{\log(2n)}\right)\right),
        \end{equation}
        and we define a spectral complexity as
        \begin{equation}
            \mathcal{R}(\theta, \Delta) =
            \left(\prod_{l=1}^L (\|W_l\|_2 + \Delta)\right) \sqrt{\sum_{l=1}^L \frac{(\|W_l^T\|_{2,1}+\Delta)^2}{(\|W_l\|_2+\Delta)^2}}.
        \end{equation}
    \end{theorem}
    Both bounds grow linearly with $(B/\gamma) \prod_{l=1}^L \|W_l\|_2$, which is a very natural property.
    While the former result is simpler, the latter does not depend explicitly on depth $L$ and width $n$.
    Nevertheless, the proof of the latter result is rather technically involved, while the proof of the former can be reproduced without substantial efforts.

    \begin{proof}[Proof of Theorem~\ref{thm:neyshabur_et_al}]
        First of all, define:
        \begin{equation}
            \beta =
            \left(\prod_{l=1}^L \|W_l\|_2\right)^{1/L},
            \qquad
            \tilde W_l =
            \frac{\beta}{\|W_l\|_2} W_l.
        \end{equation}
        Since ReLU is homogeneous, $f_{\tilde\theta} = f_\theta$.
        Also, $\prod_{l=1}^L \|W_l\|_2 = \prod_{l=1}^L \|\tilde W_l\|_2$ and $\sum_{l=1}^L \frac{\|\tilde W_l\|_F^2}{\|\tilde W_l\|_2^2} = \sum_{l=1}^L \frac{\|W_l\|_F^2}{\|W_l\|_2^2}$.
        Hence both the model and the bound do not change if we substitute $\theta$ with $\tilde\theta$.
        Hence w.l.o.g. assume that $\|W_l\|_2 = \beta$ $\forall l \in [L]$.

        We now use Lemma~\ref{lemma:2_of_neyshabur_et_al} to find $\sigma > 0$ for which the condition of Lemma~\ref{lemma:1_of_neyshabur_et_al} is satisfied.
        In particular, we have to upper-bound the probability for $\|U_l\|_2 \geq \beta / L$ for some $l \in [L]$.
        Notice that for $\xi \sim \NN(0, \sigma^2 I_{\dim\theta})$ $U_l$ has i.i.d. zero-centered gaussian entries $\forall l \in [L]$.
        In a trivial case of $1 \times 1$ matrices, we can apply a simple tail bound:
        \begin{equation}
            \PP_{\xi \sim \NN(0,\sigma^2)} (|\xi| \geq \epsilon) =
            2 \PP_{\xi \sim \NN(0,1)} \left(\xi \geq \frac{\epsilon}{\sigma}\right) \leq
            2 e^{-\frac{\epsilon^2}{2\sigma^2}}.
        \end{equation}
        This bound follows from Chernoff bound, which is a simple corollary of Markov's inequality:
        \begin{theorem}[Chernoff bound]
            For a real-valued random variable $X$, for any $a \in \RR$, and for any $\lambda \in \RR_+$ we have:
            \begin{equation}
                \PP(X \geq a) \leq
                \frac{\EE e^{\lambda X}}{e^{\lambda a}}.
            \end{equation}
        \end{theorem}
        Indeed,
        \begin{equation}
            \PP_{\xi \sim \NN(0,1)} (\xi \geq a) \leq
            \frac{\EE e^{\lambda \xi}}{e^{\lambda a}} \leq
            e^{-\sup_{\lambda} \left(\lambda a - \log \EE e^{\lambda \xi}\right)} =
            e^{-\sup_{\lambda} \left(\lambda a - \frac{\lambda^2}{2}\right)} =
            e^{-\sup_{\lambda} \left(-\frac{1}{2} \left(\lambda - a\right)^2 + \frac{a^2}{2}\right)} =
            e^{-\frac{a^2}{2}},
        \end{equation}
        where we have used the moment-generating function for gaussians:
        \begin{equation}
            \EE_{\xi \sim \NN(0,1)} e^{\lambda \xi} =
            \sum_{k=0}^\infty \frac{\lambda^k \EE \xi^k}{k!} =
            \sum_{k=0}^\infty \frac{\lambda^{2k} (2k-1)!!}{(2k)!} =
            \sum_{k=0}^\infty \frac{\lambda^{2k}}{(2k)!!} =
            \sum_{k=0}^\infty \frac{\lambda^{2k}}{2^k k!} =
            e^{\frac{\lambda^2}{2}}.
        \end{equation}
        We can apply the same bound for a linear combination of i.i.d. standard gaussians:
        \begin{equation}
            \PP_{\xi_{1:m} \sim \NN(0,1)} \left(\left|\sum_{i=1}^m a_i \xi_i\right| \geq \epsilon\right) =
            \PP_{\xi \sim \NN(0,\sum_{i=1}^m a_i^2)} (|\xi| \geq \epsilon) =
            2 e^{-\frac{\epsilon^2}{2\sum_{i=1}^m a_i^2}}.
        \end{equation}
        Moreover, a similar bound holds for matrix-linear combinations:
        \begin{theorem}[\cite{tropp2011user}]
            Let $A_{1:m}$ be $n \times n$ deterministic matrices and let $\xi_{1:m}$ be i.i.d. standard gaussians.
            Then
            \begin{equation}
                \PP\left(\left\|\sum_{i=1}^m \xi_i A_i\right\|_2 \geq \epsilon\right) \leq
                n e^{-\frac{\epsilon^2}{2 \left\|\sum_{i=1}^m A_i^2\right\|_2}}.
            \end{equation}
        \end{theorem}
        \emph{What don't we have a factor of 2 here?}

        Let us return to bounding the probability of $\|U_l\|_2 \geq \beta/L$.
        For any $l \in [L]$ Tropp's theorem gives:
        \begin{equation}
            \PP(\|U_l\|_2 \geq t) \leq
            \PP(\|\tilde U_l\|_2 \geq t) =
            \PP_{\xi_{1:n,1:n} \sim \NN(0,1)}\left(\left\|\sigma \sum_{i,j=1}^n \xi_{ij} 1_{ij}\right\|_2 \geq t\right) \leq
            n e^{-\frac{t^2}{2 \sigma^2 n}},
        \end{equation}
        where $\tilde U_l$ is a $n \times n$ matrix with entries:
        \begin{equation}
            \tilde U_{l,ij} =
            \begin{cases}
                U_{l,ij}, & 1 \leq i \leq n_l, \; 1 \leq j \leq n_{l+1},\\
                \NN(0,\sigma^2), & \text{else.}
            \end{cases}
        \end{equation}
        Hence by a union bound:
        \begin{equation}
            \PP(\forall l \in [L] \; \|U_l\|_2 \geq t) \leq
            L n e^{-\frac{t^2}{2 \sigma^2 n}}.
        \end{equation}
        Equating the right-hand side to $1/2$ gives $t = \sigma \sqrt{2n \log(2 L n)}$.
        Next, taking $t \leq \beta/L$ gives 
        \begin{equation}
            \sigma \leq 
            \sigma_{max,1} =
            \frac{\beta}{L \sqrt{2 n \log(2 L n)}}
        \end{equation}
        and allows us to apply Lemma~\ref{lemma:2_of_neyshabur_et_al}: w.p. $\geq 1/2$ over $\xi$,
        \begin{equation}
            \max_{x \in \XX_B} \left|f_{\theta+\xi}(x) - f_\theta(x)\right| \leq
            e B \beta^{L-1} \sum_{l=1}^L \|U_l\|_2 \leq
            e B \beta^{L-1} L \sigma \sqrt{2n \log(2 L n)}.
        \end{equation}
        In order to apply Lemma~\ref{lemma:1_of_neyshabur_et_al} we need to ensure that this equation is bounded by $\gamma/2$.
        This gives 
        \begin{equation}
            \sigma \leq
            \sigma_{max,2} =
            \frac{\gamma}{2 e B \beta^{L-1} L \sqrt{2n \log(2 L n)}}.
        \end{equation}
        Taking $\sigma = \sigma_{max} = \min(\sigma_{max,1}, \sigma_{max,2})$ hence ensures the condition of Lemma~\ref{lemma:1_of_neyshabur_et_al}.
        The problem is that $\sigma$ now depends on $\beta$ and hence on $\hat\theta_m$; this means that the prior $P = \NN(0, \sigma^2 I_{\dim\theta})$ depends on $\hat\theta_m$.
        For this reason, we have to apply a union bound argument for choosing $\sigma$.

        Let $\tilde\BB$ be a discrete subset of $\RR_+$.
        Hence $\forall \tilde\beta \in \tilde\BB$ $\forall \delta \in (0,1)$ w.p. $\geq 1 - \delta$ over $S_m$ $\forall \theta$ such that $\sigma_{max}(\beta) \geq \sigma_{max}(\tilde\beta)$
        \begin{equation}
            R(f_\theta) \leq \hat R_{m,\gamma}(f_\theta) +
            \sqrt{\frac{1}{2m-1} \left(\log \frac{16m}{\delta} + \frac{\|\theta\|_2^2}{\sigma_{max}^2(\tilde\beta)} \right)}.
        \end{equation}
        A union bound gives $\forall \delta \in (0,1)$ w.p. $\geq 1 - \delta$ over $S_m$ $\forall \theta$ $\forall \tilde\beta \in \tilde\BB$ such that $\sigma_{max}(\beta) \geq \sigma_{max}(\tilde\beta)$
        \begin{equation}
            R(f_\theta) \leq \hat R_{m,\gamma}(f_\theta) +
            \sqrt{\frac{1}{2m-1} \left(\log \frac{16m}{\delta} + \frac{\|\theta\|_2^2}{\sigma_{max}^2(\tilde\beta)} + \log|\tilde\BB| \right)}.
        \end{equation}

        We need $\tilde\BB$ to be finite in order to have a finite bound.
        First note that for $\beta^L < \gamma / (2B)$ we have $\forall x \in \XX_B$ $|f_\theta(x)| \leq \beta^L B \leq \gamma/2$ which implies $\hat R_{m,\gamma}(f_\theta) = 1$.
        In this case the bound is trivially true.

        Second, for $\beta^L > \gamma \sqrt{m} / (2B)$ the second term of the final bound (see Theorem~\ref{thm:neyshabur_et_al}) is greater than $1$ and the bound again becomes trivially true.
        Hence it suffices to take any finite $\tilde\BB$ with $\min(\tilde\BB) = \beta_{min} = (\gamma / (2B))^{1/L}$ and $\max(\tilde\BB) = \beta_{max} = (\gamma \sqrt{m} / (2B))^{1/L}$.
        Note that for $\beta \in [\beta_{min}, \beta_{max}]$ $\sigma_{max} = \sigma_{max,2}$; indeed,
        \begin{equation}
            \frac{\sigma_{max,1}}{\sigma_{max,2}} =
            2e \gamma^{-1} B \beta^L \geq
            e >
            1.
        \end{equation}
        Hence $\sigma_{max}(\beta) \geq \sigma_{max}(\tilde\beta)$ is equivalent to $\beta \leq \tilde\beta$.

        We shall take $\tilde\BB$ such that the following holds:
        \begin{equation}
            \forall \beta \in [\beta_{min}, \beta_{max}]
            \quad
            \exists \tilde\beta \in \tilde\BB:
            \quad
            e^{-1} \tilde\beta^{L-1} \leq
            \beta^{L-1} \leq
            e \tilde\beta^{L-1}.
            \label{eq:BB_condition}
        \end{equation}
        In this case, obviously, $\beta \leq \tilde\beta$ and
        \begin{equation}
            \sigma^2_{max}(\tilde\beta) =
            \frac{\gamma^2}{8 e^2 B^2 \tilde\beta^{2L-2} L^2 n \log(2 L n)} \geq
            \frac{\gamma^2}{8 e^4 B^2 \beta^{2L-2} L^2 n \log(2 L n)}.
        \end{equation}

        We shall prove that the following $\tilde\BB$ conforms condition~(\ref{eq:BB_condition}):
        \begin{equation}
            \tilde\BB =
            \left\{\beta_{min} \left(1 + \frac{2k}{L}\right)\right\}_{k=0}^K,
            \qquad
            K = 
            \max\left\{k: \; \beta_{min} \left(1 + \frac{2k}{L}\right) \leq \beta_{max}\right\}.
        \end{equation}
        Hence $2K = \lfloor L (\beta_{max} / \beta_{min} - 1) \rfloor = \lfloor L (m^{1/2L} - 1) \rfloor$.
        This gives:
        \begin{equation}
            \log|\tilde\BB| =
            \log (K+1) \leq
            \log (L m^{1/2L} / 2) =
            \log (L/2) + \frac{1}{2L} \log m.
        \end{equation}
        Indeed, for any $\beta \in [\beta_{min}, \beta_{max}]$ $\exists \tilde\beta \in \tilde\BB:$ $|\beta - \tilde\beta| \leq \beta_{min} / L \leq \tilde\beta / L$.
        Hence
        \begin{equation}
            e \tilde\beta^{L-1} \geq
            (\tilde\beta + \tilde\beta / L)^{L-1} \geq
            (\tilde\beta + |\beta - \tilde\beta|)^{L-1} \geq
            \beta^{L-1},
        \end{equation}
        \begin{equation}
            e^{-1} \tilde\beta^{L-1} \leq
            (\tilde\beta - \tilde\beta / L)^{L-1} \leq
            (\tilde\beta - |\beta - \tilde\beta|)^{L-1} \leq
            \beta^{L-1},
        \end{equation}
        which proves condition~(\ref{eq:BB_condition}).

        Let us first write the expression before the $(2m-1)^{-1}$ factor:
        \begin{equation}
            \log \frac{16m}{\delta} + \frac{\|\theta\|_2^2}{\sigma_{max}^2(\tilde\beta)} + \log|\tilde\BB| \leq
            \log \frac{8Lm}{\delta} + \frac{1}{2L} \log m + 8 \gamma^{-2} e^4 B^2 \beta^{2L} L^2 n \log(2 L n) \sum_{l=1}^L \frac{\|W_l\|_F^2}{\beta^2}.
        \end{equation}
        This gives the final bound:
        \begin{equation}
            R(f_\theta) \leq \hat R_{m,\gamma}(f_\theta) +
            \sqrt{\frac{1}{2m-1} \left(\log \frac{8Lm}{\delta} + \frac{1}{2L} \log m + 8 e^4 \left(\frac{B \mathcal{R}(\theta)}{\gamma}\right)^2 L^2 n \log(2 L n) \right)},
        \end{equation}
        where we have introduced a spectral complexity:
        \begin{equation}
            \mathcal{R}(\theta) =
            \beta^{L} \sum_{l=1}^L \frac{\|W_l\|_F}{\beta} =
            \left(\prod_{l=1}^L \|W_l\|_2\right) \sqrt{\sum_{l=1}^L \frac{\|W_l\|_F^2}{\|W_l\|_2^2}}.
        \end{equation}
    \end{proof}

    \begin{proof}[Proof of Lemma~\ref{lemma:1_of_neyshabur_et_al}]
        Let $\theta$ and $\xi$ conform Condition~\ref{eq:lemma_1_condition} and let $\theta' = \theta + \xi$.
        Define:
        \begin{equation}
            A_\theta =
            \{\theta': \; \max_{x\in\XX} |f_{\theta'}(x) - f_\theta(x)| < \gamma/2\}.
        \end{equation}
        Following Condition~\ref{eq:lemma_1_condition}, we get $\PP(A_\theta) \geq 1/2$.

        Since $\xi$ has density, $\theta'$ has density as well; denote it by $q'$.
        Define:
        \begin{equation}
            \tilde q(\tilde\theta) =
            \frac{1}{Z} q'(\tilde\theta) [\tilde\theta \in A_\theta],
            \quad
            \text{where $Z = \PP(A_\theta)$.}
        \end{equation}
        Note that $\max_{x\in\XX} |f_{\tilde\theta}(x) - f_\theta(x)| < \gamma/2$ a.s. wrt $\tilde\theta$ for $\tilde\theta \sim \tilde q(\tilde\theta)$.
        Therefore:
        \begin{equation}
            R(f_\theta) \leq
            R_{\gamma/2}(f_{\tilde\theta})
            \quad
            \text{and}
            \quad
            \hat R_{m,\gamma/2}(f_{\tilde\theta}) \leq
            \hat R_{m,\gamma}(f_\theta)
            \quad
            \text{a.s. wrt $\tilde\theta$.}
        \end{equation}
        Hence
        \begin{multline}
            R(f_\theta) \leq
            \EE_{\tilde\theta} R_{\gamma/2}(f_{\tilde\theta}) \leq
            \text{(w.p. $\geq 1-\delta$ over $S_m$)}
            \\\leq
            \EE_{\tilde\theta} \hat R_{m,\gamma/2}(f_{\tilde\theta}) + \sqrt{\frac{1}{2m-1} \left(\log \frac{4m}{\delta} + \kld{\tilde q}{p} \right)} \leq
            \hat R_{m,\gamma}(f_{\theta}) + \sqrt{\frac{1}{2m-1} \left(\log \frac{4m}{\delta} + \kld{\tilde q}{p} \right)}.
        \end{multline}

        The only thing that remains is estimating the KL-term.
        Define:
        \begin{equation}
            \tilde q^c(\tilde\theta) =
            \frac{1}{1-Z} q'(\tilde\theta) [\tilde\theta \notin A_\theta].
        \end{equation}
        We then get:
        \begin{multline}
            \kld{q'}{p} =
            \kld{\tilde q Z + \tilde q^c (1-Z)}{p} =
            \EE_{\theta' \sim q'} \left(\log(\tilde q(\theta') Z + \tilde q^c(\theta') (1-Z)) - \log p(\theta')\right) =
            \\=
            \EE_{b \sim B(Z)} \EE_{\theta' \sim q' | b} \left(\log(q'(\theta' | 1) Z + q'(\theta' | 0) (1-Z)) - (Z + (1-Z))\log p(\theta')\right) =
            \\=
            Z (\log Z + \kld{q'|1}{p}) + (1-Z) (\log (1-Z) + \kld{q'|0}{p}) =
            \\=
            Z \kld{\tilde q}{p} + (1-Z) \kld{\tilde q^c}{p} - H(B(Z)).
        \end{multline}
        This implies the following:
        \begin{multline}
            \kld{\tilde q}{p} =
            \frac{1}{Z} \left(\kld{q'}{p} + H(B(Z)) - (1-Z) \kld{\tilde q^c}{p}\right) \leq
            \\\leq
            \frac{1}{P(A_\theta)} \left(\kld{q'}{p} + \log 2\right) \leq
            2 \left(\kld{q'}{p} + \log 2\right).
        \end{multline}
        Hence w.p. $\geq 1-\delta$ over $S_m$ we have:
        \begin{equation}
            R(f_\theta) \leq
            \hat R_{m,\gamma}(f_{\theta}) + \sqrt{\frac{1}{2m-1} \left(\log \frac{4m}{\delta} + \kld{\tilde q}{p} \right)} \leq
            \hat R_{m,\gamma}(f_{\theta}) + \sqrt{\frac{1}{2m-1} \left(\log \frac{16m}{\delta} + 2 \kld{q'}{p} \right)}.   
        \end{equation}
    \end{proof}

    \begin{proof}[Proof of Lemma~\ref{lemma:2_of_neyshabur_et_al}]
        Recall the forward dynamics:
        \begin{equation}
            h_2(x;\theta) = W_1 x \in \RR^{n_2},
            \quad
            x_l(x;\theta) = \phi(h_l(x;\theta)) \in \RR^{n_l},
            \quad
            h_{l+1}(x;\theta) = W_l x_l(x;\theta) \in \RR^{n_{l+1}} \; \forall l \in \{2,\ldots,L\}.
        \end{equation}
        Assume that $\xx$, $\theta$, and $\xi$ are fixed.
        Define:
        \begin{equation}
            \Delta_l =
            \|h_{l+1}(x; \theta+\xi) - h_{l+1}(x; \theta)\|_2
            \quad
            \forall l \in [L].
        \end{equation}
        We are going to prove the following by induction:
        \begin{equation}
            \Delta_l \leq
            \left(1 + \frac{1}{L}\right)^l \|x\|_2 \left(\prod_{l'=1}^l \|W_{l'}\|_2\right) \sum_{l'=1}^l \frac{\|U_{l'}\|_2}{\|W_{l'}\|_2}.
        \end{equation}
        The induction base is given as:
        \begin{equation}
            \Delta_1 =
            \|h_2(x; \theta+\xi) - h_2(x; \theta)\|_2 =
            \|U_1 x\|_2 \leq
            \|U_1\|_2 \|x\|_2,
        \end{equation}
        and we prove the induction step below:
        \begin{multline}
            \Delta_l =
            \|h_{l+1}(x; \theta+\xi) - h_{l+1}(x; \theta)\|_2 =
            \| (W_l + U_l) x_l(x; \theta+\xi) - W_l x_l(x; \theta) \|_2 =
            \\=
            \| (W_l + U_l) (x_l(x; \theta+\xi) - x_l(x; \theta)) + U_l x_l(x; \theta) \|_2 \leq
            \\\leq
            \| W_l + U_l \|_2 \|x_l(x; \theta+\xi) - x_l(x; \theta)\|_2 + \|U_l\|_2 \| x_l(x; \theta) \|_2 \leq
            \\\leq
            (\|W_l\|_2 + \|U_l\|_2) \|h_l(x; \theta+\xi) - h_l(x; \theta)\|_2 + \|U_l\|_2 \| h_l(x; \theta) \|_2 \leq
            \\\leq
            \|W_l\|_2 \left(1 + \frac{1}{L}\right) \Delta_{l-1} + \|U_l\|_2 \|x\|_2 \prod_{l'=1}^{l-1} \|W_{l'}\|_2 \leq
            \\\leq
            \left(1 + \frac{1}{L}\right)^l \|x\|_2 \left(\prod_{l'=1}^l \|W_{l'}\|_2\right) \sum_{l'=1}^{l-1} \frac{\|U_{l'}\|_2}{\|W_{l'}\|_2} + \frac{\|U_l\|_2}{\|W_l\|_2} \|x\|_2 \prod_{l'=1}^l \|W_{l'}\|_2 \leq
            \\\leq
            \left(1 + \frac{1}{L}\right)^l \|x\|_2 \left(\prod_{l'=1}^l \|W_{l'}\|_2\right) \sum_{l'=1}^l \frac{\|U_{l'}\|_2}{\|W_{l'}\|_2}.
        \end{multline}
        A simple estimate then gives the required statement:
        \begin{multline}
            \|f_{\theta+\xi}(x) - f_\theta(x)\|_2 =
            \|h_{L+1}(x; \theta+\xi) - h_{L+1}(x; \theta)\|_2 =
            \\=
            \Delta_L \leq
            \left(1 + \frac{1}{L}\right)^L \|x\|_2 \left(\prod_{l=1}^L \|W_{l}\|_2\right) \sum_{l=1}^L \frac{\|U_{l}\|_2}{\|W_{l}\|_2} \leq
            e B \left(\prod_{l=1}^L \|W_{l}\|_2\right) \sum_{l=1}^L \frac{\|U_{l}\|_2}{\|W_{l}\|_2}.
        \end{multline}
    \end{proof}

    \chapter{Neural tangent kernel}

    \section{Gradient descent training as a kernel method}

    Consider a parametric model with scalar output $f(x;\theta) \in \RR$ and let $\theta \in \RR^d$.
    We aim to minimize a loss $\EE_{x,y} \ell(y,f(x;\theta))$ via a gradient descent:
    \begin{equation}
        \dot\theta_t =
        -\eta \EE_{x,y} \left.\frac{\partial\ell(y,z)}{\partial z}\right|_{z=f(x; \theta_t)} \nabla_\theta f(x; \theta_t).
        \label{eq:theta_evolution}
    \end{equation}
    If we define a feature map $\Phi_t(x) = \nabla_\theta f(x; \theta_t)$, then we can express the model as:
    \begin{equation}
        f(x; \theta) =
        f(x; \theta_t) + \Phi^T_t(x) (\theta - \theta_t) + O(\|\theta - \theta_t\|_2^2).        
    \end{equation}
    It is a locally linear model in the vicinity of $\theta_t$ given a feature map $\Phi_t$.

    We now multiply both parts of the equation~(\ref{eq:theta_evolution}) by $\nabla^T_\theta f(x'; \theta_t)$:
    \begin{equation}
        \dot f_t(x') =
        -\eta \EE_{x,y} \left.\frac{\partial\ell(y,z)}{\partial z}\right|_{z=f_t(x)} \hat\Theta_t(x',x),
        \label{eq:f_evolution}
    \end{equation}
    where $\hat\Theta_t(x',x) = \nabla^T_\theta f(x'; \theta_t) \nabla_\theta f(x; \theta_t)$ and $f_t(x') = f(x'; \theta_t)$.

    Here $\hat\Theta_t$ is a kernel and $\Phi_t(x) = \nabla_\theta f(x; \theta_t)$ is a corresponding feature map.
    We call $\hat\Theta_t$ an \emph{empirical tangent kernel} at time-step $t$.
    It depends on the initialization, hence it is random.
    Given a train dataset $(\vec x, \vec y)$ of size $m$, the evolution of the responses on this dataset writes as follows:
    \begin{equation}
        \dot f_t(\vec x) =
        -\frac{\eta}{m} \hat\Theta_t(\vec x, \vec x) \left.\frac{\partial\ell(\vec y,\vec z)}{\partial \vec z}\right|_{\vec z=f_t(\vec x)}.
    \end{equation}
    We see that the gramian of the kernel maps loss gradients wrt model outputs to output increments.
    Note that while dynamics~(\ref{eq:theta_evolution}) is complete, (\ref{eq:f_evolution}) is not, since $K_t$ cannot be determined solely in terms of $f_t$.

    Nevertheless, if we consider linearized dynamics, $K_t$ becomes independent of $t$ and can computed once at the initialization, thus making dynamics~(\ref{eq:f_evolution}) complete.
    Let us define a linearized model:
    \begin{equation}
        f_{lin}(x; \theta) =
        f(x; \theta_0) + \nabla^T_\theta f(x; \theta_0) (\theta - \theta_0).
        \label{eq:f_lin}
    \end{equation}
    This model then evolves similarly to $f$ (eq.~(\ref{eq:f_evolution})), but with a kernel fixed at initialization:
    \begin{equation}
        \dot f_{lin,t}(x') =
        -\eta \EE_{x,y} \left.\frac{\partial\ell(y,z)}{\partial z}\right|_{z=f_{lin,t}(x)} \hat\Theta_0(x',x).
        \label{eq:f_lin_evolution}
    \end{equation}
    The gradient descent dynamics becomes:
    \begin{equation}
        \dot\theta_t =
        -\eta \EE_{x,y} \left.\frac{\partial\ell(y,z)}{\partial z}\right|_{z=f_{lin,t}(x)} \nabla_\theta f(x; \theta_0).
        \label{eq:theta_lin_evolution}
    \end{equation}

    \subsection{Exact solution for a square loss}
    \label{sec:exact_square}

    These equations are analytically solvable if we consider a square loss: $\ell(y,z) = \frac{1}{2} (y - z)^2$, see \cite{lee2019wide}.
    Let $(\vec x, \vec y)$, where $\vec x = \{x_i\}_{i=1}^m$ and $\vec y = \{y_i\}_{i=1}^m$, is a train set.
    Let $f(\vec x) = \{f(x_i)\}_{i=1}^m \in \RR^m$ be a vector of model responses on the train data.
    Finally, let $\hat\Theta_t(\vec x, \vec x) \in \RR^{m \times m}$ be a Gramian of the kernel $\hat\Theta_t$: $\hat\Theta_t(\vec x, \vec x)_{ij} = \hat\Theta_t(x_i,x_j)$.
    Eq.~(\ref{eq:f_lin_evolution}) evaluated on train set becomes:
    \begin{equation}
        \dot f_{lin,t}(\vec x) =
        \eta \frac{1}{m} \hat\Theta_0(\vec x, \vec x) (\vec y - f_{lin,t}(\vec x)).
    \end{equation}
    Its solution writes as follows:
    \begin{equation}
        f_{lin,t}(\vec x) =
        \vec y + e^{-\eta \hat\Theta_0(\vec x, \vec x) t / m} (f_0(\vec x) - \vec y).
    \end{equation}
    Given this, the weight dynamics~(\ref{eq:theta_lin_evolution}) becomes:
    \begin{equation}
        \dot\theta_t =
        -\eta \frac{1}{m} \nabla_\theta f(\vec x; \theta_0) e^{-\eta \hat\Theta_0(\vec x, \vec x) t / m} (f_0(\vec x) - \vec y),
    \end{equation}
    where we have assumed that $\nabla_\theta f(\vec x; \theta_0) \in \RR^{d \times m}$.
    Solving it gives:
    \begin{equation}
        \theta_t =
        \theta_0 - \nabla_\theta f(\vec x; \theta_0) \hat\Theta^{-1}_0(\vec x, \vec x) (I - e^{-\eta \hat\Theta_0(\vec x, \vec x) t / m}) (f_0(\vec x) - \vec y).
    \end{equation}
    Substituting the solution back to~(\ref{eq:f_lin}) gives a model prediction on an arbitrary input $x$:
    \begin{equation}
        f_{lin,t}(x) =
        f_0(x) - \hat\Theta_0(x, \vec x) \hat\Theta^{-1}_0(\vec x, \vec x) (I - e^{-\eta \hat\Theta_0(\vec x, \vec x) t / m}) (f_0(\vec x) - \vec y),
        \label{eq:exact_square_loss}
    \end{equation}
    where we have defined a row-vector $\hat\Theta_0(x, \vec x)$ with components $\hat\Theta_{0,i}(x, \vec x) = \nabla^T_\theta f(x; \theta_0) \nabla_\theta f(x_i; \theta_0)$.

    \subsection{Convergence to a gaussian process}
    \label{sec:gp_convergence}

    Consider a network with $L$ hidden layers and no biases:
    \begin{equation}
        f(x) = W_L \phi(W_{L-1} \ldots \phi(W_0 x)),
    \end{equation}
    where $W_l \in \RR^{n_{l+1} \times n_l}$ and a non-linearity $\phi$ is applied element-wise.
    Note that $x \in \RR^{n_0}$; we denote with $k = n_{L+1}$ the dimensionality of the output: $f: \; \RR^{n_0} \to \RR^k$.
    We shall refer $n_l$ as the width of the $l$-th hidden layer.

    Let us assume $x$ is fixed.
    Define:
    \begin{equation}
    h_1 = W_0 x \in \RR^{n_1},
    \quad
    x_l = \phi(h_l) \in \RR^{n_l},
    \quad
    h_{l+1} = W_l x_l \in \RR^{n_{l+1}} \; \forall l \in [L].
    \end{equation}
    Hence given $x$ $f(x) = h_{L+1}$.
    Define also:
    \begin{equation}
        q_l = 
        \frac{1}{n_l} \EE h_l^T h_l.
    \end{equation}

    Let us assume that the weights are initialized with zero-mean gaussians, so that the forward dynamics is normalized:
    \begin{equation}
        W_l^{ij} \sim
        \NN\left(0, \frac{\sigma_w^2}{n_l}\right).
    \end{equation}

    Obiously, all components of $h_l$ are distributed identically.
    Their means are zeros, let us compute the variances:
    \begin{equation}
        q_{l+1} =
        \frac{1}{n_{l+1}} \EE x_l^T W_l^T W_l x_l =
        \frac{\sigma_w^2}{n_l} \EE x_l^T x_l =
        \frac{\sigma_w^2}{n_l} \EE \phi(h_l)^T \phi(h_l)
        \quad
        \forall l \in [L],
        \qquad
        q_1 =
        \frac{1}{n_1} \EE x^T W_0^T W_0 x =
        \frac{\sigma_w^2}{n_0} \|x\|_2^2.
    \end{equation}
    We are going to prove by induction that $\forall l \in [L+1]$ $\forall i \in [n_l]$ $h_l^i$ converges weakly to $\NN(0, q_l)$ as $n_{1:l-1} \to \infty$ sequentially.
    Since components of $W_0$ are gaussian, $h_1^i \sim \NN(0, q_1)$ $\forall i \in [n_0]$ --- this gives the induction base.
    If all $h_l^i$ converge weakly to $\NN(0, q_l)$ as $n_{1:l-1} \to \infty$ sequentially then $\lim_{n_{1:l}\to\infty} q_{l+1} = \sigma_w^2 \EE_{z \sim \NN(0, q_l)} (\phi(z))^2$.
    Hence by the virture of CLT, $h_{l+1}^i$ converges in distribution to $\NN(0, q_{l+1})$ as $n_{1:l} \to \infty$ sequentially --- this gives the induction step.

    Consider two inputs, $x^1$ and $x^2$, together with their hidden representations $h_l^1$ and $h_l^2$.
    Let us prove that $\forall l \in [L+1]$ $\forall i \in [n_l]$ $(h_l^{1,i}, h_l^{2,i})^T$ converges in distribution to $\NN(0, \Sigma_l)$ as $n_{1:l-1} \to \infty$ sequentially, where the covariance matrix is defined as follows:
    \begin{equation}
        \Sigma_l =
        \begin{pmatrix}
            q_l^{11} & q_l^{12} \\
            q_l^{12} & q_l^{22}
        \end{pmatrix};
        \qquad
        q_l^{ab} =
        \frac{1}{n_l} \EE h_l^{a,T} h_l^b,
        \quad
        a, b \in \{1,2\}.
    \end{equation}
    We have already derived the dynamics for the diagonal terms in the subsequent limits of infinite width:
    \begin{equation}
        \lim_{n_{1:l}\to\infty} q_{l+1}^{aa} =
        \sigma_w^2 \EE_{z \sim \NN(0,q_l^{aa})} (\phi(z))^2,
        \quad
        q_1^{aa} =
        \sigma_w^2 \frac{\|x^a\|_2^2}{n_0},
        \quad
        a \in \{1,2\}.
    \end{equation}
    Consider the diagonal term:
    \begin{equation}
        q_{l+1}^{12} =
        \frac{1}{n_{l+1}} \EE_{h_l^1,h_l^2} \EE_{W_l} \phi(h_l^1)^T W_l^T W_l \phi(h_l^2) =
        \frac{\sigma_w^2}{n_l} \EE_{h_l^1,h_l^2} \phi(h_l^1)^T \phi(h_l^2).
    \end{equation}
    By induction hypothesis, as $n_{1:l-1} \to \infty$ we have a weak limit:
    \begin{equation}
        \begin{pmatrix}
            h_l^{1,i} \\
            h_l^{2,i}
        \end{pmatrix} \to
        \NN(0, \Sigma_l).
    \end{equation}
    Hence
    \begin{equation}
        \lim_{n_{1:l} \to \infty} q_{l+1}^{12} =
        \sigma_w^2 \EE_{(u^1,u^2)^T \sim \NN(0,\Sigma_l)} \phi(u^1) \phi(u^2).
    \end{equation}
    Note that
    \begin{equation}
        \begin{pmatrix}
            h_{l+1}^{1,i} \\
            h_{l+1}^{2,i}
        \end{pmatrix} =
        \sum_{j=1}^{n_l}
        W_l^{ij}
        \begin{pmatrix}
            x_{l}^{1,j} \\
            x_{l}^{2,j}
        \end{pmatrix}.
    \end{equation}
    Here we have a sum of $n_l$ i.i.d. random vectors with zero mean, and the covariance matrix of the sum is $\Sigma_{l+1}$.
    Hence by the multivariate CLT, $(h_{l+1}^{1,i}, h_{l+1}^{2,i})^T$ converges weakly to $\NN(0, \Sigma_{l+1})$ as $n_{1:l} \to \infty$ sequentially.

    Similarly, for any $k \geq 1$
    \begin{equation}
        \begin{pmatrix}
            h_{l+1}^{1,i} \\
            \ldots\\
            h_{l+1}^{k,i}
        \end{pmatrix} =
        \sum_{j=1}^{n_l}
        W_l^{ij}
        \begin{pmatrix}
            x_{l}^{1,j} \\
            \ldots \\
            x_{l}^{k,j}
        \end{pmatrix}.
    \end{equation}
    Again, these vectors converge to a gaussian by the multivariate CLT.
    Hence $\forall l \in [L+1]$ $h_l^i(\cdot)$ converges weakly to a gaussian process as $n_{1:l-1}\to\infty$ sequentially.
    Note that a gaussian process is completely defined by its mean and covariance functions:
    \begin{equation}
        \Sigma_l(x,x') =
        \begin{pmatrix}
            q_l(x,x) & q_l(x,x') \\
            q_l(x',x) & q_l(x',x')
        \end{pmatrix}
        \quad
        \forall l \in [L+1];
    \end{equation}
    \begin{equation}
        q_{l+1}(x, x') =
        \sigma_w^2 \EE_{(u,v)^T \sim \NN(0,\Sigma_l(x, x'))} \phi(u) \phi(v)
        \quad
        \forall l \in [L],
        \qquad
        q_1(x, x') =
        \frac{\sigma_w^2}{n_0} x^T x'.
    \end{equation}
    Hence the model at initialization converges to a gaussian process with zero mean and covariance $\Sigma_{L+1}(\cdot,\cdot)$.
    This GP is referred as NNGP, and $q_{L+1}$ --- as NNGP kernel.

    \subsection{The kernel diverges at initialization}

    For a fixed $x$, let us define the following quantity:
    \begin{equation}
        B_l^i = \frac{\partial f^i}{\partial h_l} \in \RR^{n_l}
        \quad
        \forall l \in [L+1].
    \end{equation}
    We have then:
    \begin{equation}
        B_l^i = D_l W_l^T B_{l+1}^i
        \quad
        \forall l \in [L],
        \quad
        B_{L+1}^{ij} = \delta_{ij},
    \end{equation}
    where $D_l = \diag(\phi'(h_l))$.
    This gives:
    \begin{equation}
        \nabla_{W_l} f^i(x) =
        B_{l+1}^i x_l^T \in \RR^{n_{l+1} \times n_l}.
    \end{equation}

    Define the scaled covariance for $B_l$ components for two inputs:
    \begin{multline}
        \beta_l^{ij}(x, x') =
        \EE B_l^{i,T} B_l^{\prime,j} =
        \EE B_{l+1}^{i,T} W_l D_l D_l' W_l^T B_{l+1}^{\prime,j} =
        \frac{\sigma_w^2}{n_l} \EE \tr(D_l D_l') (B_{l+1}^{i,T} B_{l+1}^{\prime,j}) =\\=
        \sigma_w^2 \beta_{l+1}^{ij}(x, x') \EE_{(u,v)^T \sim \NN(0,\Sigma_l(x, x'))} \phi'(u) \phi'(v)
        \quad
        \forall l \in [L-1],
    \end{multline}
    \begin{equation}
        \beta_L^{ij} =
        \EE B_L^{i,T} B_L^{\prime,j} =
        \EE B_{L+1}^{i,T} W_L D_L D_L' W_L^T B_{L+1}^{\prime,j} =
        \frac{\sigma_w^2}{n_L} \EE \tr(D_L D_L') \delta_{ij} =
        \sigma_w^2 \EE_{(u,v)^T \sim \NN(0,\Sigma_L(x, x'))} \phi'(u) \phi'(v) \delta_{ij}.
    \end{equation}
    Note that $\beta_l^{ij} = \beta_l \delta_{ij}$.
    Similarly to $q_l$, define the following:
    \begin{equation}
        \chi_l(x, x') =
        \sigma_w^2 \EE_{(u,v)^T \sim \NN(0,\Sigma_l(x, x'))} \phi'(u) \phi'(v).
    \end{equation}
    This allows us to write:
    \begin{equation}
        \beta_l(x, x') =
        \prod_{l'=l}^{L} \chi_{l'}(x, x')
        \quad
        \forall l \in [L].
    \end{equation}

    In the case of non-scalar output ($k > 1$), the tangent kernel is a $k \times k$ matrix with components defined as:
    \begin{equation}
        \hat\Theta^{ij}(x,x') =
        \nabla^T_\theta f^i(x) \nabla_\theta f^j(x').
    \end{equation}
    For the sake of convenience, we introduce layer-wise tangent kernels:
    \begin{equation}
        \hat\Theta_l^{ij}(x,x') =
        \tr(\nabla^T_{W_l} f^i(x) \nabla_{W_l} f^j(x')).
    \end{equation}
    In this case $\hat\Theta(x,x') = \sum_{l=0}^L \hat\Theta_l(x,x')$.

    We denote $B_l$ and $h_l$ evaluated at $x'$ by $B'_l$ and $h'_l$, respectively.
    This allows us to write:
    \begin{equation}
        \hat\Theta_l^{ij}(x,x') =
        \tr\left(\phi(h_l) B_{l+1}^{i,T} B_{l+1}^{\prime,j} \phi(h'_l)^T\right) =
        \left(\phi(h'_l)^T \phi(h_l)\right) \left(B_{l+1}^{i,T} B_{l+1}^{\prime,j}\right)
        \quad \forall l \in [L].
    \end{equation}
    If we assume that the two scalar products are independent then the expected kernel is a product of expectations:
    \begin{equation}
        \EE \hat\Theta_l^{ij}(x,x') =
        \EE \left(\phi(h'_l)^T \phi(h_l)\right) \EE \left(B_{l+1}^{i,T} B_{l+1}^{\prime,j}\right) =
        \frac{n_l q_{l+1}(x,x')}{\sigma_w^2} \beta_{l+1}(x,x') \delta_{ij}.
        \quad \forall l \in [L].
    \end{equation}
    Hence (a) each kernel is diagonal, (b) $l$-th kernel expectation diverges as $n_l \to \infty$ $\forall l \in [L]$.

    \subsection{NTK parameterization}

    It is possible to leverage the kernel divergence by altering the network parameterization:
    \begin{equation}
    h_1 = \frac{\sigma_w}{\sqrt{n_0}} W_0 x \in \RR^{n_1},
    \qquad
    x_l = \phi(h_l) \in \RR^{n_l},
    \qquad
    h_{l+1} = \frac{\sigma_w}{\sqrt{n_l}} W_l x_l \in \RR^{n_{l+1}} \quad \forall l \in [L].
    \end{equation}
    In this case, the weights are standard gaussians:
    \begin{equation}
        W_l^{ij} \sim
        \NN\left(0, 1\right).
    \end{equation}

    \begin{equation}
        B_l^i = \frac{\partial f^i}{\partial h_l} \in \RR^{n_l}
        \quad
        \forall l \in [L+1].
    \end{equation}
    We have then:
    \begin{equation}
        B_l^i = \frac{\sigma_w}{\sqrt{n_l}} D_l W_l^T B_{l+1}^i
        \quad
        \forall l \in [L],
        \quad
        B_{L+1}^{ij} = \delta_{ij},
    \end{equation}

    Both forward and backward dynamics at initialization remains unchanged.
    What changes are the gradients wrt weights:
    \begin{equation}
        \nabla_{W_l} f^i(x) =
        \frac{\sigma_w}{\sqrt{n_l}} B_{l+1}^i x_l^T.
    \end{equation}
    This results in a change of the tangent kernel scaling:
    \begin{equation}
        \EE \hat\Theta_l^{ij}(x,x') =
        \frac{\sigma_w^2}{n_l} \EE \left(\phi(h'_l)^T \phi(h_l)\right) \EE \left(B_{l+1}^{i,T} B_{l+1}^{\prime,j}\right) =
        q_{l+1}(x,x') \beta_{l+1}(x,x') \delta_{ij}
        \quad \forall l \in [L].
    \end{equation}
    Now the kernel expectation neither diverges nor vanishes as $n \to \infty$.
    Since the expectation is finite, the kernel itself converges to it as $n \to \infty$.
    Indeed, consider the $l$-th kernel:
    \begin{equation}
        \hat\Theta_l^{ij}(x,x') =
        \frac{\sigma_w^2}{n_l} \left(\phi(h'_l)^T \phi(h_l)\right) \left(B_{l+1}^{i,T} B_{l+1}^{\prime,j}\right).
    \end{equation}
    The first multiplier converges to $q_{l+1}(x,x')$ due to the Law of Large Numbers.
    Similar holds for the second one: it converges to $\beta_{l+1}(x,x') \delta_{ij}$ by the LLN.
    Together these two give:
    \begin{equation}
        \plim_{n_l \to \infty} \ldots \plim_{n_1 \to \infty} \hat\Theta_l^{ij}(x,x') =
        \EE \hat\Theta_l^{ij}(x,x') =
        q_{l+1}(x,x') \beta_{l+1}(x,x') \delta_{ij}
        \quad \forall l \in [L].
    \end{equation}
    And for the whole kernel, we have:
    \begin{equation}
        \plim_{n_L \to \infty} \ldots \plim_{n_1 \to \infty} \hat\Theta^{ij}(x,x') =
        \EE \hat\Theta^{ij}(x,x') =
        \sum_{l=1}^{L+1} q_l(x,x') \beta_l(x,x') \delta_{ij} =
        \sum_{l=1}^{L+1} \left(q_l(x,x') \prod_{l'=l}^{L} \chi_{l'}(x, x')\right) \delta_{ij}.
    \end{equation}
    See \cite{arora2019exact} for the above expression for the expected kernel, and \cite{jacot2018neural} for the formal proof of convergence in subsequent limits.
    See also \cite{arora2019exact} for a convergence proof in stronger terms.

    \subsection{GD training and posterior inference in gaussian processes}

    Denote the limit kernel at initialization by $\Theta_0$:
    \begin{equation}
        \Theta_0(x,x') =
        \sum_{l=1}^{L+1} \left(q_l(x,x') \prod_{l'=l}^{L} \chi_{l'}(x, x')\right) I_{k \times k}.
    \end{equation}
    Unlike the empirical kernel, the limit one is deterministic.
    Similarly to Section~\ref{sec:exact_square}, we assume that $\vec x$ is a training set of size $n$, and $k=1$.
    Then let $\Theta_0(\vec x, \vec x) \in \RR^{n \times n}$ be a Gramian for the limit kernel.

    Given (a) the kernel has a deterministic limit, and (b) the model at initialization converges to a limit model, the model trained to minimize square loss converges to the following limit model at any time $t$:
    \begin{equation}
        \lim f_{lin,t}(x) =
        \lim f_0(x) - \Theta_0(x, \vec x) \Theta^{-1}_0(\vec x, \vec x) (I - e^{-\eta \Theta_0(\vec x, \vec x) t / n}) (\lim f_0(\vec x) - \vec y).
    \end{equation}
    Looking at this expression we notice that since the limit model at initialization is a gaussian process (see Section~\ref{sec:gp_convergence}), the limit model is a gaussian process at any time $t$.
    Its mean and covariance are given as follows:
    \begin{equation}
        \mu_{lin,t}(x) =
        \Theta_0(x, \vec x) \Theta^{-1}_0(\vec x, \vec x) (I - e^{-\eta \Theta_0(\vec x, \vec x) t / n}) \vec y;
    \end{equation}
    \begin{multline}
        q_{lin,t}(x,x') =
        q_{L+1}(x,x') -\\-
        \left(
            \Theta_0(x', \vec x) \Theta^{-1}_0(\vec x, \vec x) (I - e^{-\eta \Theta_0(\vec x, \vec x) t / n}) q_{L+1}(\vec x, x) +
            \Theta_0(x, \vec x) \Theta^{-1}_0(\vec x, \vec x) (I - e^{-\eta \Theta_0(\vec x, \vec x) t / n}) q_{L+1}(\vec x, x')
        \right) +\\+
        \Theta_0(x, \vec x) \Theta^{-1}_0(\vec x, \vec x) (I - e^{-\eta \Theta_0(\vec x, \vec x) t / n}) q_{L+1}(\vec x, \vec x) (I - e^{-\eta \Theta_0(\vec x, \vec x) t / n}) \Theta^{-1}_0(\vec x, \vec x) \Theta_0(\vec x, x').
    \end{multline}

    Assume that the limit kernel is bounded away from zero: $\lambda_{min}(\Theta_0(\vec x, \vec x)) \geq \lambda_0 > 0$.
    Given this, the model converges to the following limit GP as $t \to \infty$:
    \begin{equation}
        \mu_{lin,\infty}(x) =
        \Theta_0(x, \vec x) \Theta^{-1}_0(\vec x, \vec x) \vec y;
    \end{equation}
    \begin{multline}
        q_{lin,\infty}(x,x') =
        q_{L+1}(x,x') + \Theta_0(x, \vec x) \Theta^{-1}_0(\vec x, \vec x) q_{L+1}(\vec x, \vec x) \Theta^{-1}_0(\vec x, \vec x) \Theta_0(\vec x, x') -\\-
        \left(
            \Theta_0(x', \vec x) \Theta^{-1}_0(\vec x, \vec x) q_{L+1}(\vec x, x) +
            \Theta_0(x, \vec x) \Theta^{-1}_0(\vec x, \vec x) q_{L+1}(\vec x, x')
        \right).
    \end{multline}
    Note that the exact bayesian posterior inference gives a different result:
    \begin{equation}
        \mu_{lin}(x\mid\vec x) =
        q_{L+1}(x, \vec x) q^{-1}_{L+1}(\vec x, \vec x) \vec y;
    \end{equation}
    \begin{equation}
        q_{lin}(x,x'\mid\vec x) =
        q_{L+1}(x,x') - q_{L+1}(x, \vec x) q^{-1}_{L+1}(\vec x, \vec x) q_{L+1}(\vec x, x').        
    \end{equation}
    Nevertheless, if we consider training only the last layer of the network, the tangent kernel becomes:
    \begin{equation}
        \Theta(x,x') =
        \Theta_L(x,x') =
        q_{L+1}(x,x').
    \end{equation}
    Given this, the two GPs, result of NN training and exact posterior, coincide.

    Let us return to the assumption of positive defniteness of the limit kernel.
    \cite{du2018gradient} proved that if no inputs are parallel, this assumption holds:
    \begin{theorem}
        If for any $i \neq j$ $x_i^T x_j < \|x_i\|_2 \|x_j\|_2$ then $\lambda_0 := \lambda_{min}(\Theta_0(\vec x, \vec x)) > 0$.
    \end{theorem}

    \section{Stationarity of the kernel}

    Assume $k=1$; in this case NTK is scalar-valued.
    For analytic activation function $\phi$ we have the following:
    \begin{equation}
        \EE_\theta (\Theta_t(x_1, x_2) - \Theta_0(x_1, x_2)) =
        \sum_{k=1}^\infty \left(\EE_\theta \left(\left.\frac{d^k\Theta_t(x_1, x_2)}{dt^k}\right|_{t=0}\right)\frac{t^k}{k!}\right).
    \end{equation}
    Hence if we show that all derivatives of the NTK at $t=0$ vanish as $n\to\infty$, this would mean that the NTK does not evolve with time for large $n$: $\Theta_t(x,x') \to \Theta_0(x,x')$ as $n \to \infty$.

    Let us consider $l_2$-loss: $\ell(y,z) = \frac{1}{2} (y-z)^2$.
    Consider the first derivative:
    \begin{multline}
        \EE_\theta \left(\left.\frac{d\Theta_t(x_1, x_2)}{dt}\right|_{t=0}\right) =
        \EE_\theta \left(\left.\frac{d(\nabla^T_\theta f_t(x_1) \nabla_\theta f_t(x_2))}{dt}\right|_{t=0}\right) =
        \EE_\theta \left( \left. \left(\dot\theta_t^T \nabla_\theta \nabla^T_\theta f_t(x_1) \nabla_\theta f_t(x_2) + (x_1 \leftrightarrow x_2)\right) \right|_{t=0} \right) =\\=
        \EE_{x,y} \EE_\theta (\eta (y - f_0(x)) \nabla^T_\theta f_0(x) \nabla_\theta \nabla^T_\theta f_0(x_1) \nabla_\theta f_0(x_2) + (x_1 \leftrightarrow x_2)).
        \label{eq:dtheta_dt}
    \end{multline}
    We shall show that it is $O(n^{-1})$, and that it also implies that all higher-order derivatives are $O(n^{-1})$ too.

    From now on, we shall consider only initialization: $t=0$; for this reason, we shall omit the subscript $0$.
    Following \cite{Dyer2020Asymptotics}, we start with a definition of a correlation function.
    Define a rank-$k$ derivative tensor $T_{\mu_1 \ldots \mu_k}$ as follows:
    \begin{equation}
        T_{\mu_1 \ldots \mu_k}(x; f) =
        \frac{\partial^k f(x)}{\partial \theta^{\mu_1} \ldots \partial \theta^{\mu_k}}.
    \end{equation}
    For $k=0$ we define $T(x; f) = f(x)$.
    We are now ready to define the correlation function $C$:
    \begin{equation}
        C(x_1,\ldots,x_m) =
        \sum_{\mu_1,\ldots,\mu_{k_m}} \Delta_{\mu_1 \ldots \mu_{k_m}}^{(\pi)} \EE_\theta \left(
            T_{\mu_1 \ldots \mu_{k_1}}(x_1) T_{\mu_{k_1+1} \ldots \mu_{k_2}}(x_2) \ldots T_{\mu_{k_{m-1}+1} \ldots \mu_{k_m}}(x_m) 
        \right).
    \end{equation}
    Here $0 \leq k_1 \leq \ldots \leq k_m$, $k_m$ and $m$ are even, $\pi \in S_{k_m}$ is a permutation, and $\Delta_{\mu_1 \ldots \mu_{k_m}}^{(\pi)} = \delta_{\mu_{\pi(1)} \mu_{\pi(2)}} \ldots \delta_{\mu_{\pi(k_m-1)} \mu_{\pi(k_m)}}$.
    For example, 
    \begin{multline}
        \EE_\theta (f(x) \nabla^T_\theta f(x) \nabla_\theta \nabla^T_\theta f(x_1) \nabla_\theta f(x_2)) = 
        \sum_{\mu,\nu} \EE_\theta (f(x) \partial_\mu f(x) \partial^2_{\mu,\nu} f(x_1) \partial_\nu f(x_2)) =\\= 
        \sum_{\mu_1,\mu_2,\mu_3,\mu_4} \delta_{\mu_1 \mu_2} \delta_{\mu_3 \mu_4} \EE_\theta (f(x) \partial_{\mu_1} f(x) \partial^2_{\mu_2,\mu_3} f(x_1) \partial_{\mu_4} f(x_2)) =
        C(x,x,x_1,x_2)
        \label{eq:dtheta_dt_as_corr_f}
    \end{multline}
    is a correlation function with $m=4$, $k_1=0$, $k_2=1$, $k_3=3$, $k_4=4$, and $\pi(j) = j$.
    Moreover, $\EE_\theta ((f(x)-y) \nabla^T_\theta f(x) \nabla_\theta \nabla^T_\theta f(x_1) \nabla_\theta f(x_2))$ is a correlation function too: consider $f_y(x) = f(x) - y$ instead of $f(x)$.
    Hence the whole~(\ref{eq:dtheta_dt}) is a linear combination of correlation functions.

    If two derivative tensors have two indices that are summed over, we shall say that they are contracted.
    Formally, we shall say that $T_{\mu_{k_{i-1}+1} \ldots \mu_{k_i}}(x_i)$ is contracted with $T_{\mu_{k_{j-1}+1} \ldots \mu_{k_j}}(x_j)$ for $1 \leq i,j \leq m$, if there exists an even $s \leq k_m$ such that $k_{i-1} < \pi(s-1) \leq k_i$, while $k_{j-1} < \pi(s) \leq k_j$, or vice versa.

    Define the \emph{cluster graph} $G_C(V,E)$ as a non-oriented non-weighted graph with vertices $V = \{v_1, \ldots, v_m\}$ and edges $E = \{(v_i,v_j) \, | \, \text{$T(x_i)$ and $T(x_j)$ are contracted in $C$}\}$.
    Let $n_e$ be the number of even-sized connected components of $G_C(V,E)$ and $n_o$ be the number of odd-sized components.
    \begin{conjecture}[\cite{Dyer2020Asymptotics}]
        \label{conj:C_asymptotics}
        If $m$ is even, $C(x_1,\ldots,x_m) = O_{n\to\infty}(n^{s_C})$, where $s_C = n_e + n_o / 2 - m / 2$.
        If $m$ is odd, $C(x_1,\ldots,x_m) = 0$.
    \end{conjecture}
    Applying this conjecture to~(\ref{eq:dtheta_dt_as_corr_f}) gives $C(x,x,x_1,x_2) = O(n^{-1})$ ($n_e=0$, $n_0=2$, $m=4$), hence the whole eq.~(\ref{eq:dtheta_dt}) is $O(n^{-1})$.

    Let us show that having the first derivative of the NTK being $O(n^{-1})$ implies all higher-order derivatives to be $O(n^{-1})$.
    \begin{lemma}[\cite{Dyer2020Asymptotics}]
        \label{lemma:derivative_asymptotics}
        Suppose Conjecture~\ref{conj:C_asymptotics} holds.
        Let $C(\vec x) = \EE_\theta F(\vec x; \theta)$ be a correlation function and suppose $C(\vec x) = O(n^{s_C})$ for $s_C$ defined in Conjecture~\ref{conj:C_asymptotics}.
        Then $\EE_\theta d^k F(\vec x; \theta) / dt^k = O(n^{s_C})$ $\forall k \geq 1$.
    \end{lemma}
    \begin{proof}
        Consider the first derivative:
        \begin{multline}
            \EE_\theta \frac{dF(\vec x)}{dt} =
            \EE_\theta (\dot\theta^T \nabla_\theta F(\vec x)) =
            \EE_{x,y} \EE_\theta (\eta (y - f(x)) \nabla^T_\theta f(x) \nabla_\theta F(\vec x)) =\\=
            \eta \EE_{x,y} \EE_\theta (y \nabla^T_\theta f(x) \nabla_\theta F(\vec x)) -
            \eta \EE_{x,y} \EE_\theta (f(x) \nabla^T_\theta f(x) \nabla_\theta F(\vec x)).
        \end{multline}
        This is a sum of linear combination of correlation functions.
        By Conjecture~\ref{conj:C_asymptotics}, the first sum evaluates to zero, while the second one has $m' = m+2$, $n_e'$ even clusters, and $n_o'$ odd clusters.
        If $\nabla_\theta f(x)$ is contracted with an even cluster of $C$, we have $n_e' = n_e - 1$, $n_o' = n_o + 2$.
        In contrast, if $\nabla_\theta f(x)$ is contracted with an odd cluster of $C$, we have $n_e' = n_e + 1$, $n_o' = n_o$.
        
        In the first case we have $s_C' = n_e' + n_o'/2 - m'/2 = s_C - 1$, while for the second $s_C' = s_C$.
        In any case, the result is a linear combination of correlation functions with $s_C' \leq s_C$ for each.
    \end{proof}

    \subsection{Finite width corrections for the NTK}

    Let us define $O_{1,t}(x) = f_t(x)$ and for $s \geq 2$
    \begin{equation}
        O_{s,t}(x_1, \ldots, x_s) = 
        \nabla^T_\theta O_{s-1,t}(x_1,\ldots,x_{s-1}) \nabla_\theta f_t(x_s).
    \end{equation}
    In this case $O_{2,t}(x_1,x_2)$ is the empirical kernel $\hat\Theta_t(x_1,x_2)$.
    Note that $O_{s,t}$ evolves as follows:
    \begin{equation}
        \dot O_{s,t}(x_1, \ldots, x_s) = 
        \eta \EE_{x,y} (y - f_t(x)) \nabla^T_\theta f_t(x) \nabla_\theta O_{s,t}(x_1, \ldots, x_s) =
        \eta \EE_{x,y} (y - f_t(x)) O_{s+1,t}(x_1, \ldots, x_s, x).
    \end{equation}

    Since $O_s$ has $s$ derivative tensors and a single cluster, by the virtue of Conjecture~\ref{conj:C_asymptotics}, $\EE_\theta O_{s,0} = O(n^{1 - s/2})$ for even $s$ and $\EE_\theta O_{s,0} = 0$ for odd $s$.
    At the same time, $\EE_\theta \dot O_{s,0} = O(n^{1 - (s+2)/2}) = O(n^{-s/2})$ for even $s$ and $\EE_\theta \dot O_{s,0} = O(n^{1 - (s+1)/2}) = O(n^{1/2 - s/2})$ for odd $s$.

    As for the second moments, we have $\EE_\theta (O_{s,0})^2 = O(n^{2 - s})$ for even $s$ and $\EE_\theta (O_{s,0})^2 = O(n^{1 - s})$ for odd $s$.
    Similarly, we have $\EE_\theta (\dot O_{s,0})^2 = O(n^{2/2 - (2s+2)/2}) = O(n^{-s})$ for even $s$ and $\EE_\theta (\dot O_{s,0})^2 = O(n^{2 - (2s+2)/2}) = O(n^{1 - s})$ for odd $s$.

    The asymptotics for the first two moments implies the asymptotic for a random variable itself:
    \begin{equation}
        O_{s,0}(x_{1:s}) =
        \begin{cases}
            O(n^{1 - s/2}) &\text{for even $s$;}
            \\
            O(n^{1/2 - s/2}) &\text{for odd $s$;}
        \end{cases}
        \qquad
        \dot O_{s,0}(x_{1:s}) =
        \begin{cases}
            O(n^{-s/2}) &\text{for even $s$;}
            \\
            O(n^{1/2 - s/2}) &\text{for odd $s$.}
        \end{cases}
    \end{equation}
    Lemma~\ref{lemma:derivative_asymptotics} gives $\forall k \geq 1$:
    \begin{equation}
        \left.\frac{d^k O_{s,t}}{dt^k}(x_{1:s})\right|_{t=0} =
        \begin{cases}
            O(n^{-s/2}) &\text{for even $s$;}
            \\
            O(n^{1/2 - s/2}) &\text{for odd $s$.}
        \end{cases}
    \end{equation}
    Then given an analytic activation function, we have $\forall t \geq 0$:
    \begin{equation}
        \dot O_{s,t}(x_{1:s}) =
        \sum_{k=1}^\infty \left.\frac{d^k O_{s,t}}{dt^k}(x_{1:s})\right|_{t=0} \frac{t^k}{k!} =
        \begin{cases}
            O(n^{-s/2}) &\text{for even $s$;}
            \\
            O(n^{1/2 - s/2}) &\text{for odd $s$.}
        \end{cases}
    \end{equation}

    This allows us to write a finite system of ODE for the model evolution up to $O(n^{-1})$ terms:
    \begin{equation}
        \dot f_{t}(x_1) = 
        \eta \EE_{x,y} (y - f_t(x)) \Theta_{t}(x_1, x),
        \qquad
        f_0(x_1) =
        f(x_1; \theta),
        \quad
        \theta \sim
        \NN(0, I),
    \end{equation}
    \begin{equation}
        \dot\Theta_{t}(x_1, x_2) = 
        \eta \EE_{x,y} (y - f_t(x)) O_{3,t}(x_1, x_2, x),
        \qquad
        \Theta_0(x_1, x_2) =
        \nabla_\theta^T f_0(x_1) \nabla_\theta f_0(x_2),
    \end{equation}
    \begin{equation}
        \dot O_{3,t}(x_1, x_2, x_3) = 
        \eta \EE_{x,y} (y - f_t(x)) O_{4,t}(x_1, x_2, x_3, x),
        \qquad
        O_{3,0}(x_1, x_2, x_3) =
        \nabla_\theta^T \Theta_0(x_1, x_2) \nabla_\theta f_0(x_3),
    \end{equation}
    \begin{equation}
        \dot O_{4,t}(x_1, x_2, x_3, x_4) = 
        O(n^{-2}),
        \qquad
        O_{4,0}(x_1, x_2, x_3, x_4) =
        \nabla_\theta^T O_{3,0}(x_1, x_2, x_3) \nabla_\theta f_0(x_4).
    \end{equation}
    Let us expand all the quantities wrt $n^{-1}$:
    \begin{equation}
        O_{s,t}(x_{1:s}) =
        O_{s,t}^{(0)}(x_{1:s}) + n^{-1} O_{s,t}^{(1)}(x_{1:s}) + O(n^{-2}),
    \end{equation}
    where $O_{s,t}^{(k)}(x_{1:s}) = \Theta_{n\to\infty}(1)$.
    Then the system above transforms into the following:
    \begin{equation}
        \dot f_{t}^{(0)}(x_1) = 
        \eta \EE_{x,y} (y - f_t^{(0)}(x)) \Theta_{t}^{(0)}(x_1, x),
        \lim_{n\to\infty} f(x_1; \theta_0),
    \end{equation}
    \begin{equation}
        \dot f_{t}^{(1)}(x_1) = 
        \eta \EE_{x,y} ((y - f_t^{(0)}(x)) \Theta_{t}^{(1)}(x_1, x) - f_t^{(1)}(x) \Theta_{t}^{(0)}(x_1, x)),
    \end{equation}
    \begin{equation}
        \Theta_t^{(0)}(x_1, x_2) =
        \nabla_\theta^T f_0^{(0)}(x_1) \nabla_\theta f_0^{(0)}(x_2),
    \end{equation}
    \begin{equation}
        \dot\Theta_{t}^{(1)}(x_1, x_2) = 
        \eta \EE_{x,y} (y - f_t^{(0)}(x)) O_{3,t}^{(1)}(x_1, x_2, x),
    \end{equation}
    \begin{equation}
        \dot O_{3,t}^{(1)}(x_1, x_2, x_3) = 
        \eta \EE_{x,y} (y - f_t^{(0)}(x)) O_{4,t}^{(1)}(x_1, x_2, x_3, x),
    \end{equation}
    \begin{equation}
        O_{4,t}^{(1)}(x_1, x_2, x_3, x_4) =
        \nabla_\theta^T O_{3,0}^{(0)}(x_1, x_2, x_3) \nabla_\theta f_0^{(0)}(x_4),
    \end{equation}
    where we have ignored the initial conditions for the time being.
    Integrating this system is straightforward:
    \begin{equation}
        f_{t}^{(0)}(\vec x) =
        \vec y + e^{-\eta \Theta_0^{(0)}(\vec x, \vec x) t / n} (f_0^{(0)}(\vec x) - \vec y),
    \end{equation}
    where $\vec x$ is a train dataset of size $n$.
    For the sake of brevity, let us introduce the following definition:
    \begin{equation}
        \Delta f_t^{(0)}(x) =
        e^{-\eta \Theta_0^{(0)}(x, \vec x) t / n} (f_0^{(0)}(\vec x) - \vec y).
    \end{equation}
    This gives:
    \begin{equation}
        O_{3,t}^{(1)}(x_1, x_2, x_3) = 
        O_{3,0}^{(1)}(x_1, x_2, x_3) - 
        \eta \EE_{x',y'} \int_{0}^t O_{4,0}^{(1)}(x_1, x_2, x_3, x') \Delta f_{t'}^{(0)}(x') \, dt'.
    \end{equation}
    \begin{multline}
        \Theta_{t}^{(1)}(x_1, x_2) = 
        \Theta_{0}^{(1)}(x_1, x_2) -
        \eta \EE_{x',y'} \int_{0}^t O_{3,0}^{(1)}(x_1, x_2, x) \Delta f_{t'}^{(0)}(x') \, dt' + 
        \\+
        \eta^2 \EE_{x'',y''} \EE_{x',y'} \int_{0}^{t} \int_{0}^{t''} O_{4,0}^{(1)}(x_1, x_2, x'', x') \Delta f_{t'}^{(0)}(x') \Delta f_{t''}^{(0)}(x'') \, dt' \, dt''.
    \end{multline}
    \begin{equation}
        \dot f_{t}^{(1)}(x_1) = 
        -\eta \EE_{x,y} \left(\Delta f_{t}^{(0)}(x) \Theta_{t}^{(1)}(x_1, x) + f_t^{(1)}(x) \Theta_{t}^{(0)}(x_1, x)\right),
    \end{equation}

    \begin{equation}
        \dot f(t) =
        A f(t) + g(t),
        \quad
        f(0) =
        f_0;
    \end{equation}
    \begin{equation}
        f(t) =
        C(t) e^{At};
        \quad
        \dot C(t) e^{At} + C(t) A e^{At} =
        C(t) A e^{At} + g(t);
        \quad
        \dot C(t) = g(t) e^{-At}.
    \end{equation}

    \begin{equation}
        f_t^{(1)}(\vec x_1) =
        e^{-\eta \Theta_0^{(0)}(\vec x_1, \vec x) t / n} C_t(\vec x);
    \end{equation}
    \begin{equation}
        \dot C_t(\vec x) =
        -\eta \EE_{x',y'} e^{\eta \Theta_0^{(0)}(\vec x, \vec x_1) t / n} \Theta_{t}^{(1)}(\vec x_1, x') \Delta f_{t}^{(0)}(x');
    \end{equation}
    \begin{equation}
        C_t(\vec x) =
        f_0^{(1)}(\vec x) - \eta \EE_{x',y'} \int_0^t e^{\eta \Theta_0^{(0)}(\vec x, \vec x_1) t' / n} \Theta_{t}^{(1)}(\vec x_1, x') \Delta f_{t'}^{(0)}(x') \, dt';
    \end{equation}
    \begin{equation}
        f_t^{(1)}(\vec x_1) =
        e^{-\eta \Theta_0^{(0)}(\vec x_1, \vec x) t / n} f_0^{(1)}(\vec x) -
        \eta e^{-\eta \Theta_0^{(0)}(\vec x_1, \vec x_2) t / n} \EE_{x',y'} \int_0^t e^{\eta \Theta_0^{(0)}(\vec x_2, \vec x_3) t' / n} \Theta_{t}^{(1)}(\vec x_3, x') \Delta f_{t'}^{(0)}(x') \, dt'.
    \end{equation}

    \subsection{Proof of Conjecture~\ref{conj:C_asymptotics} for linear nets}

    \subsubsection{Shallow nets.}

    We first consider shallow linear nets:
    \begin{equation}
        f(x) = 
        \frac{1}{\sqrt{n}} a^T W x.
    \end{equation}
    We shall use the following theorem:
    \begin{theorem}[\cite{isserlis1918formula}]
        Let $z = (z_1, \ldots, z_l)$ be a centered multivariate Gaussian variable.
        Then, for any positive $k$, for any ordered set of indices $i_{1:2k}$,
        \begin{equation}
            \EE_z (z_{i_1} \cdots z_{i_{2k}}) =
            \frac{1}{2^k k!} \sum_{\pi \in S_{2k}} \EE (z_{i_{\pi(1)}} z_{i_{\pi(2)}}) \cdots \EE (z_{i_{\pi(m-1)}} z_{i_{\pi(m)}}) =
            \sum_{p \in P^2_{2k}} \prod_{\{j_1,j_2\} \in p} \EE (z_{i_{j_1}} z_{i_{j_2}}),
        \end{equation}
        where $P^2_{2k}$ is a set of all unordered pairings $p$ of a $2k$-element set, i.e.
        \begin{equation}
            P^2_{2k} = 
            \bigcup_{\pi \in S_{2k}} \{\{\pi(1), \pi(2)\}, \ldots, \{\pi(2k-1), \pi(2k)\}\}.
        \end{equation}
        At the same time,
        \begin{equation}
            \EE_z (z_{i_1} \cdots z_{i_{2k-1}}) =
            0.
        \end{equation}
    \end{theorem}
    For example,
    \begin{equation}
        \EE_z (z_1 z_2 z_3 z_4) =
        \EE_z (z_1 z_2) \EE_z (z_3 z_4) + \EE_z (z_1 z_3) \EE_z (z_2 z_4) + \EE_z (z_1 z_4) \EE_z (z_2 z_3).
    \end{equation}
    
    Consider a correlation function without derivatives:
    \begin{multline}
        C(x_1,\ldots,x_m) =
        \EE_\theta (f(x_1) \ldots f(x_m)) =
        n^{-m/2} \EE_\theta (a_{i_1} W^{i_1} x_1 \ldots a_{i_m} W^{i_m} x_m) =\\=
        n^{-m/2} \EE_\theta (a_{i_1} \cdots a_{i_m}) \EE_\theta (W^{i_1} x_1 \cdots W^{i_m} x_m) =\\=
        n^{-m/2} [m = 2k] \left(\sum_{p_a \in P^2_{2k}} \prod_{\{j^a_1,j^a_2\} \in p_a} \delta_{i_{j^a_1} i_{j^a_2}}\right)
        \left(\sum_{p_w \in P^2_{2k}} \prod_{\{j^w_1,j^w_2\} \in p_w} \delta^{i_{j^w_1} i_{j^w_2}} x_{j^w_1}^T x_{j^w_2}\right).
    \end{multline}
    For even $m$, we shall associate a graph $\gamma$ with each pair $(p_a, p_w)$.
    Such a graph has $m$ vertices $(v_1, \ldots, v_m)$.
    For any $\{j^a_1,j^a_2\} \in p_a$ there is an edge $(v_{j^a_1}, v_{j^a_2})$ marked $a$, and an edge $(v_{j^w_1}, v_{j^w_2})$ marked $W$ for any $\{j^w_1,j^w_2\} \in p_w$.
    Hence each vertex has a unique $a$-neighbor and a unique $W$-neighbor; these two can be the same vertex.
    Hence $\gamma$ is a union of cycles.
    We call $\gamma$ a \emph{Feynman diagram of $C$.}
    
    Denote by $\Gamma(C)$ a set of Feynman diagrams of $C$, and by $l_\gamma$ a number of cycles in the diagram $\gamma$.
    It is easy to notice that each cycle contributes a factor of $n$ when one takes a sum over $i_1, \ldots, i_m$.
    Hence we have:
    \begin{equation}
        C(x_1,\ldots,x_m) =
        n^{-m/2} [m = 2k] \sum_{p_a,p_w} \left(n^{l_{\gamma(p_a,p_w)}} \prod_{\{j^w_1,j^w_2\} \in p_w} x_{j^w_1}^T x_{j^w_2} \right) =
        [m = 2k] O_{n\to\infty} (n^{\max_{\gamma \in \Gamma(C)} l_\gamma - m/2}).
    \end{equation}

    Consider now a correlation function with derivatives.
    Assume there is an edge $(v_i,v_j)$ in $G_C$; hence corresponding derivative tensors are contracted in $C$.
    In this case, we should consider only those Feynman diagrams $\gamma$ that have an edge $(v_i,v_j)$, either of $a$ or of $w$ type.
    Denoting a set of such diagrams as $\Gamma(C)$, we get the same bound as before:
    \begin{equation}
        C(x_1,\ldots,x_m) =
        [m = 2k] O_{n\to\infty} (n^{\max_{\gamma \in \Gamma(C)} l_\gamma - m/2}).
    \end{equation}

    In order to illustrate this principle, let us consider the case $m=4$.
    For simplicity, assume also all inputs to be equal: $x_1 = \ldots = x_4 = x$.
    If there are no derivatives, we have:
    \begin{equation}
        \EE_\theta ((f(x))^4) =
        n^{-2} (\delta_{i_1 i_2} \delta_{i_3 i_4} + \delta_{i_1 i_3} \delta_{i_2 i_4} + \delta_{i_1 i_4} \delta_{i_2 i_3}) (\delta^{i_1 i_2} \delta^{i_3 i_4} + \delta^{i_1 i_3} \delta^{i_2 i_4} + \delta^{i_1 i_4} \delta^{i_2 i_3}) (x^T x)^2 =
        (3 + 6 n^{-1}) (x^T x)^2.
    \end{equation}
    In this case there are three Feynman diagrams with two cycles each, and six diagrams with a single cycle.
    Let us introduce a couple of contracted derivative tensors:
    \begin{multline}
        \EE_\theta ((f(x))^2 \nabla_\theta^T f(x) \nabla_\theta f(x)) =
        n^{-2} \EE_\theta (a_{i_1} W^{i_1} x a_{i_2} W^{i_2} x (\delta_{i_3 k} W^{i_3} x \delta^{kl} \delta_{i_4 l} W^{i_4} x + a_{i_3} \delta^{i_3 k} \delta_{kl} a_{i_4} \delta^{i_4 l} x^T x)) =\\=
        n^{-2} \EE_\theta (a_{i_1} W^{i_1} x a_{i_2} W^{i_2} x (\delta_{i_3 i_4} W^{i_3} x W^{i_4} x + a_{i_3} \delta^{i_3 i_4} a_{i_4} x^T x)) =\\=
        n^{-2} (\delta_{i_1 i_2} \delta_{i_3 i_4}) (\delta^{i_1 i_2} \delta^{i_3 i_4} + \delta^{i_1 i_3} \delta^{i_2 i_4} + \delta^{i_1 i_4} \delta^{i_2 i_3}) (x^T x)^2 + n^{-2} (\delta_{i_1 i_2} \delta_{i_3 i_4} + \delta_{i_1 i_3} \delta_{i_2 i_4} + \delta_{i_1 i_4} \delta_{i_2 i_3}) (\delta^{i_1 i_2} \delta^{i_3 i_4}) (x^T x)^2 =\\=
        (2 + 4 n^{-1}) (x^T x)^2.
    \end{multline}
    Here we have only those Feynman diagrams that have an edge $(v_3,v_4)$.
    There are two such diagrams with two cycles each, and four with a single cycle.

    Note that if there is an edge in a cluster graph $G_C$, there is also an edge, $a$ or $w$ type, in each $\gamma$ from $\Gamma(C)$.
    Note also that each cycle in $\gamma$ consists of even number of edges.
    Hence each cycle consists of even clusters and an even number of odd clusters.
    Hence there could be at most $n_e + n_o/2$ cycles in $\gamma$, which proves Conjecture~\ref{conj:C_asymptotics} for shallow linear nets:
    \begin{equation}
        C(x_1,\ldots,x_m) =
        [m = 2k] O_{n\to\infty} (n^{\max_{\gamma \in \Gamma(C)} l_\gamma - m/2}) =
        [m = 2k] O_{n\to\infty} (n^{n_e + n_o/2 - m/2}).
    \end{equation}

    \subsubsection{Deep nets.}

    In the case of a network with $L$ hidden layers, there are $L+1$ edges of types $W_0, \ldots W_L$ adjacent to each node.
    Feynman diagrams are still well-defined, however, it is not obvious how to define the number of loops in this case.
    
    The correct way to do it is to count the loops in a corresponding \emph{double-line diagram.}
    Given a Feynman diagram $\gamma$, define the double-line diagram $DL(\gamma)$ as follows:
    \begin{itemize}
        \item Each vertex $v_i$ of $\gamma$ maps to $L$ vertices $v_i^{(1)}, \ldots, v_i^{(L)}$ in $DL(\gamma)$.
        \item An edge $(v_i,v_j)$ of type $W_0$ maps to an edge $(v_i^{(1)}, v_j^{(1)})$.
        \item An edge $(v_i,v_j)$ of type $W_L$ maps to an edge $(v_i^{(L)}, v_j^{(L)})$.
        \item $\forall l \in [L-1]$ an edge $(v_i,v_j)$ of type $W_l$ maps to a pair of edges: $(v_i^{(l)}, v_j^{(l)})$ and $(v_i^{(l+1)}, v_j^{(l+1)})$.
    \end{itemize}
    We see that each of the $Lm$ vertices of a double-line diagram has degree $2$; hence the number of loops is well-defined.
    For $L=1$, a double-line diagram recovers the corresponding Feynman diagram without edge types.
    For any $L$, we have the following:
    \begin{equation}
        C(x_1,\ldots,x_m) =
        [m = 2k] O_{n\to\infty} (n^{\max_{\gamma \in \Gamma(C)} l_\gamma - Lm/2}),
        \label{eq:C_bound_deep_loops}
    \end{equation}
    where now $l_\gamma$ is a number of loops in $DL(\gamma)$.

    In order to get intuition about this result, let us consider a network with two hidden layers.
    For the sake of simplicity, assume $x$ is a scalar:
    \begin{equation}
        f(x) =
        \frac{1}{n} a^T W v x.
    \end{equation}
    \begin{equation}
        \EE_\theta (f(x_1) f(x_2)) =
        n^{-2} \EE_\theta (a_{i_1} W^{i_1 j_1} v_{j_1} x_1 a_{i_2} W^{i_2 j_2} v_{j_2} x_2) =
        n^{-2} \delta_{i_1 i_2} \delta^{i_1 i_2} \delta^{j_1 j_2} \delta_{j_1 j_2} x_1 x_2 =
        x_1 x_2.
    \end{equation}
    Here both $a$ and $v$ result in a single Kronecker delta, hence they correspond to a single edge in a double-line diagram.
    At the same time, $W$ results in a product of two deltas, in its turn resulting in a pair of edges in the diagram.

    Similar to the case of $L=1$, contracted derivative tensors force the existence of corresponding edges in the Feynman diagram.
    Given a Feynman diagram $\gamma$, define $s_\gamma = l_\gamma - Lm/2$, or, in other words, a number of loops in $DL(\gamma)$ minus a number of vertices in $DL(\gamma)$ halved.
    Let $c_\gamma$ be a number of connected components of $\gamma$.
    We shall prove that 
    \begin{equation}
        s_\gamma \leq c_\gamma - \frac{m}{2}.
        \label{eq:loops_and_components}
    \end{equation}

    Note that eq.~(\ref{eq:loops_and_components}) holds for $L=1$ since all connected components of $\gamma$ are loops in this case.
    Let us express $\gamma$ as a union of its connected components $\gamma'$; given this, $s_\gamma = \sum_{\gamma'} s_{\gamma'}$.
    We are going to show that $s_{\gamma'} \leq 1 - m'/2$, where $m'$ is a number of vertices in the component $\gamma'$.
    The latter will imply $s_\gamma \leq c_\gamma - m/2$.

    Let $v$, $e$, and $f$ be a number of vertices, a number of edges, and a number faces of $\gamma'$.
    We already know that $v = m'$, $e = (L+1)m'/2$, and $f = l_{\gamma'}$.
    Hence $s_{\gamma'} = l_{\gamma'} - L m'/2 = f - L v/2$.
    On the other hand, the Euler characteristic of $\gamma'$ is $\chi = v - e + f = s_{\gamma'} + m' (1 + L/2) - (L+1)m'/2 = s_{\gamma'} + m'/2$.
    Since $\gamma'$ is a triangulation of a planar surface with at least one boundary, $\chi \leq 1$.
    Hence $s_{\gamma'} \leq 1 - m'/2$, which was required.

    Consequently, we may rewrite~(\ref{eq:C_bound_deep_loops}) as follows:
    \begin{equation}
        C(x_1,\ldots,x_m) =
        [m = 2k] O_{n\to\infty} (n^{\max_{\gamma \in \Gamma(C)} c_\gamma - m/2}).
    \end{equation}
    It is now easy to conclude that $c_\gamma \leq n_e + n_o/2$.
    Indeed, each connected component of $\gamma$ consists of connected components of the cluster graph $G_C$.
    Hence $c_\gamma \leq n_e + n_o$.
    Moreover, each connected component of $\gamma$ consists of even number of vertices, hence it can contain only even number of odd connected components of $G_C$.
    This gives $c_\gamma \leq n_e + n_o/2$, which is required.

    \section{GD convergence via kernel stability}

    Recall the model prediction dynamics (eq.~(\ref{eq:f_evolution})):
    \begin{equation}
        \dot f_t(x') =
        -\eta \EE_{x,y} \left.\frac{\partial\ell(y,z)}{\partial z}\right|_{z=f_t(x)} \hat\Theta_t(x',x).
    \end{equation}
    On the train dataset $(\vec x, \vec y)$ we have the following:
    \begin{equation}
        \dot f_t(\vec x) =
        -\frac{\eta}{m} \hat\Theta_t(\vec x, \vec x) \left.\frac{\partial\ell(\vec y,\vec z)}{\partial \vec z}\right|_{\vec z=f_t(\vec x)}.
    \end{equation}
    For the special case of square loss:
    \begin{equation}
        \dot f_t(\vec x) =
        \frac{\eta}{m} \hat\Theta_t(\vec x, \vec x) (\vec y - f_t(\vec x)).
    \end{equation}
    Let us consider the evolution of a loss:
    \begin{equation}
        \frac{d}{dt}\left(\frac{1}{2}\|\vec y - f_t(\vec x)\|_2^2\right) =
        -\frac{\eta}{m} (\vec y - f_t(\vec x))^T \hat\Theta_t(\vec x, \vec x) (\vec y - f_t(\vec x)) \leq
        -\frac{\eta}{m} \lambda_{min}(\hat\Theta_t(\vec x, \vec x)) \|\vec y - f_t(\vec x)\|_2^2
    \end{equation}
    Consider $\lambda_{min} \geq 0$ such that $\forall t \geq 0$ $\lambda_{min}(\hat\Theta_t(\vec x, \vec x)) \geq \lambda_{min}$.
    This allows us to solve the differential inequality:
    \begin{equation}
        \|\vec y - f_t(\vec x)\|_2^2 \leq
        e^{-2\eta \lambda_{min} t / m} \|\vec y - f_0(\vec x)\|_2^2.
    \end{equation}
    Hence having $\lambda_{min} > 0$ ensures that the gradient descent converges to a zero-loss solution.
    There is a theorem that guarantees that the least eigenvalue of the kernel stays separated away from zero for wide-enough NTK-parameterized two-layered networks with ReLU activation:
    \begin{theorem}[\cite{du2018gradient}]
        Consider the following model:
        \begin{equation}
            f(x; a_{1:n}, w_{1:n}) =
            \frac{1}{\sqrt{n}} \sum_{i=1}^n a_i [w_i^T x]_+.
        \end{equation}
        Assume we aim to minimize the square loss on the dataset $(\vec x, \vec y)$ of size $m$ via a gradient descent on the input weights:
        \begin{equation}
            \dot w_i(t) =
            \frac{1}{\sqrt{n}} \sum_{k=1}^m (y_k - f(x_k; a_{1:n}, w_{1:n}(t))) a_i [w_i^T(t) x_k > 0] x_k,
            \quad
            w_i(0) \sim \NN(0, I_{n_0}),
            \quad
            a_i \sim U(\{-1,1\})
            \quad
            \forall i \in [n].
        \end{equation}
        Assume also that $\|x_k\|_2 \leq 1$ and $|y_k| < 1$ $\forall k \in [m]$.
        Let $H^\infty$ be an expected gramian of the NTK at initialization and let $\lambda_0$ be its least eigenvalue:
        \begin{equation}
            H^\infty_{kl} =
            \EE_{w \sim \NN(0, I_{n_0})} [w^T x_k > 0] [w^T x_l > 0] x_k^T x_l,
            \qquad
            \lambda_0 =
            \lambda_{min}(H^\infty).
        \end{equation}
        Then $\forall \delta \in (0,1)$ taking
        \begin{equation}
            n >
            2^8 \max\left(
                \frac{2^5 3^2}{\pi} \frac{m^6}{\lambda_0^4 \delta^3}, \;
                \frac{m^2}{\lambda_0^2} \log\left(\frac{2m}{\delta}\right)
            \right) =
            \Omega\left(\frac{m^6}{\lambda_0^4 \delta^3}\right)
        \end{equation}
        guarantees that w.p. $\geq 1-\delta$ over initialization we have an exponential convergence to a zero-loss solution:
        \begin{equation}
            \|\vec y - f_t(\vec x)\|_2^2 \leq
            e^{-\lambda_0 t} \|\vec y - f_0(\vec x)\|_2^2.
        \end{equation}
        \label{thm:gd_convergence_via_kernel_stability}
    \end{theorem}

    \begin{proof}
        From what was shown above, it suffices to show that $\lambda_{min}(H(t)) \geq \lambda_0 / 2$ with given probability for $n$ sufficiently large, where $H(t)$ is a gram matrix of the NTK at time $t$:
        \begin{equation}
            H_{kl}(t) =
            \frac{1}{n} \sum_{i=1}^n [w_i^T(t) x_k > 0] [w_i^T(t) x_l > 0] x_k^T x_l.
        \end{equation}
        We shall first show that $H(0) \geq 3 \lambda_0 / 4$:
        \begin{lemma}
            $\forall \delta \in (0,1)$ taking $n \geq 128 m^2 \lambda_0^{-2} \log(m/\delta)$ guarantees that w.p. $\geq 1-\delta$ over initialization we have $\| H(0) - H^\infty \|_2 \leq \lambda_0 / 4$ and $\lambda_{min}(H(0)) \geq 3 \lambda_0 / 4$.
            \label{lemma:initial_gram_matrix_convergence}
        \end{lemma}
        Next, we shall show that the initial Gram matrix $H(0)$ is stable with respect to initial weights $w_{1:n}(0)$:
        \begin{lemma}
            $\forall \delta \in (0,1)$ w.p. $\geq 1-\delta$ over initialization for any set of weights $w_{1:n}$ that satisfy $\forall i \in [n]$ $\|w_i(0) - w_i\|_2 \leq R(\delta) := (\sqrt{2\pi} / 16) \delta \lambda_0 m^{-2}$, the corresponding Gram matrix $H$ satisfies $\| H - H(0) \| < \lambda_0 / 4$ and $\lambda_{min}(H) > \lambda_0 / 2$.
            \label{lemma:initial_gram_matrix_stability}
        \end{lemma}
        After that, we shall show that lower bounded eigenvalues of the Gram matrix gives exponential convergence on the train set. Moreover, weights stay close to initialization, as the following lemma states:
        \begin{lemma}
            Suppose for $s \in [0,t]$ $\lambda_{min}(H(s)) \geq \lambda_0 / 2$.
            Then we have $\| \vec y - f_t(\vec x) \|_2^2 \leq e^{-\lambda_0 t} \| \vec y - f_0(\vec x) \|_2^2$ and for any $i \in [n]$ $\| w_i(t) - w_i(0) \|_2 \leq R' := (2 / \lambda_0) \sqrt{(m / n)} \| \vec y - f_0(\vec x) \|_2$.
            \label{lemma:lower_bounded_eigenvalues_consequence}
        \end{lemma}
        Finally, we shall show that when $R' < R(\delta)$, the conditions of Lemma~\ref{lemma:initial_gram_matrix_stability} and of Lemma~\ref{lemma:lower_bounded_eigenvalues_consequence} hold $\forall t \geq 0$ simultaneously:
        \begin{lemma}
            Let $\delta \in (0,1/3)$.
            If $R' < R(\delta)$, then w.p. $\geq 1-3\delta$ over initialization $\forall t \geq 0$ $\lambda_{min}(H(t)) \geq \lambda_0 / 2$ and $\forall i \in [n]$ $\| w_i(t) - w_i(0) \|_2 \leq R'$ and $\| \vec y - f_t(\vec x) \|_2^2 \leq e^{-\lambda_0 t} \| \vec y - f_0(\vec x) \|_2^2$.
            \label{lemma:the_two_lemmas_hold_simultaneously}
        \end{lemma}
        Hence for $\delta \in (0,1)$, $R' < R(\delta/3)$ suffices for the theorem to hold:
        \begin{equation}
            \frac{2 \sqrt{m} \| \vec y - f_0(\vec x) \|_2}{\lambda_0 \sqrt{n}} =
            R' <
            R(\delta/3) =
            \frac{\sqrt{2\pi} \delta \lambda_0}{48 m^2},
        \end{equation}
        which is equivalent to:
        \begin{equation}
            n >
            \frac{2^9 3^2 m^5 \| \vec y - f_0(\vec x) \|_2^2}{\pi \lambda_0^4 \delta^2}.
        \end{equation}
        We further bound:
        \begin{equation}
            \EE \| \vec y - f_0(\vec x) \|_2^2 =
            \EE \|\vec y\|_2^2 - 2 \vec y^T \EE f_0(\vec x) + \EE \|f_0(\vec x)\|_2^2 \leq
            2m.
        \end{equation}
        Hence by Markov's inequality, w.p. $\geq 1-\delta$
        \begin{equation}
            \| \vec y - f_0(\vec x) \|_2^2 \leq
            \frac{\EE \| \vec y - f_0(\vec x) \|_2^2}{\delta} \leq
            \frac{2m}{\delta}.
        \end{equation}
        By a union bound, in order to have the desired properties w.p. $\geq 1-2\delta$, we need:
        \begin{equation}
            n >
            \frac{2^{10} 3^2 m^6}{\pi \lambda_0^4 \delta^3}.
        \end{equation}
        If we want the things hold w.p. $\geq 1-\delta$, noting Lemma~\ref{lemma:initial_gram_matrix_convergence}, we finally need the following:
        \begin{equation}
            n >
            \max\left(
                (2^{13} 3^2 / \pi) \frac{m^6}{\lambda_0^4 \delta^3}, \;
                2^8 \frac{m^2}{\lambda_0^2} \log\left(\frac{2m}{\delta}\right)
            \right).
        \end{equation}
    \end{proof}

    Let us prove the lemmas.
    \begin{proof}[Proof of Lemma~\ref{lemma:initial_gram_matrix_convergence}]
        Since all $H_{kl}(0)$ are independent random variables, we can apply Hoeffding's inequality for each of them independently:
        \begin{equation}
            \PP(|H_{kl}(0) - H_{kl}^\infty| \geq \epsilon) \leq
            2 e^{-n\epsilon^2/2}.
        \end{equation}
        For a given $\delta$, take $\epsilon$ such that $\delta = 2 e^{-n\epsilon^2/2}$.
        This gives $\epsilon = \sqrt{-2\log(\delta/2) / n}$, or,
        \begin{equation}
            |H_{kl}(0) - H_{kl}^\infty| \leq
            \frac{2\sqrt{\log(1/\delta)}}{\sqrt{n}}
            \quad
            \text{w.p. $\geq 1-\delta$ over initialization.}
        \end{equation}
        Applying a union bound gives:
        \begin{equation}
            |H_{kl}(0) - H_{kl}^\infty| \leq
            \frac{2\sqrt{\log(m^2/\delta)}}{\sqrt{n}} \leq
            \sqrt{\frac{8 \log(m/\delta)}{n}}
            \quad
            \forall k,l \in [m]
            \quad
            \text{w.p. $\geq 1-\delta$ over initialization.}
        \end{equation}
        Hence
        \begin{equation}
            \| H(0) - H^\infty \|_2^2 \leq
            \| H(0) - H^\infty \|_F^2 \leq
            \sum_{k,l = 1}^m |H_{kl}(0) - H_{kl}^\infty|^2 \leq
            \frac{8 m^2 \log(m/\delta)}{n}.
        \end{equation}
        In order to get $\| H(0) - H^\infty \|_2 \leq \lambda_0 / 4$, we need to solve:
        \begin{equation}
            \sqrt{\frac{8 m^2 \log(m/\delta)}{n}} \leq
            \frac{\lambda_0}{4}.
        \end{equation}
        This gives:
        \begin{equation}
            n \geq
            \frac{128 m^2 \log(m/\delta)}{\lambda_0^2}.
        \end{equation}
        This gives $\| H(0) - H^\infty \|_2 \leq \lambda_0 / 4$, which implies:
        \begin{equation}
            \lambda_{min}(H(0)) =
            \lambda_{min}(H^\infty + (H(0) - H^\infty)) \geq
            \lambda_{min}(H^\infty) - \lambda_{max}(H(0) - H^\infty) \geq
            \lambda_0 - \lambda_0 / 4 =
            3 \lambda_0 / 4.
        \end{equation}
    \end{proof}

    \begin{proof}[Proof of Lemma~\ref{lemma:initial_gram_matrix_stability}]
        We define the event in the space of $w_i(0)$ realizations:
        \begin{equation}
            A_{ki} =
            \{\exists w: \; \|w - w_i(0)\|_2 \leq R, \; [w^T x_k \geq 0] \neq [w_i^T(0) x_k \geq 0]\}.
        \end{equation}
        When $A_{ki}$ holds, we can always take $w = w'_{ki}$, where $w'_{ki}$ is defined as follows:
        \begin{equation}
            w'_{ki} =
            \begin{cases}
                w_i(0) - R x_k, & $\text{if $w_i^T(0) x_k \geq 0$}$ \\
                w_i(0) + R x_k, & $\text{if $w_i^T(0) x_k < 0$.}$
            \end{cases}
        \end{equation}
        Hence $A_{ki}$ is equivalent to the following:
        \begin{equation}
            A'_{ki} =
            \{[w_{ki}^{\prime,T} x_k \geq 0] \neq [w_i^T(0) x_k \geq 0]\}.
        \end{equation}
        This event holds iff $|w_i^T(0) x_k| < R$.
        Since $w_i(0) \sim \NN(0,I)$, we have
        \begin{equation}
            \PP(A_{ki}) =
            \PP(A'_{ki}) =
            \PP_{z \sim \NN(0,1)} \{|z| < R\} \leq
            \frac{2R}{\sqrt{2\pi}}.
        \end{equation}
        We can bound the entry-wise deviation of $H'$ from the $H(0)$ matrix:
        \begin{multline}
            \EE |H_{kl}(0) - H'_{kl}| =
            \EE \left(\frac{1}{n} \left|x_k^T x_l \sum_{i=1}^n \left([w_i^T(0) x_k > 0] [w_i^T(0) x_l > 0] - [w_{ki}^{\prime,T} x_k > 0] [w_{li}^{\prime,T} x_l > 0]\right) \right|\right) \leq
            \\\leq
            \frac{1}{n} \sum_{i=1}^n \EE [A'_{ki} \cup A'_{li}] \leq
            \frac{4R}{\sqrt{2\pi}}.
        \end{multline}
        Hence $\EE \sum_{k,l=1}^m |H_{kl}(0) - H'_{kl}| \leq 4 m^2 R / \sqrt{2\pi}$.
        Hence by Markov's inequality,
        \begin{equation}
            \sum_{k,l=1}^m |H_{kl}(0) - H'_{kl}| \leq 
            \frac{4 m^2 R}{\sqrt{2\pi} \delta}
            \quad
            \text{w.p. $\geq 1-\delta$ over initialization.}
        \end{equation}
        Since $\| H(0) - H' \|_2 \leq \| H(0) - H' \|_F \leq \sum_{k,l=1}^m |H_{kl}(0) - H'_{kl}|$, the same probabilistic bound holds for $\| H(0) - H' \|_2$.

        Note that $\forall k \in [m]$ $\forall i \in [n]$ for any $w \in \RR^{n_0}$ such that $\| w - w_i(0) \|_2 \leq R$, $[w^T x_k \geq 0] \neq [w_i^T(0) x_k \geq 0]$ implies $[w_{ki}^{\prime,T} x_k \geq 0] \neq [w_i^T(0) x_k \geq 0]$.
        Hence $\forall k,l \in [m]$ for any set of weights $w_{1:n}$ such that $\forall i \in [n]$ $\|w_i - w_i(0)\|_2 \leq R$, $|H_{kl}(0) - H_{kl}| \leq |H_{kl}(0) - H'_{kl}|$.
        This means that w.p. $\geq 1-\delta$ over initialization, for any set of weights $w_{1:n}$ such that $\forall i \in [n]$ $\|w_i - w_i(0)\|_2 \leq R$,
        \begin{equation}
            \| H(0) - H \|_2 \leq
            \| H(0) - H' \|_2 \leq
            \frac{4 m^2 R}{\sqrt{2\pi} \delta}.
        \end{equation}
        In order to get the required bound, it suffices to solve the equation:
        \begin{equation}
            \frac{4 m^2 R}{\sqrt{2\pi} \delta} =
            \frac{\lambda_0}{4},
            \quad
            \text{which gives}
            \quad
            R = 
            \frac{\sqrt{2\pi} \delta \lambda_0}{16 m^2}.
        \end{equation}
        The bound on the minimal eigenvalue is then straightforward:
        \begin{equation}
            \lambda_{min}(H) =
            \lambda_{min}(H(0) + (H - H(0))) \geq
            \lambda_{min}(H(0)) - \lambda_{max}(H - H(0)) \geq
            3 \lambda_0 / 4 - \lambda_0 / 4 =
            \lambda_0 / 2.
        \end{equation}
    \end{proof}

    \begin{proof}[Proof of Lemma~\ref{lemma:lower_bounded_eigenvalues_consequence}]
        For $s \in [0,t]$ we have:
        \begin{equation}
            \frac{d\|\vec y - f_s(\vec x)\|_2^2}{ds} =
            -2 (\vec y - f_s(\vec x))^T H(s) (\vec y - f_s(\vec x)) \leq
            -\lambda_0 \|\vec y - f_s(\vec x)\|_2^2,
        \end{equation}
        which implies:
        \begin{equation}
            \frac{d(\log(\|\vec y - f_s(\vec x)\|_2^2))}{ds} \leq
            -\lambda_0.
        \end{equation}
        Hence
        \begin{equation}
            \log(\|\vec y - f_s(\vec x)\|_2^2) \leq
            \log(\|\vec y - f_0(\vec x)\|_2^2) - \lambda_0 s,
        \end{equation}
        or, equivalently,
        \begin{equation}
            \|\vec y - f_s(\vec x)\|_2^2 \leq
            e^{-\lambda_0 s} \|\vec y - f_0(\vec x)\|_2^2,
        \end{equation}
        which holds, for instance, for $s = t$.
        In order to bound weight deviation, we first bound the gradient norm:
        \begin{multline}
            \left\|\frac{dw_i(s)}{ds}\right\|_2 =
            \left\|\frac{1}{\sqrt{n}} \sum_{k=1}^m (y_k - f_s(x_k)) a_i [w_i^T(s) x_k > 0] x_k\right\|_2 \leq
            \\\leq
            \frac{1}{\sqrt{n}} \sum_{k=1}^m |y_k - f_s(x_k)| \leq
            \sqrt{\frac{m}{n}} \|\vec y - f_s(\vec x)\|_2 \leq
            \sqrt{\frac{m}{n}} e^{-\lambda_0 s / 2} \|\vec y - f_0(\vec x)\|_2.
        \end{multline}
        This gives $\forall i \in [n]$:
        \begin{multline}
            \| w_i(t) - w_i(0) \|_2 =
            \left\|\int_0^t \frac{dw_i(s)}{ds} \, ds\right\|_2 \leq
            \int_0^t \left\|\frac{dw_i(s)}{ds}\right\|_2 \, ds \leq
            \\\leq
            \frac{2 \sqrt{m}}{\lambda_0 \sqrt{n}} \left(1 - e^{-\lambda_0 t / 2}\right) \|\vec y - f_0(\vec x)\|_2 \leq
            \frac{2 \sqrt{m}}{\lambda_0 \sqrt{n}} \|\vec y - f_0(\vec x)\|_2.
        \end{multline}
    \end{proof}

    \begin{proof}[Proof of Lemma~\ref{lemma:the_two_lemmas_hold_simultaneously}]
        Proof by contradiction.
        Take $\delta \in (0,1/3)$ and suppose that $R' < R(\delta)$, however, w.p. $> 3\delta$ over initialization $\exists t^* > 0:$ either $\lambda_{min}(H(t^*)) < \lambda_0 / 2$, or $\exists i \in [n]$ $\|w_i(t^*) - w_i(0)\|_2 > R'$, or $\|\vec y - f_{t^*}(\vec x)\|_2 > \exp(-\lambda_0 t^*) \|\vec y - f_0(\vec x)\|_2$.
        If either of the last two holds, then by Lemma~\ref{lemma:lower_bounded_eigenvalues_consequence}, $\exists s \in [0,t^*]$ $\lambda_{min}(H(s)) < \lambda_0 / 2$.
        If the former holds, we can take $s = t^*$.
        Hence by virtue of Lemma~\ref{lemma:initial_gram_matrix_stability}, for this particular $s$ w.p. $> 2\delta$ over initialization $\exists i \in [n]$ $\| w_i(s) - w_i(0) \|_2 > R(\delta)$.
        Define:
        \begin{equation}
            t_0 =
            \inf \left\{t \geq 0: \; \max_{i \in [n]} \| w_i(t) - w_i(0) \|_2 > R(\delta) \right\}.
        \end{equation}
        Note that w.p. $> 2\delta$ over initialization $t_0 \leq s \leq t^* < \infty$.
        Since $w_i(\cdot)$ is a continuous map, w.p. $> 2\delta$ over initialization $\max_{i \in [n]} \| w_i(t_0) - w_i(0) \|_2 = R(\delta)$.
        Hence by Lemma~\ref{lemma:initial_gram_matrix_stability}, w.p. $> \delta$ over initialization $\forall t \in [0,t_0]$ $\lambda_{min}(H(t)) \geq \lambda_0 / 2$.
        Hence by Lemma~\ref{lemma:lower_bounded_eigenvalues_consequence}, $\forall i \in [n]$ $\| w_i(t_0) - w_i(0) \|_2 \leq R'$.
        Hence w.p. $> \delta$ over initialization we have a contradiction with $\max_{i \in [n]} \| w_i(t_0) - w_i(0) \|_2 = R(\delta)$ and $R' < R(\delta)$.
    \end{proof}

    \subsection{Component-wise convergence guarantees and kernel alignment}


    Denote $\vec u(t) = f_t(\vec x)$.
    We have the following dynamics for quadratic loss:
    \begin{equation}
        \frac{d\vec u(t)}{dt} =
        H(t) (\vec y - \vec u(t)),
        \quad
        u_k(0) =
        \frac{1}{\sqrt{n}} \sum_{i=1}^n a_i [w_i^T(0) x_k]_+
        \quad
        \forall k \in [m],
    \end{equation}
    where
    \begin{equation}
        H_{kl}(t) =
        \frac{1}{n} \sum_{i=1}^n [w_i^T(t) x_k \geq 0] [w_i^T(t) x_l \geq 0] x_k^T x_l.
    \end{equation}
    Additionaly, following \cite{arora2019fine}, consider the limiting linearized dynamics:
    \begin{equation}
        \frac{d{\vec u}'(t)}{dt} =
        H^\infty (\vec y - {\vec u}'(t)),
        \quad
        u'_k(0) =
        u_k(0)
        \quad
        \forall k \in [m],
    \end{equation}
    where
    \begin{equation}
        H^\infty_{kl} =
        \EE H_{kl}(0) =
        \EE_{w \sim \NN(0, I)} [w^T x_k \geq 0] [w^T x_l \geq 0] x_k^T x_l.
    \end{equation}
    Solving the above gives:
    \begin{equation}
        {\vec u}'(t) =
        \vec y + e^{-H^\infty t} (\vec u(0) - \vec y)
    \end{equation}
    Consider an eigenvalue-eigenvector decomposition for $H^\infty$: $H^\infty = \sum_{k=1}^m \lambda_k \vec v_k^T \vec v_k$, where $\{\vec v_k\}_{k=1}^m$ forms an orthonormal basis in $\RR^m$ and $\lambda_1 \geq \ldots \geq \lambda_m \geq 0$.
    Note that $\exp(-H^\infty t)$ then has the same set of eigenvectors, and each eigenvector $\vec v_k$ corresponds to an eigenvalue $\exp(-\lambda_k t)$.
    Then the above solution is rewritten as:
    \begin{equation}
        {\vec u}'(t) - \vec y =
        -\sum_{k=1}^m e^{-\lambda_k t} (\vec v_k^T (\vec y - \vec u(0))) \vec v_k,
    \end{equation}
    which implies
    \begin{equation}
        \|{\vec u}'(t) - \vec y\|_2^2 =
        \sum_{k=1}^m e^{-2\lambda_k t} (\vec v_k^T (\vec y - \vec u(0)))^2.
    \end{equation}
    We see that components $\vec v_k^T (\vec y - \vec u(0))$ that correspond to large $\lambda_k$ decay faster.
    Hence convergence is fast if $\forall k \in [m]$ large $\vec v_k^T (\vec y - \vec u(0))$ implies large $\lambda_k$.
    In this case, we shall say that the initial kernel aligns well with the dataset.

    It turns out, that realistic datasets align well with NTKs of realistic nets, however, datasets with random labels do not.
    This observation substitutes a plausible explanation for a phenomenon noted in~\cite{zhang2016understanding}: large networks learn corrupted datasets much slower than clean ones.

    The above speculation is valid for the limiting linearized dynamics ${\vec u}'(t)$.
    It turns out that given $n$ large enough, the true dynamics $\vec u(t)$ stays close to its limiting linearized version:
    \begin{theorem}[\cite{arora2019fine}]
        Suppose $\lambda_0 = \lambda_{min}(H^\infty) > 0$.
        Take $\epsilon > 0$ and $\delta \in (0,1)$.
        Then there exists a constant $C_n > 0$ such that for 
        \begin{equation}
            n \geq
            C_n \frac{m^7}{\lambda_0^4 \delta^4 \epsilon^2},
        \end{equation}
        w.p. $\geq 1-\delta$ over initialization, $\forall t \geq 0$ $\left\| \vec u(t) - \vec u'(t) \right\|_2 \leq \epsilon$.
        \label{thm:kernel_alignment}
    \end{theorem}

    \begin{proof}
        We start with stating a reformulation of Lemma~\ref{lemma:the_two_lemmas_hold_simultaneously}:
        \begin{lemma}
            Let $\delta \in (0,1)$.
            There exists $C_n' > 0$ such that for $n \geq C_n' \frac{m^6}{\lambda_0^4 \delta^3}$, w.p. $\geq 1-\delta$ over initialization, $\forall t \geq 0$
            \begin{equation}
                \| w_i(t) - w_i(0) \|_2 \leq
                R' :=
                \frac{4\sqrt{m} \|\vec y - \vec u(0)\|_2}{\sqrt{n}}
                \quad
                \forall i \in [n].
            \end{equation}
            \label{lemma:the_two_lemmas_hold_simultaneously_analogue}
        \end{lemma}
        We proceed with an analogue of Lemma~\ref{lemma:initial_gram_matrix_stability}:
        \begin{lemma}
            Let $\delta \in (0,1)$.
            There exist $C_H, C_Z > 0$ such that w.p. $\geq 1-\delta$ over initialization, $\forall t \geq 0$
            \begin{equation}
                \| H(t) - H(0) \|_F \leq
                C_H \frac{m^3}{n^{1/2} \lambda_0 \delta^{3/2}},
                \qquad
                \| Z(t) - Z(0) \|_F \leq
                C_Z \sqrt{\frac{m^2}{n^{1/2} \lambda_0 \delta^{3/2}}}.
            \end{equation}
            \label{lemma:initial_gram_matrix_stability_analogue}
        \end{lemma}
        The last lemma we need is an analogue of Lemma~\ref{lemma:initial_gram_matrix_convergence}:
        \begin{lemma}
            Let $\delta \in (0,1)$.
            There exist $C_H' > 0$ such that w.p. $\geq 1-\delta$ over initialization,
            \begin{equation}
                \| H(0) - H^\infty \|_F \leq
                C_H' \frac{m}{n^{1/2}} \log\left(\frac{m}{\delta}\right).
            \end{equation}
            \label{lemma:initial_gram_matrix_convergence_analogue}
        \end{lemma}

        Let us elaborate the dynamics over:
        \begin{equation}
            \frac{d\vec u(t)}{dt} =
            H(t) (\vec y - \vec u(t)) =
            H^\infty (\vec y - \vec u(t)) + (H(t) - H^\infty) (\vec y - \vec u(t)) =
            H^\infty (\vec y - \vec u(t)) + \vec\zeta(t).
        \end{equation}
        \begin{equation}
            \vec u(t) =
            e^{-H^\infty t} \vec C(t).
        \end{equation}
        \begin{equation}
            \frac{d\vec u(t)}{dt} =
            -H^\infty e^{-H^\infty t} \vec C(t) + e^{-H^\infty t} \frac{d\vec C(t)}{dt} =
            H^\infty (\vec y - \vec u(t)) + \vec\zeta(t) - H^\infty \vec y + e^{-H^\infty t} \frac{d\vec C(t)}{dt} - \vec\zeta(t).
        \end{equation}
        \begin{equation}
            \frac{d\vec C(t)}{dt} =
            e^{H^\infty t} (H^\infty \vec y + \vec\zeta(t)).
        \end{equation}
        \begin{equation}
            \vec C(t) =
            \vec u(0) + (e^{H^\infty t} - I) \vec y + \int_0^t e^{H^\infty \tau} \vec\zeta(\tau) \, d\tau.
        \end{equation}
        \begin{equation}
            \vec u(t) =
            \vec y + e^{-H^\infty t} (\vec u(0) - \vec y) + \int_0^t e^{H^\infty (\tau-t)} \vec\zeta(\tau) \, d\tau.
        \end{equation}
        \begin{multline}
            \| \vec u(t) - \vec u'(t) \|_2 =
            \left\|\int_0^t e^{H^\infty (\tau-t)} \vec\zeta(\tau) \, d\tau\right\|_2 \leq
            \int_0^t \left\|e^{H^\infty (\tau-t)} \vec\zeta(\tau)\right\|_2 \, d\tau \leq
            \\\leq
            \max_{\tau \in [0,t]} \| \vec\zeta(\tau)\|_2 \int_0^t \left\|e^{-H^\infty \tau}\right\|_2 \, d\tau \leq
            \max_{\tau \in [0,t]} \| \vec\zeta(\tau)\|_2 \int_0^t e^{-\lambda_0 \tau} \, d\tau \leq
            \\\leq
            \max_{\tau \in [0,t]} \| \vec\zeta(\tau)\|_2 \frac{1}{\lambda_0} \left(1 - e^{-\lambda_0 t}\right) \leq
            \frac{1}{\lambda_0} \max_{\tau \in [0,t]} \| \vec\zeta(\tau)\|_2.
        \end{multline}
        \begin{multline}
            \| \vec\zeta(\tau) \|_2 =
            \| (H(\tau) - H^\infty) (\vec y - \vec u(\tau)) \|_2 \leq
            \left(\| H(\tau) - H(0) \|_2 + \| H(0) - H^\infty \|_2\right) \| \vec y - \vec u(\tau) \|_2 \leq
            \\\leq
            \left(\| H(\tau) - H(0) \|_F + \| H(0) - H^\infty \|_F\right) \| \vec y - \vec u(0) \|_2.
        \end{multline}
        Due to Lemma~\ref{lemma:initial_gram_matrix_convergence_analogue} and Lemma~\ref{lemma:initial_gram_matrix_stability_analogue}, and since $\| \vec y - \vec u(0) \|_2 \leq \sqrt{2m / \delta}$ w.p. $\geq 1-\delta$ over initialization, we have
        \begin{equation}
            \| \vec\zeta(\tau) \|_2 \leq
            \left(C_H \frac{m^3}{n^{1/2} \lambda_0 \delta^{3/2}} + C_H' \frac{m}{n^{1/2}} \log\left(\frac{m}{\delta}\right)\right) \sqrt{\frac{2m}{\delta}} =
            \sqrt{2} C_H \frac{m^{7/2}}{n^{1/2} \lambda_0 \delta^2} + \sqrt{2} C_H' \frac{m^{3/2}}{n^{1/2} \delta^{1/2}} \log\left(\frac{m}{\delta}\right)
        \end{equation}
        w.p. $\geq 1-3\delta$ over initialization.
        Given $\epsilon > 0$, we need
        \begin{equation}
            n \geq
            C_n \frac{m^7}{\lambda_0^4 \delta^4 \epsilon^2}
        \end{equation}
        for some $C_n > 0$ in order to ensure $\| \vec u(t) - \vec u'(t) \|_2 \leq \epsilon$ w.p. $\geq 1-\delta$ over initialization.
    \end{proof}

    \bibliography{references}

\begin{thebibliography}{}

\bibitem[Arora et~al., 2019a]{arora2019fine}
Arora, S., Du, S., Hu, W., Li, Z., and Wang, R. (2019a).
\newblock Fine-grained analysis of optimization and generalization for
  overparameterized two-layer neural networks.
\newblock In {\em International Conference on Machine Learning}, pages
  322--332.

\bibitem[Arora et~al., 2019b]{arora2019exact}
Arora, S., Du, S.~S., Hu, W., Li, Z., Salakhutdinov, R.~R., and Wang, R.
  (2019b).
\newblock On exact computation with an infinitely wide neural net.
\newblock In {\em Advances in Neural Information Processing Systems}, pages
  8141--8150.

\bibitem[Bartlett et~al., 2017]{bartlett2017spectrally}
Bartlett, P.~L., Foster, D.~J., and Telgarsky, M.~J. (2017).
\newblock Spectrally-normalized margin bounds for neural networks.
\newblock In {\em Advances in Neural Information Processing Systems}, pages
  6240--6249.

\bibitem[Bartlett et~al., 2019]{bartlett2019nearly}
Bartlett, P.~L., Harvey, N., Liaw, C., and Mehrabian, A. (2019).
\newblock Nearly-tight vc-dimension and pseudodimension bounds for piecewise
  linear neural networks.
\newblock {\em J. Mach. Learn. Res.}, 20:63--1.

\bibitem[Donsker and Varadhan, 1985]{donsker1985large}
Donsker, M. and Varadhan, S. (1985).
\newblock Large deviations for stationary gaussian processes.
\newblock {\em Communications in Mathematical Physics}, 97(1-2):187--210.

\bibitem[Draxler et~al., 2018]{draxler2018essentially}
Draxler, F., Veschgini, K., Salmhofer, M., and Hamprecht, F. (2018).
\newblock Essentially no barriers in neural network energy landscape.
\newblock In {\em International Conference on Machine Learning}, pages
  1309--1318.

\bibitem[Du et~al., 2019]{du2018gradient}
Du, S.~S., Zhai, X., Poczos, B., and Singh, A. (2019).
\newblock Gradient descent provably optimizes over-parameterized neural
  networks.
\newblock In {\em International Conference on Learning Representations}.

\bibitem[Dudley, 1967]{dudley1967sizes}
Dudley, R.~M. (1967).
\newblock The sizes of compact subsets of hilbert space and continuity of
  gaussian processes.
\newblock {\em Journal of Functional Analysis}, 1(3):290--330.

\bibitem[Dyer and Gur-Ari, 2020]{Dyer2020Asymptotics}
Dyer, E. and Gur-Ari, G. (2020).
\newblock Asymptotics of wide networks from feynman diagrams.
\newblock In {\em International Conference on Learning Representations}.

\bibitem[Dziugaite and Roy, 2017]{dziugaite2017computing}
Dziugaite, G.~K. and Roy, D.~M. (2017).
\newblock Computing nonvacuous generalization bounds for deep (stochastic)
  neural networks with many more parameters than training data.
\newblock {\em arXiv preprint arXiv:1703.11008}.

\bibitem[Garipov et~al., 2018]{garipov2018loss}
Garipov, T., Izmailov, P., Podoprikhin, D., Vetrov, D.~P., and Wilson, A.~G.
  (2018).
\newblock Loss surfaces, mode connectivity, and fast ensembling of dnns.
\newblock In {\em Advances in Neural Information Processing Systems}, pages
  8789--8798.

\bibitem[Glorot and Bengio, 2010]{glorot2010understanding}
Glorot, X. and Bengio, Y. (2010).
\newblock Understanding the difficulty of training deep feedforward neural
  networks.
\newblock In {\em Proceedings of the thirteenth international conference on
  artificial intelligence and statistics}, pages 249--256.

\bibitem[He et~al., 2015]{he2015delving}
He, K., Zhang, X., Ren, S., and Sun, J. (2015).
\newblock Delving deep into rectifiers: Surpassing human-level performance on
  imagenet classification.
\newblock In {\em Proceedings of the IEEE international conference on computer
  vision}, pages 1026--1034.

\bibitem[Hoeffding, 1963]{hoeffding1963probability}
Hoeffding, W. (1963).
\newblock Probability inequalities for sums of bounded random variables.
\newblock {\em Journal of the American Statistical Association},
  58(301):13--30.

\bibitem[Huang and Yau, 2019]{huang2019dynamics}
Huang, J. and Yau, H.-T. (2019).
\newblock Dynamics of deep neural networks and neural tangent hierarchy.
\newblock {\em arXiv preprint arXiv:1909.08156}.

\bibitem[Isserlis, 1918]{isserlis1918formula}
Isserlis, L. (1918).
\newblock On a formula for the product-moment coefficient of any order of a
  normal frequency distribution in any number of variables.
\newblock {\em Biometrika}, 12(1/2):134--139.

\bibitem[Jacot et~al., 2018]{jacot2018neural}
Jacot, A., Gabriel, F., and Hongler, C. (2018).
\newblock Neural tangent kernel: Convergence and generalization in neural
  networks.
\newblock In {\em Advances in neural information processing systems}, pages
  8571--8580.

\bibitem[Jin et~al., 2017]{jin2017escape}
Jin, C., Ge, R., Netrapalli, P., Kakade, S.~M., and Jordan, M.~I. (2017).
\newblock How to escape saddle points efficiently.
\newblock In {\em International Conference on Machine Learning}, pages
  1724--1732.

\bibitem[Kawaguchi, 2016]{kawaguchi2016deep}
Kawaguchi, K. (2016).
\newblock Deep learning without poor local minima.
\newblock In {\em Advances in neural information processing systems}, pages
  586--594.

\bibitem[Laurent and Brecht, 2018]{laurent2018deep}
Laurent, T. and Brecht, J. (2018).
\newblock Deep linear networks with arbitrary loss: All local minima are
  global.
\newblock In {\em International conference on machine learning}, pages
  2902--2907. PMLR.

\bibitem[Lee et~al., 2019]{lee2019wide}
Lee, J., Xiao, L., Schoenholz, S., Bahri, Y., Novak, R., Sohl-Dickstein, J.,
  and Pennington, J. (2019).
\newblock Wide neural networks of any depth evolve as linear models under
  gradient descent.
\newblock In {\em Advances in neural information processing systems}, pages
  8572--8583.

\bibitem[Lee et~al., 2016]{lee2016gradient}
Lee, J.~D., Simchowitz, M., Jordan, M.~I., and Recht, B. (2016).
\newblock Gradient descent only converges to minimizers.
\newblock In {\em Conference on learning theory}, pages 1246--1257.

\bibitem[Lu and Kawaguchi, 2017]{lu2017depth}
Lu, H. and Kawaguchi, K. (2017).
\newblock Depth creates no bad local minima.
\newblock {\em arXiv preprint arXiv:1702.08580}.

\bibitem[Marchenko and Pastur, 1967]{marchenko1967}
Marchenko, V.~A. and Pastur, L.~A. (1967).
\newblock Распределение собственных значений в
  некоторых ансамблях случайных матриц.
\newblock {\em Математический сборник}, 72(4):507--536.

\bibitem[McAllester, 1999a]{mcallester1999pac}
McAllester, D.~A. (1999a).
\newblock Pac-bayesian model averaging.
\newblock In {\em Proceedings of the twelfth annual conference on Computational
  learning theory}, pages 164--170.

\bibitem[McAllester, 1999b]{mcallester1999some}
McAllester, D.~A. (1999b).
\newblock Some pac-bayesian theorems.
\newblock {\em Machine Learning}, 37(3):355--363.

\bibitem[McDiarmid, 1989]{mcdiarmid1989method}
McDiarmid, C. (1989).
\newblock On the method of bounded differences.
\newblock {\em Surveys in combinatorics}, 141(1):148--188.

\bibitem[Nagarajan and Kolter, 2019]{nagarajan2019uniform}
Nagarajan, V. and Kolter, J.~Z. (2019).
\newblock Uniform convergence may be unable to explain generalization in deep
  learning.
\newblock In {\em Advances in Neural Information Processing Systems}, pages
  11615--11626.

\bibitem[Neyshabur et~al., 2018]{neyshabur2018a}
Neyshabur, B., Bhojanapalli, S., and Srebro, N. (2018).
\newblock A {PAC}-bayesian approach to spectrally-normalized margin bounds for
  neural networks.
\newblock In {\em International Conference on Learning Representations}.

\bibitem[Neyshabur et~al., 2015]{neyshabur2014in}
Neyshabur, B., Tomioka, R., and Srebro, N. (2015).
\newblock In search of the real inductive bias: On the role of implicit
  regularization in deep learning.
\newblock In {\em ICLR (Workshop)}.

\bibitem[Nguyen, 2019]{nguyen2019connected}
Nguyen, Q. (2019).
\newblock On connected sublevel sets in deep learning.
\newblock In {\em International Conference on Machine Learning}, pages
  4790--4799.

\bibitem[Nguyen and Hein, 2017]{nguyen2017loss}
Nguyen, Q. and Hein, M. (2017).
\newblock The loss surface of deep and wide neural networks.
\newblock In {\em Proceedings of the 34th International Conference on Machine
  Learning-Volume 70}, pages 2603--2612.

\bibitem[Panageas and Piliouras, 2017]{panageas2017gradient}
Panageas, I. and Piliouras, G. (2017).
\newblock Gradient descent only converges to minimizers: Non-isolated critical
  points and invariant regions.
\newblock In {\em 8th Innovations in Theoretical Computer Science Conference
  (ITCS 2017)}. Schloss Dagstuhl-Leibniz-Zentrum fuer Informatik.

\bibitem[Pennington et~al., 2017]{pennington2017resurrecting}
Pennington, J., Schoenholz, S., and Ganguli, S. (2017).
\newblock Resurrecting the sigmoid in deep learning through dynamical isometry:
  theory and practice.
\newblock In {\em Advances in neural information processing systems}, pages
  4785--4795.

\bibitem[Poole et~al., 2016]{poole2016exponential}
Poole, B., Lahiri, S., Raghu, M., Sohl-Dickstein, J., and Ganguli, S. (2016).
\newblock Exponential expressivity in deep neural networks through transient
  chaos.
\newblock In {\em Advances in neural information processing systems}, pages
  3360--3368.

\bibitem[Sauer, 1972]{sauer1972density}
Sauer, N. (1972).
\newblock On the density of families of sets.
\newblock {\em Journal of Combinatorial Theory, Series A}, 13(1):145--147.

\bibitem[Saxe et~al., 2013]{saxe2013exact}
Saxe, A.~M., McClelland, J.~L., and Ganguli, S. (2013).
\newblock Exact solutions to the nonlinear dynamics of learning in deep linear
  neural networks.
\newblock {\em arXiv preprint arXiv:1312.6120}.

\bibitem[Schoenholz et~al., 2016]{schoenholz2016deep}
Schoenholz, S.~S., Gilmer, J., Ganguli, S., and Sohl-Dickstein, J. (2016).
\newblock Deep information propagation.
\newblock {\em arXiv preprint arXiv:1611.01232}.

\bibitem[Tao, 2012]{tao2012topics}
Tao, T. (2012).
\newblock {\em Topics in random matrix theory}, volume 132.
\newblock American Mathematical Soc.

\bibitem[Tropp, 2011]{tropp2011user}
Tropp, J.~A. (2011).
\newblock User-friendly tail bounds for sums of random matrices.
\newblock {\em Foundations of Computational Mathematics}, 12(4):389–434.

\bibitem[Vapnik and Chervonenkis, 1971]{vapnik1971}
Vapnik, V.~N. and Chervonenkis, A.~Y. (1971).
\newblock О равномерной сходимости частот
  появления событий к их вероятностям.
\newblock {\em Теория вероятностей и ее
  применения}, 16(2):264--279.

\bibitem[Voiculescu, 1987]{voiculescu1987multiplication}
Voiculescu, D. (1987).
\newblock Multiplication of certain non-commuting random variables.
\newblock {\em Journal of Operator Theory}, pages 223--235.

\bibitem[Yu and Chen, 1995]{yu1995local}
Yu, X.-H. and Chen, G.-A. (1995).
\newblock On the local minima free condition of backpropagation learning.
\newblock {\em IEEE Transactions on Neural Networks}, 6(5):1300--1303.

\bibitem[Zhang et~al., 2016]{zhang2016understanding}
Zhang, C., Bengio, S., Hardt, M., Recht, B., and Vinyals, O. (2016).
\newblock Understanding deep learning requires rethinking generalization.
\newblock {\em arXiv preprint arXiv:1611.03530}.

\bibitem[Zhou et~al., 2019]{zhou2018nonvacuous}
Zhou, W., Veitch, V., Austern, M., Adams, R.~P., and Orbanz, P. (2019).
\newblock Non-vacuous generalization bounds at the imagenet scale: a
  {PAC}-bayesian compression approach.
\newblock In {\em International Conference on Learning Representations}.

\end{thebibliography}
    \bibliographystyle{apalike}

\end{document}